%% file: main.tex
\documentclass{article}

\pdfoutput=1
\usepackage[dvipsnames,svgnames]{xcolor}
\usepackage{subcaption}

\usepackage[accepted,nohyperref]{icml2024}
\usepackage[toc,page,header]{appendix}
\usepackage{titletoc}
\usepackage{url}
\usepackage[utf8]{inputenc} 
\usepackage[T1]{fontenc}    
\usepackage[hidelinks]{hyperref}       
\usepackage{enumitem}
\usepackage{url}            
\usepackage{xspace}
\usepackage{booktabs}       
\usepackage{amsfonts}       
\usepackage{nicefrac}       
\usepackage{microtype}      
\usepackage{multirow}
\usepackage{multicol}
\usepackage{amsmath}
\usepackage{amssymb}
\usepackage{amsthm}
\usepackage{graphicx}
\usepackage{wrapfig}
\usepackage{mathtools}
\usepackage[capitalise, noabbrev]{cleveref}
\usepackage[normalem]{ulem}
\usepackage{tcolorbox}
\usepackage{fourier-orns}

\usepackage[suppress]{color-edits}  
\addauthor{noveen}{blue}
\addauthor{wc}{red}
\addauthor{jianmo}{purple}
\addauthor{derek}{cyan}
\addauthor{ben}{orange}

\hypersetup{
    colorlinks = true,
    citecolor = [RGB]{0,130,130},
    urlcolor = [RGB]{200,0,100}
}

\newcommand{\titlevariable}[0]{How to Train Data-Efficient LLMs}

\icmltitlerunning{\titlevariable}

\begin{document}

\include{defintions}

\twocolumn[
\icmltitle{\titlevariable}


\begin{icmlauthorlist}
\icmlauthor{Noveen Sachdeva}{deepmind,ucsd}
\icmlauthor{Benjamin Coleman}{deepmind}
\icmlauthor{Wang-Cheng Kang}{deepmind}
\icmlauthor{Jianmo Ni}{deepmind}
\icmlauthor{Lichan Hong}{deepmind}
\icmlauthor{Ed H. Chi}{deepmind}
\icmlauthor{James Caverlee}{deepmind,tamu}
\icmlauthor{Julian McAuley}{ucsd}
\icmlauthor{Derek Zhiyuan Cheng}{deepmind}
\end{icmlauthorlist}

\icmlaffiliation{ucsd}{University of California, San Diego}
\icmlaffiliation{deepmind}{Google DeepMind}
\icmlaffiliation{tamu}{Texas A\&M University}

\icmlcorrespondingauthor{Noveen Sachdeva}{noveen@google.com}

\icmlkeywords{Data Efficiency, Data Sampling, Data Pruning, Pre-training, LLMs}

\vskip 0.3in
]

\printAffiliationsAndNotice{}  

\begin{abstract}
    The training of large language models (LLMs) is expensive. In this paper, we study data-efficient approaches for pre-training LLMs, \ie, techniques that aim to optimize the Pareto frontier of model quality and training resource/data consumption. We seek to understand the tradeoffs associated with data selection routines based on (i) expensive-to-compute \emph{data-quality} estimates, and (ii) maximization of \textit{coverage} and diversity-based measures in the feature space. Our first technique, \askllm, leverages the zero-shot reasoning capabilities of instruction-tuned LLMs to directly assess the quality of a training example. To target coverage, we propose \density sampling, which models the data distribution to select a diverse sample. In our comparison of $19$ samplers, involving hundreds of evaluation tasks and pre-training runs, we find that \askllm and \density are the best methods in their respective categories. Coverage sampling can \emph{recover} the performance of the full data, while models trained on \askllm data consistently \emph{outperform} full-data training---even when we reject $90$\% of the original dataset, while converging up to $70$\% faster.
\end{abstract}

\section{Introduction}

Large language model (LLM) pre-training is perhaps the most data- and compute-intensive task attempted by the machine learning community to date,  with impressive capabilities primarily being accomplished by training massive transformer architectures on trillions of tokens of text \cite{gpt4, gemini, llama2}.

\begin{figure*}[!t] 
    \centering
    \includegraphics[width=0.95\linewidth,trim={0 0 0 0},clip]{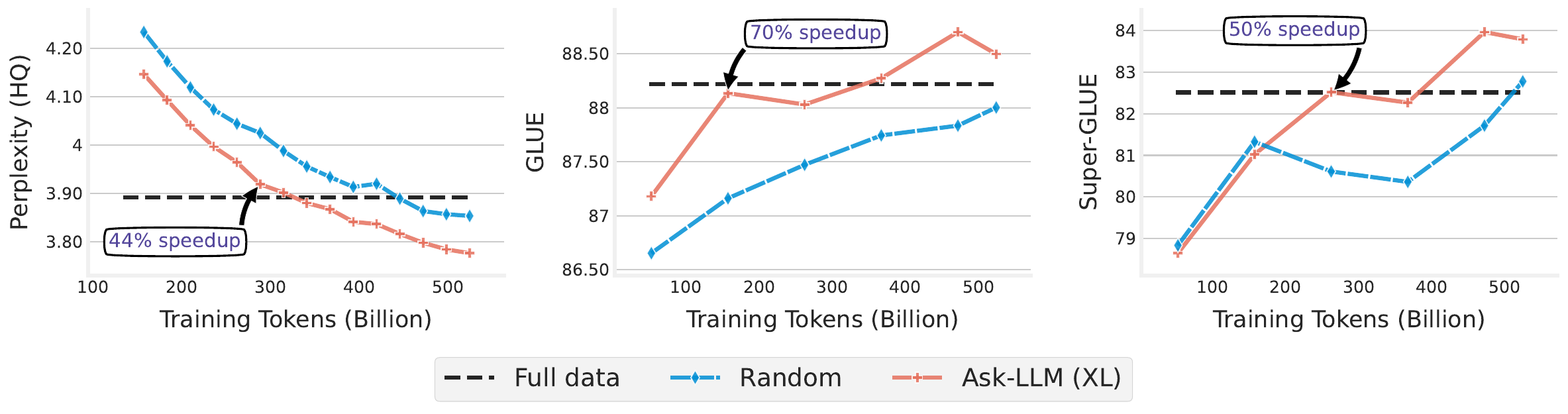}
    
    \vspace{-0.2cm}
    \caption{Data-efficient pre-training run of T5-Large ($800$M) using \askllm with Flan-T5-XL as the data quality scorer. Training on $60$\% of the original dataset, \askllm is able to train T5-Large both better and $70$\% faster, compared to training on $100$\% of the dataset.}
    \label{fig:ask_llm_teaser_results}
\end{figure*}

But even these incredibly capable LLMs are subject to empirical scaling laws, which predict sharply diminishing returns from a linear increase in model- or data-size~\cite{chinchilla, kaplan_scaling}.
Power-law scaling therefore acts as a soft limit on model quality, beyond which it is prohibitively expensive to drive performance by scaling up the data or model.
At the same time, \citet{prototypes}---in the context of vision pre-training---show that we can significantly improve the power law constants in the aforementioned scaling laws if we prioritize \emph{important} training examples using some robust notion of data quality or impact.

A similar call for data-curation is also apparent in the context of training LLMs, where our largest models are quickly approaching their capacity and data thresholds. 
LIMA~\cite{lima} showed that LLaMA-65B~\cite{llama} can be better aligned with human preferences when trained on a set of 1,000 carefully selected fine-tuning prompts, compared to training on as much as 52,000 unfiltered examples.
\citet{tirumala2023d4} recently conducted a large-scale data-efficient pre-training evaluation, showing that a 6.7B OPT model~\cite{zhang2022opt} can converge up to 20\% faster on data curated by 
a technique based on stratified cluster sampling.
The Phi-2 experiments also suggest that when data curation is performed at a human-expert level (\eg, by textbook editors), models can outperform baselines that are up to 25x larger~\cite{javaheripi2023phi}.

Data curation routines can be fundamentally characterized as selecting training samples for quality, coverage, or some mixture of both (\cref{fig:coverage_quality_spectrum}). In this work, we seek to understand how quality and coverage affect the data efficiency of LLM pre-training. Our core research question is:
\begin{quote}\itshape
    ``Are cheap-to-compute heuristics like maximum-coverage enough to pre-train a SoTA LLM, or are there real benefits from costly samplers that carefully evaluate the quality of each example?''
\end{quote}

This question is crucial to answer because data-curation algorithms can improve the Pareto frontier of the data-quantity$\leftrightarrow$model-quality tradeoff, 
directly addressing the bottleneck of power-law scaling by enabling higher-quality models to be trained using less data.
Data curation also unlocks new tradeoffs between training time, inference cost, data collection effort, and downstream performance.
For example, if we consider the compute-constrained (single-epoch) regime, a data-efficient LLM training routine may reach the desired performance using only X\% of the data (corresponding to an X\% training speedup).

Despite considerable interest from the community for building data-efficient training methods~\cite{prototypes,el2n,svp,selective_backprop,katharopoulos2018not},
large-scale analyses of data pruning strategies are rare because of the extreme computational cost---especially in the context of LLM pre-training.
To be more specific, an extensive comparative study necessarily entails pre-training (i) various sizes of LLMs, (ii) for a variety of data sampling rates, (iii) obtained through various pruning strategies. Further, downstream evaluations for LLMs also frequently involve fine-tuning, which is resource intensive in itself.

\subsection{Contributions} 
We hypothesize that the roles of coverage and quality depend on the stage of training, size of the model, and the sampling rate.
To understand the coverage/quality design choice better, we develop new data-efficiency routines that independently (and solely) target quality and coverage.
Our \askllm sampler prioritizes high-quality and informative training samples by asking a \emph{proxy} LLM. 
Our \density sampler seeks to maximize the coverage of latent topics in the input dataset through a diversified sampling procedure. To summarize, our contributions are as follows:

\textbf{\askllm sampling.}
We find that \askllm can train better models (\vs training on the \emph{entire dataset}) even after removing up to $90$\% of training samples, while also consistently beating well-established data curation routines. We note that even a tiny proxy model in \askllm ($60$M parameters) can outperform most baselines.

\textbf{Exhaustive benchmark.} We implement $19$ different sampling strategies for pre-training T5-Large ($800$M) and T5-Small ($60$M) on $524$B tokens and evaluate them on $111$ downstream evaluation tasks. This leads to a total of $170$ pre-training and $2,500$ fine-tuning runs.

\textbf{New insights.} By analyzing the differences between \askllm and \density sampling, we study the role of coverage, quality, and sampling cost in LLM pre-training. We support our conclusions with additional studies of the convergence rate, correlations between sampler outputs, and impact of sampling cost on downstream performance.

\textbf{Takeaway.} Our results show that while coverage sampling can \emph{recover} the performance of the full data, \askllm (quality filtering) can often \emph{exceed} it.
These experiments suggest that LLM-based quality raters are a worthwhile and effective way to drive performance in pre-training.

\section{Related Work}

Data selection is a classical problem with well-established literature on coresets, sketching, importance sampling, filtering, denoising, and a host of other algorithms with similar goals. 
While we cannot possibly catalog the entire sampling literature, we hope to provide an overview of the principles behind common data selection algorithms. We also describe how these algorithms have been applied to machine learning, with a focus on language model training.

\subsection{Coverage Sampling}

The first class of methods maximize the \textit{coverage} of the sample by selecting points that are evenly distributed across the entire input domain, e.g., an $\epsilon$-net for a Lipschitz function~\cite{phillips2017coresets}. When training language models, coverage sampling is motivated by the intuition that we ought to show the model the full breadth of genres, topics, and languages~\cite{longpre2023pretrainer}. Coverage sampling is typically accomplished by embedding examples into a metric space and selecting points which are mutually far from each other~\cite{lee2023beyond}.

Cluster sampling algorithms group inputs based on embedding similarity and select representatives from each group. These algorithms are popular, scalable, interpretable, and enjoy strong theoretical support -- $k$-means sampling provably approximates the SVM objective~\cite{tukan2021coresets} and many others~\cite{feldman2020turning}. However, there are also recent techniques based on submodular optimization of a coverage score~\cite{chen2012super, indyk2014composable, bilevel_coresets}, models of the data distribution~\cite{density}, discrepancy minimization~\cite{discrepancy}, and deduplication through token matching / locality-sensitive hashing~\cite{lee2022deduplicating}.

Many variations of cluster sampling have been applied to vision and language model training. ~\citet{prototypes} propose the ``SSL prototypes'' method for vision models, which removes points that fall too close to the nearest $k$-means centroid.
SemDeDup~\cite{semdedup} also removes points based on this distance, but targets pairs of nearby examples, or ``semantic duplicates,'' and prefers points close to the centroid.
The D4 sampler chains MinHash deduplication, SemDeDup, and SSL prototypes together to prune both high-variance, sparse regions and prototypical, dense regions of LLM pre-training datasets~\cite{tirumala2023d4}. \citet{svp} considers a $k$-centers submodular selection routine on the last-layer embeddings of ResNet vision models. 

\begin{figure}[t!] 
    \centering
    \includegraphics[width=0.9\linewidth,trim={0 0 0 0},clip]{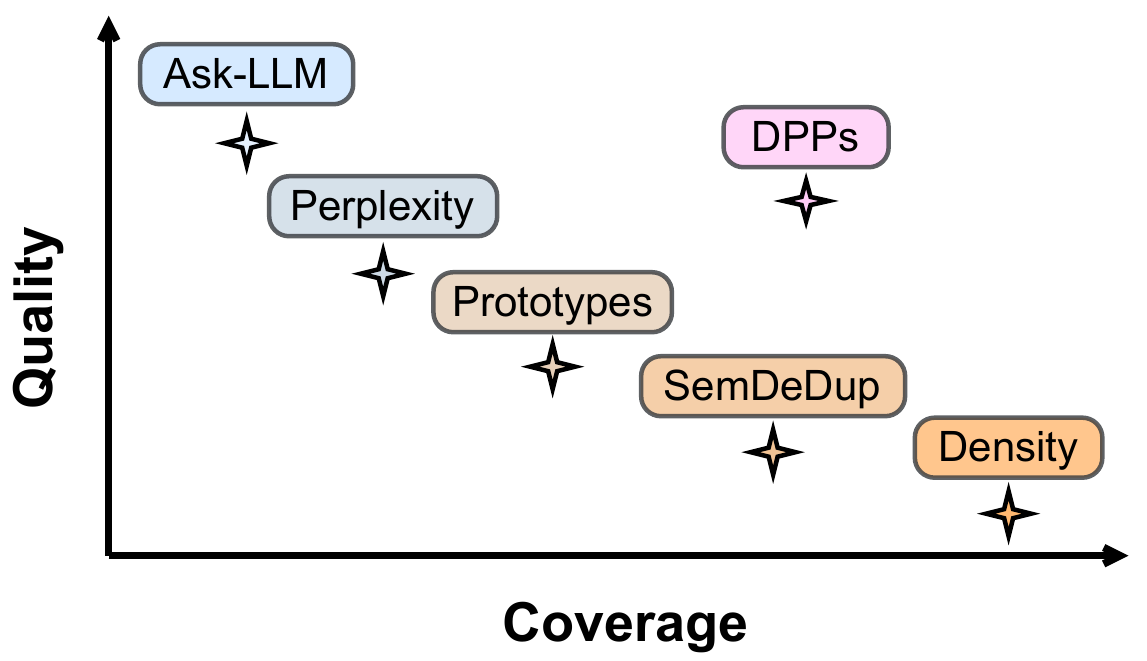} 
    
    \vspace{-0.2cm}
    \caption{While there is no inherent tradeoff between coverage and quality, samplers target these metrics on a spectrum (up and to the left indicates a more aggressive prioritization). See \cref{appendix:samplers} for a description of the plotted samplers.}
    \vspace{-0.2cm}
    \label{fig:coverage_quality_spectrum}
\end{figure}

\subsection{Quality-score Sampling}

Another class of methods are based on \textit{quality scores}, where a scoring algorithm rates every example and the sampler preferentially selects points with high scores.
Even though 
this framework was originally developed for importance sampling~\cite{hastings1970monte}, 
the machine learning community has expanded the theoretical ``score-and-sample'' framework to include a variety of practical heuristics.

For example, the selection-via-proxy (SVP) algorithm determines the importance of an input using the validation loss and uncertainty scores of a pre-trained model on the input~\cite{svp, svp_cf}. \citet{el2n} sample according to an ``EL2N score'' formed by ensembling the losses of 10 lightly-trained models. Ensemble prediction variance has also been used as the scoring metric~\cite{chitta2021training}, as have ensemble disagreement rates~\cite{meding2021trivial}. Other scores measure whether an example is likely to be forgotten~\cite{forgetting_events}, memorized~\cite{feldman2020neural}, or un-learnable~\cite{mindermann2022prioritized}. 

In the context of pre-training LLMs, there exist a few different schools-of-thought for scoring the quality of training samples. The first (and arguably most used) camp is perplexity-filtering, where we prioritize samples with \emph{low} perplexity and filter out highly surprising examples \cite{wenzek2019ccnet,marion2023less,muennighoff2023scaling}. 
Notably, recent advancements in cheaper to run model-based \emph{training-run simulators} for LLMs 
can be used to \emph{estimate} the perplexity of a training sample instead of running an LLM inference \cite{simfluence}.
Another group of methods selects training data that minimizes the \emph{distance} between the distribution of selected data and a handcrafted high-quality data source (typically wikipedia and books). Typical ways are to do this in a feature space \cite{xie2023data} or by training a contrastive-style classifer \cite{gpt2, palm2, javaheripi2023phi}. Similar ideas have also been explored for optimizing the data mixture weights for pre-training \cite{doremi}.

In concurrent work,~\citet{maini2024rephrasing} also consider an LLM-based approach similar to our \askllm sampler, but with a focus on data paraphrasing rather than selection via quality evaluation. ~\citet{engstrom2024dsdm} consider a quality evaluation based on datamodels, though their analysis suggests that this approach selects for strongly model-dependent notions of quality.

\section{Methods}

We propose two samplers, \askllm and \density. These samplers have significantly different costs---\askllm requires an LLM inference call for each training sample, whereas \density is based on a diversified sampling routine that is cheaper than even clustering the dataset.
They also exhibit substantially different selection behavior: \askllm conducts a highly \emph{nuanced} and \emph{contextual} quality evaluation for each sample, while \density asks whether
we have already sampled many similar examples.
By studying samplers on extreme ends of this spectrum, we hope to better understand the salient factors for LLM data curation.

\subsection{\askllm Sampling} \label{sec:askllm_details}

\textbf{Intuition.} Our intuition is that humans can easily identify commonly occurring failure modes in state-of-the-art data quality scorers. 
Hence, it should be possible to correct these mistakes using the reasoning capabilities of modern instruction-tuned LLMs.

\begin{figure}[t!] 
    \centering
    \includegraphics[width=0.95\linewidth,trim={0.5cm 0.2cm 0.5cm 0},clip]{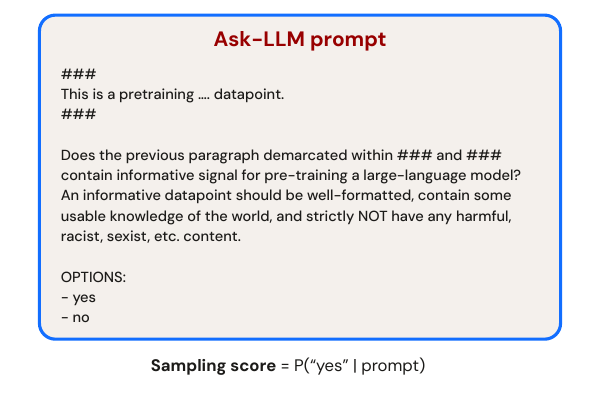} 
    
    \vspace{-0.2cm}
    \caption{The prompt for obtaining the sampling score for each training sample in \askllm.}
    \vspace{-0.2cm}
    \label{fig:ask_llm_prompt}
\end{figure}

To do so, in \askllm, we prompt an instruction-tuned \emph{proxy} LLM
with the prospective training example and ask whether the example should be used for training (see \cref{fig:ask_llm_prompt} for the prompt). 
We take the softmax probability of the token ``\texttt{yes}'' as the estimated data-quality score. Consider the following common failure modes of perplexity filtering, which the \askllm scoring model fixes (see more qualitative examples in \cref{appendix:qualitative_results}).

\textbf{Contextuality.} Perplexity filters often select samples that lack context, \eg, containing questions without answers (Examples~\ref{example:bad3},~\ref{example:bad4}, \ref{example:bad7}). \askllm correctly identifies that these examples do not provide new information.

\textbf{Nonsense.} Perplexity filters can select examples that endlessly repeat the same phrases / words (Examples \ref{example:bad6} and \ref{example:bad7}), likely because these word combinations are common (resulting in high likelihood).

\textbf{Niche examples.} Perplexity filters can reject niche topics that are otherwise informative, well-written, and contain useful \emph{tail knowledge} of the world. Example~\ref{example:increasing1} contains detailed information about a Manchester art installation but is assigned a high perplexity, likely because it contains uncommon (but valid) word combinations. Examples~\ref{example:increasing4}-\ref{example:increasing6} display similar behavior for other niche topics.

\subsection{\density Sampling} \label{sec:density_details}

\textbf{Intuition.} Our intuition is that the data distribution provides a strong coverage signal.
High-probability regions contain ``prototypical'' examples---ones with many near-duplicates and strong representation in the dataset.
Low-probability regions will contain outliers, noise, and unique/rare inputs. 
If we wish to maximize topic coverage, we should boost the signal from under-represented portions of the input domain and downsample redundant, high-density information.

The key difficulty for our \density sampler is to accurately estimate an example's local density. 
Like~\citet{tirumala2023d4} (D4), we assume access to embeddings from a pre-trained LLM. 
However, we depart from the traditional approach of clustering and opt to sample based on kernel sums. 
Given a dataset $D$ of embeddings and a kernel $k(x, y)$, we estimate the density using the following score.
$$\mathrm{score}(y) = \sum_{x \in D} k_{\lambda}(x, y).$$
$\lambda$ is a smoothing parameter called the \textit{kernel bandwidth} that controls the scale of the points' effects. To reduce the complexity from $O(N^2)$ to $O(N \log N)$, we use recent breakthroughs from the algorithm community to approximate the sum~\cite{siminelakis2019rehashing, coleman2020sub}. 
Our method resembles that of~\citet{density}, except that (i) we adopt a two-pass sampling algorithm with stronger theoretical guarantees (\cref{thm:converge_to_uniform}) and (ii) we perform the density estimation in the latent space of the model, rather than using Jaccard distances on $n$-grams.

\subsection{Sampling Techniques}

\density and \askllm are both \textit{scoring} methods that reduce an example to a floating point value that measures coverage or quality.
Once we have scores for a complete dataset of training examples (sentences, paragraphs, etc.), we can make score-based decisions about which examples to include in the training set.

\textbf{Top / Bottom $K$.} The simplest method is to sort examples by score and accept the top or bottom $K$. While straightforward, this approach is supported by the ``permutation'' theory of~\citet{prototypes}, and sensitivity score sampling (a softened version) is the core subroutine for many coresets~\cite{mai2021coresets}.
When applied to \density and perplexity scores, top-$K$ sampling selects for the head of the data distribution (similar to SSL prototypes). Bottom-$K$ sampling selects the tail and removes common items. 

\textbf{Inverse Propensity Sampling.} Inverse propensity sampling (IPS) selects items proportional to their reweighted and normalized \textit{inverse} score~\cite{rosenbaum1983central}. When applied to \density or perplexity scores, IPS implements a form of \textit{diversified sampling} that uniformizes the distribution of selected inputs (\cref{thm:converge_to_uniform}).

In our experiments, the \density sampler uses IPS to maximize the coverage of the dataset.\footnote{We also implemented top-$K$ and bottom-$K$ sampling, but these samplers do not maintain coverage and perform poorly.} For our \askllm filter, we adopt top-$k$ sampling because we expect the ``yes'' probability to be a reliable and strong measure of quality.

\subsection{Relationships Between Methods}
\label{sec:relationship_to_baselines}

\textbf{\density, Perplexity, and Loss.} When a language model is trained to minimize perplexity, the LLM itself is a data distribution model. Therefore, the perplexity and loss filtering approaches of~\citet{marion2023less},~\citet{muennighoff2023scaling}, and other authors can be viewed as model-based density sampling. 
However, our sampler measures the density of the training dataset in a latent geometric space, while perplexity measures the likelihood under the scoring model. The samplers also differ in terms of decision complexity. 
Thanks to the capacity of the LLM, a perplexity filter can make highly-nuanced decisions between two texts on the same topic.
On the other hand, our \density sampler is constructed from a simple nonparametric density model~\cite{rosenblatt1956remarks} that does not have the capacity to distinguish examples at such a granular level.

\textbf{\askllm and Perplexity.} Perplexity filters exhibit a strong in-distribution bias, making decisions based on the data used to train the scoring model (not the dataset we wish to sample).
By using the LLM for quality evaluation rather than likelihood estimation, our sampler can escape this bias because the additional context and alternative task change the sampling distribution.
This occurs even when the \askllm and perplexity models are the same size.

\textbf{\density and Clustering.} The kernel sum procedure at the core of our \density sampler operates on embedding-similarity relationships in a similar way to D4, SemDeDup, and SSL prototypes.
Indeed, near-duplicate detection can be viewed as a discretized version of similarity-based density estimation~\cite{kirsch2006distance}.
Outlier rejection, which motivates the ``nearest-to-centroid'' heuristic of SSL prototypes, also has intimate connections with density estimation~\cite{schubert2014generalized}.

\textbf{Intuition.} Perplexity should be viewed as a ``difficulty'' or ``quality'' score rather than as a coverage-maximizing score. Our \askllm sampler should be viewed as a contextualized quality score that incorporates reasoning.\footnote{Note that \askllm may also incidentally improve coverage because it does not suffer from in-distribution bias.} Our \density sampler is a pure ``coverage'' score in the latent representation space, while SemDeDup, and SSL Prototypes all incorporate quality / outlier filtering to some extent (\eg, by preferring points near / far from a centroid).

\begin{figure*}[!t] 
    \centering
    
    \captionsetup[subfigure]{skip=-3.37cm,slc=off,margin={55pt,0pt},font={color=red}}
    \subcaptionbox{\label{fig:sampling_ratio_vs_perf_large_perplexity}}{\includegraphics[width=0.31\linewidth,trim={0 4cm 27cm 0},clip]{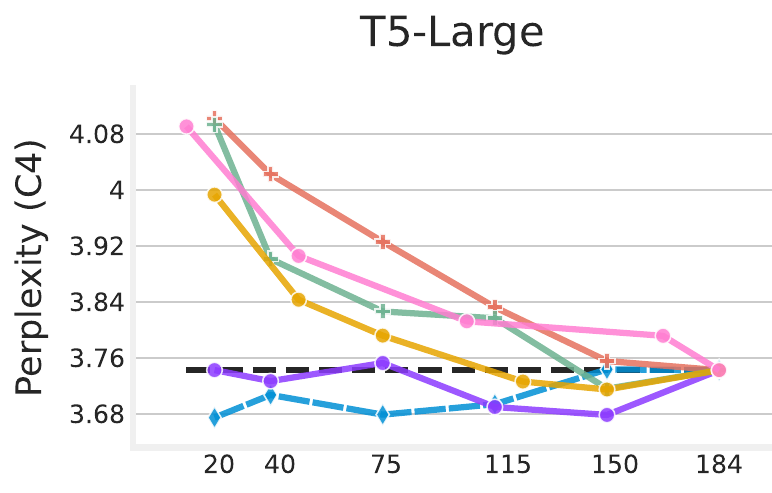}}
    \hspace{0.01\linewidth} %
    \captionsetup[subfigure]{skip=-3.37cm,slc=off,margin={54pt,0pt},font={color=red}}
    \subcaptionbox{\label{fig:sampling_ratio_vs_perf_large_perplexity_hq}}{\includegraphics[width=0.31\linewidth,trim={13.5cm 4cm 13.5cm 0},clip]{figures/sampling_ratio_vs_perf_T5-Large.pdf}}
    \hspace{0.01\linewidth} %
    \captionsetup[subfigure]{skip=-3.37cm,slc=off,margin={52pt,0pt},font={color=red}}
    \subcaptionbox{\label{fig:sampling_ratio_vs_perf_large_overscaling}}{\includegraphics[width=0.31\linewidth,trim={27cm 4cm 0 0},clip]{figures/sampling_ratio_vs_perf_T5-Large.pdf}}
    
    \vspace{0.2cm}
    
    \captionsetup[subfigure]{skip=-3.77cm,slc=off,margin={55pt,0pt},font={color=red}}
    \subcaptionbox{\label{fig:sampling_ratio_vs_perf_small_perplexity}}{\includegraphics[width=0.31\linewidth,trim={0 3cm 27cm 0},clip]{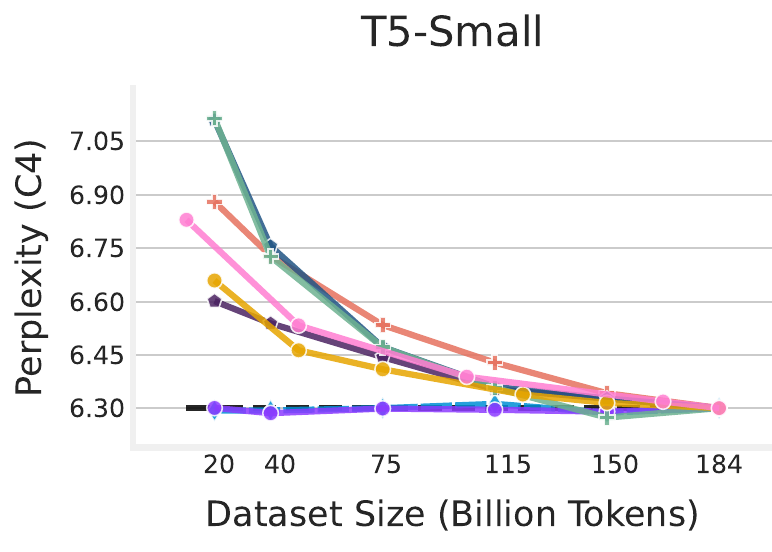}}
    \hspace{0.01\linewidth} %
    \captionsetup[subfigure]{skip=-3.77cm,slc=off,margin={54pt,0pt},font={color=red}}
    \subcaptionbox{\label{fig:sampling_ratio_vs_perf_small_perplexity_hq}}{\includegraphics[width=0.31\linewidth,trim={13.5cm 3cm 13.5cm 0},clip]{figures/sampling_ratio_vs_perf_T5-Small.pdf}}
    \hspace{0.01\linewidth} %
    \captionsetup[subfigure]{skip=-3.77cm,slc=off,margin={52pt,0pt},font={color=red}}
    \subcaptionbox{\label{fig:sampling_ratio_vs_perf_small_overscaling}}{\includegraphics[width=0.31\linewidth,trim={27cm 3cm 0 0},clip]{figures/sampling_ratio_vs_perf_T5-Small.pdf}}
    \captionsetup{subrefformat=parens}
    \vspace{0.25cm}
    \includegraphics[width=0.9\linewidth,trim={0 0 0 9cm},clip]{figures/sampling_ratio_vs_perf_T5-Small.pdf}
    \vspace{-0.2cm}
    \caption{
    Tradeoff between data quantity and model quality for T5-Small and T5-Large pre-training. Each point corresponds to a converged pre-training run over a sub-sample.
    C4 perplexity is over the in-distribution validation subset of C4, while HQ perplexity is for a higher-quality validation set (lower is better).
    Over-scaling measures the extent to which the sampling routine closes the performance gap with the next-larger (non-sampled) model (higher is better). Not all methods are shown in Figure~\ref{fig:sampling_ratio_vs_perf} or Table~\ref{tab:results}; see \cref{appendix:results}.
    }
    \label{fig:sampling_ratio_vs_perf}
\end{figure*}

\section{Experiments}

\subsection{Models}
We pre-train T5-style models \cite{t5}, which belong to the encoder-decoder family of Transformer models and offer competitive performance on many tasks~\cite{shen2023mixture}. 
See \citet{phuong2022formal} for a formal introduction to various Transformer model configurations.
We train T5-Small ($60$M) and T5-Large ($800$M), reusing all of the training settings from the original T5 implementation except the batch size ($2048\rightarrow1024)$. 
We train on batches of $1024$ sequences of length $512$ for $1$M steps.

For the quality-based data samplers (\askllm and Perplexity filtering) we use proxy quality scoring models of five different sizes: T5-\{Small, Base, Large, XL, XXL\}. For \askllm, we use FLAN-T5. For \askllm, we use FLAN-T5, which are the same sizes but have been instruction-tuned on Flan \cite{flanv2}.

\subsection{Datasets}
We use the C4 dataset\footnote{\url{www.tensorflow.org/datasets/catalog/c4}}, which was also used for pre-training the original T5. The C4 dataset is a version of the Common Crawl---a publicly available archive of web-text---that has been pre-processed using several heuristics~\citep[Section~2.2]{t5}. In its entirety, the C4 dataset contains $184$B tokens.
We use our algorithms (see \cref{appendix:samplers} for a list) to sample $\{10, 20, 40, 60, 80\}$\% of C4. 

Because a low sampling ratio yields exceedingly small datasets, we choose to train in the iso-compute setting, \ie, training all models for exactly $524$B tokens. This results in more epochs (repetitions) at smaller sampling rates. We believe this gives each data curation method an equal chance to maximize model performance, and not penalize methods that sample a small number of high-quality repeatable tokens \vs large number of non-repeatable tokens. 
See \cref{appendix:samplers}, \cref{fig:data_epochs} for a demonstration of this process.

\subsection{Evaluation} \label{sec:evaluation_metrics}
We use $111$ downstream evaluation tasks to assess diverse performance indicators for pre-trained LLMs (see \cref{appendix:evals} for a complete list). In addition to these individual tasks, to compare a \emph{normalized average performance improvement} over all downstream evaluations, we devise a metric called ``over-scaling.''

\textbf{Over-scaling (\%)}
measures the relative improvement of a model when compared against the next-largest model size, averaged over \emph{all} downstream evaluations listed in \cref{appendix:evals}. For example, a T5-Large variant with $100$\% over-scaling performs at the same level as T5-XL, while the standard T5-Large model would have an over-scaling of $0$\%.
We call this metric over-scaling because it measures the extent to which the performance exceeds the level we would expect from na\"ively scaling up the model or data.
We compute the metric by normalizing the performance improvement from sampling, \eg, for T5-Large:
$$
\expectation{\mathsf{metric}}{100 \cdot \frac{\Delta_{\mathsf{metric}}(\text{T5-L}(\mathcal{D}_{\mathsf{sampled}}), \text{T5-L}(\mathcal{D}_{\mathsf{full}})}{\Delta_{\mathsf{metric}}(\text{T5-XL}(\mathcal{D}_{\mathsf{full}}), \text{T5-L}(\mathcal{D}_{\mathsf{full}})}}
$$
where \ $\Delta_{\mathsf{metric}}(\mathbf{A}, \mathbf{B}) = \mathsf{Perf}_{\mathsf{metric}}(\mathbf{A}) - \mathsf{Perf}_{\mathsf{metric}}(\mathbf{B})$.

\input{tables/results}

\begin{figure*}[!t] 
    \centering
    \includegraphics[width=0.9\linewidth,trim={0 0cm 0 0},clip]{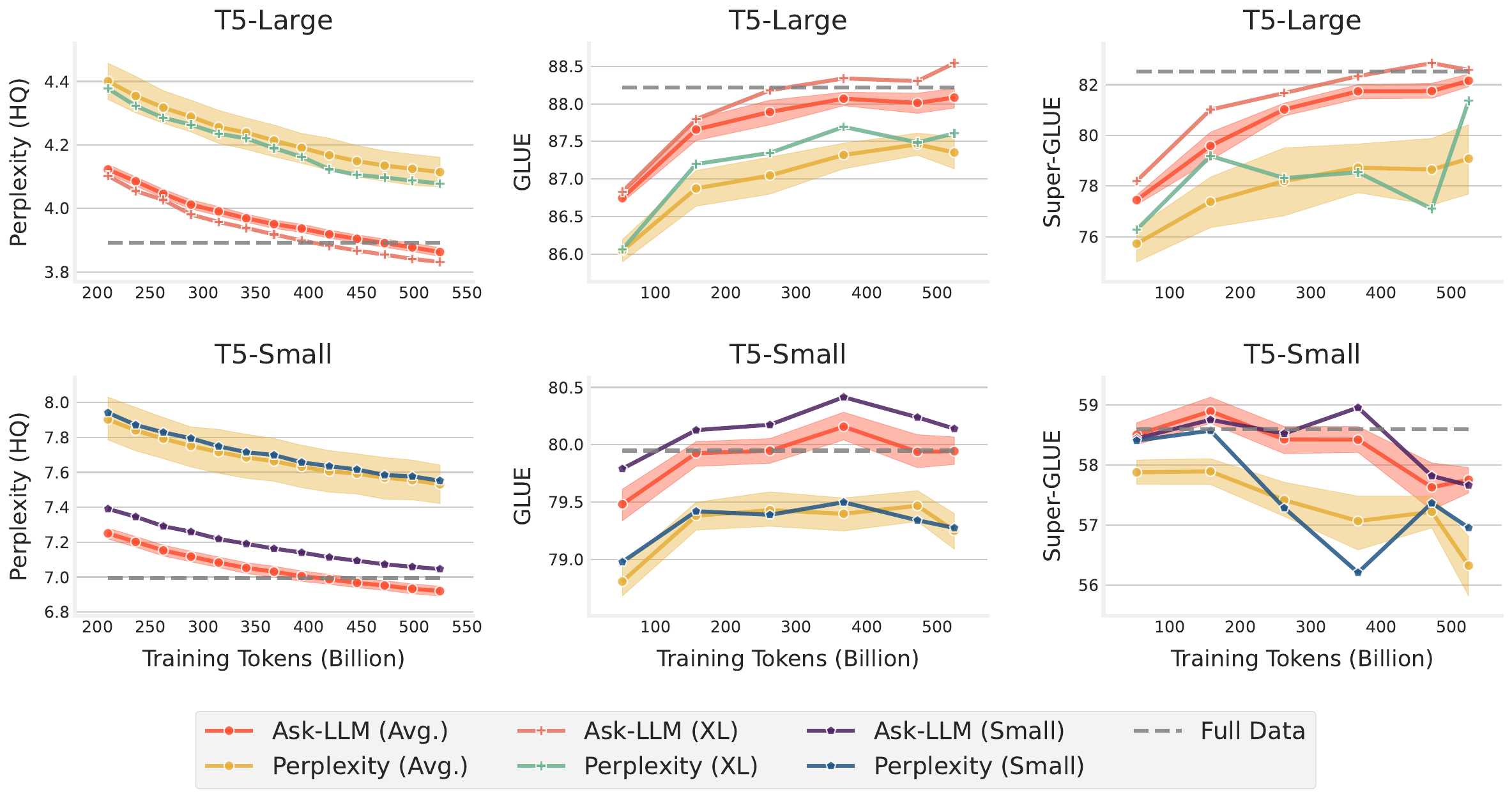} 
    
    \vspace{-0.2cm}
    \caption{
    Training efficiency comparison between two quality-score based samplers: \askllm and Perplexity filtering. \askllm (Avg.) and Perplexity filtering (Avg.) represent the training run \emph{averaged} across (i) proxy model sizes, \ie, T5-\{Small, Base, Large, XL, XXL\}; and (ii) sampling ratios, \ie, \{10, 20, 40, 60, 80\}\%. The training runs for \askllm and perplexity filtering with T5-\{Small, XL\} specifically are averaged only over the sampling ratios. Each point in this plot is the (averaged) performance of an intermediate checkpoint during the course of training on sampled data.
    }
    \vspace{-0.2cm}
    \label{fig:average_data_eff}
\end{figure*}

\begin{figure}[t!] 
    \centering
    \includegraphics[width=0.85\linewidth,trim={0 0cm 0 0},clip]{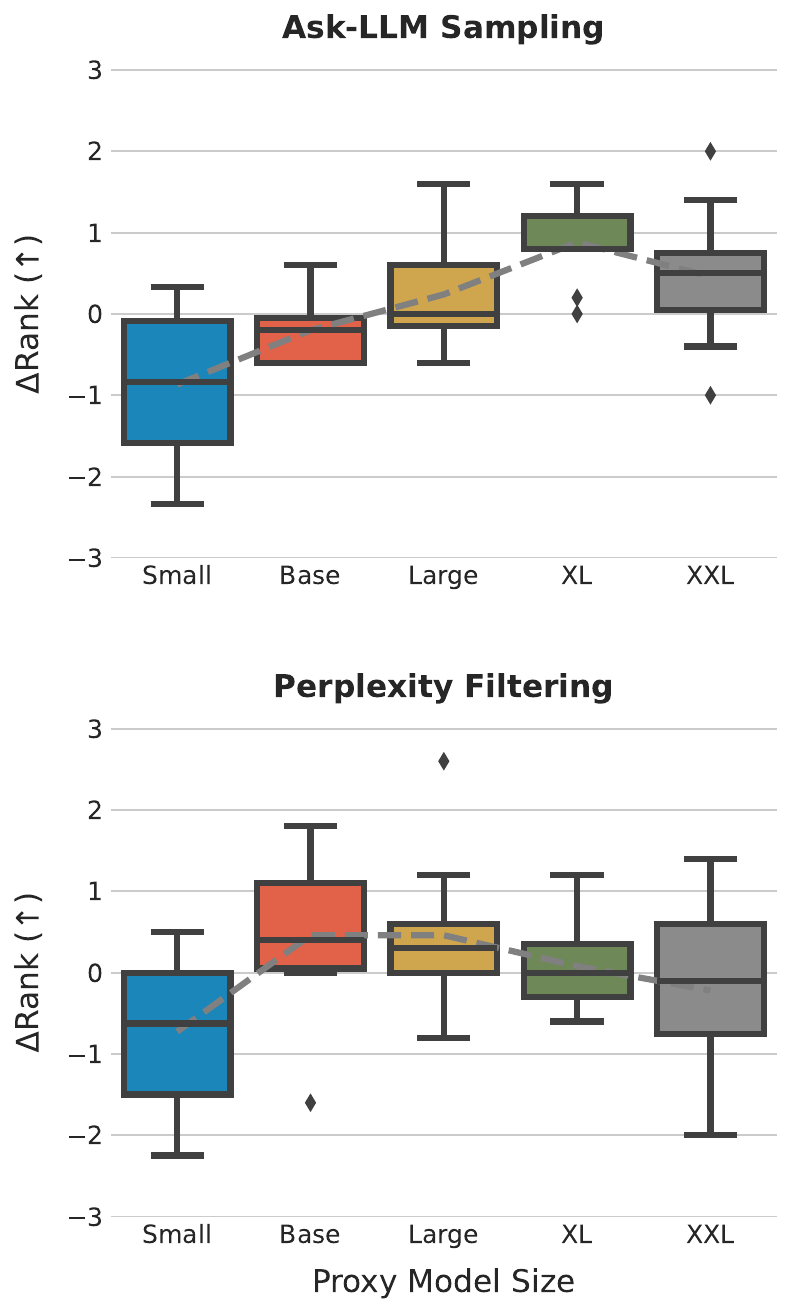} 
    
    \vspace{-0.2cm}
    \caption{
    We investigate the change in \emph{ranking} of quality-scoring models when pre-training different LLMs. 
    A positive $\Delta$Rank indicates that the scorer's task-averaged rank within \{Small, Base, Large, XL, XXL\} increased when training T5-Large \vs T5-Small.
    }
    \vspace{-0.2cm}
    \label{fig:delta_rank}
\end{figure}

\subsection{Does reasoning improve data efficiency?}
\cref{fig:sampling_ratio_vs_perf_large_overscaling} shows that \askllm closes up to 33\% of the performance gap to the next-largest model size (\ie, the over-scaling metric). 
\askllm consistently outperforms training on the full dataset as well as perplexity filtering (and coverage-maximizing baselines), despite having access to a scoring model of the same model capacity (XL). 
Similar findings hold true for training efficiency (\cref{fig:average_data_eff}). \askllm converges faster than perplexity filters, both in terms of the average (expected final performance over all proxy model sizes) and pointwise for the best configuration (Small and XL for training T5-Small and T5-Large).

\cref{fig:sampler_ranking_correlations} further demonstrates that prompting adds critical information to the sampler not present in perplexity: \askllm scores show \emph{no correlation} with the perplexity scores. 
Based on this clear behavioral difference, we conclude that reasoning and context are crucial ingredients.
We expect prompting techniques such as chain-of-thought reasoning~\cite{wei2022chain} to further drive performance.

\subsection{When are expensive quality scores justified?} 

\cref{fig:sampling_ratio_vs_perf_large_overscaling,fig:sampling_ratio_vs_perf_small_overscaling} suggest that coverage scores---especially those provided by \density---perform well in the \emph{mid-data regime} (roughly $25$\% to $50$\% sampling rate).
On the other hand, expensive quality scoring---via the \askllm procedure---is Pareto optimal for the entire quantity-quality trade-off.
The higher costs of LLM-based filters are most justified in two scenarios: (i) improving full-data performance, where quality filtering by removing the lowest-quality data is the main way to push the upper limit of model performance; or (ii) in the low-data regime, where keeping only the highest-quality data drives the most model performance compared to other sampling strategies.

We also observe that random sampling is a strong baseline, aligning with recent observations in the literature. \citet{guo2022deepcore} found that only three methods outperformed random sampling in a computer vision benchmark of 15 algorithms.
\citet{ayed2023data} prove the existence of adversarial problem instances where score-based sampling cannot outperform random sampling.
These results only serve to highlight the significance of \askllm's gains.

\subsection{Effect of quality-scoring model capacity}

\cref{fig:delta_rank} demonstrates a clear scaling trend for \askllm's quality-scoring model: larger scoring models are increasingly beneficial as the scale of the to-be-trained LLM increases.
Perplexity filters do not seem to exhibit such trends.
The strongly consistent scaling for \askllm also suggests an interesting performance-recipe: to improve downstream data-efficiency, use better quality-scoring models.
Creating better quality scorers for \askllm (via fine-tuning, chain-of-thought prompting, more capable scoring models, \etc) is thus an exciting direction for future work.

Despite the scaling trends, we would also like to emphasize that even small \askllm models provide compelling sampling performance already for both training T5-Small and T5-Large models. For example, \askllm (Small) outperforms perplexity filtering with \emph{any} scoring-model in \cref{fig:sampling_ratio_vs_perf_small_overscaling} (including T5-XXL) by a sizable margin.

\begin{figure}[t!] 
    \centering
    \includegraphics[width=0.95\linewidth,trim={0 0cm 0 0},clip]{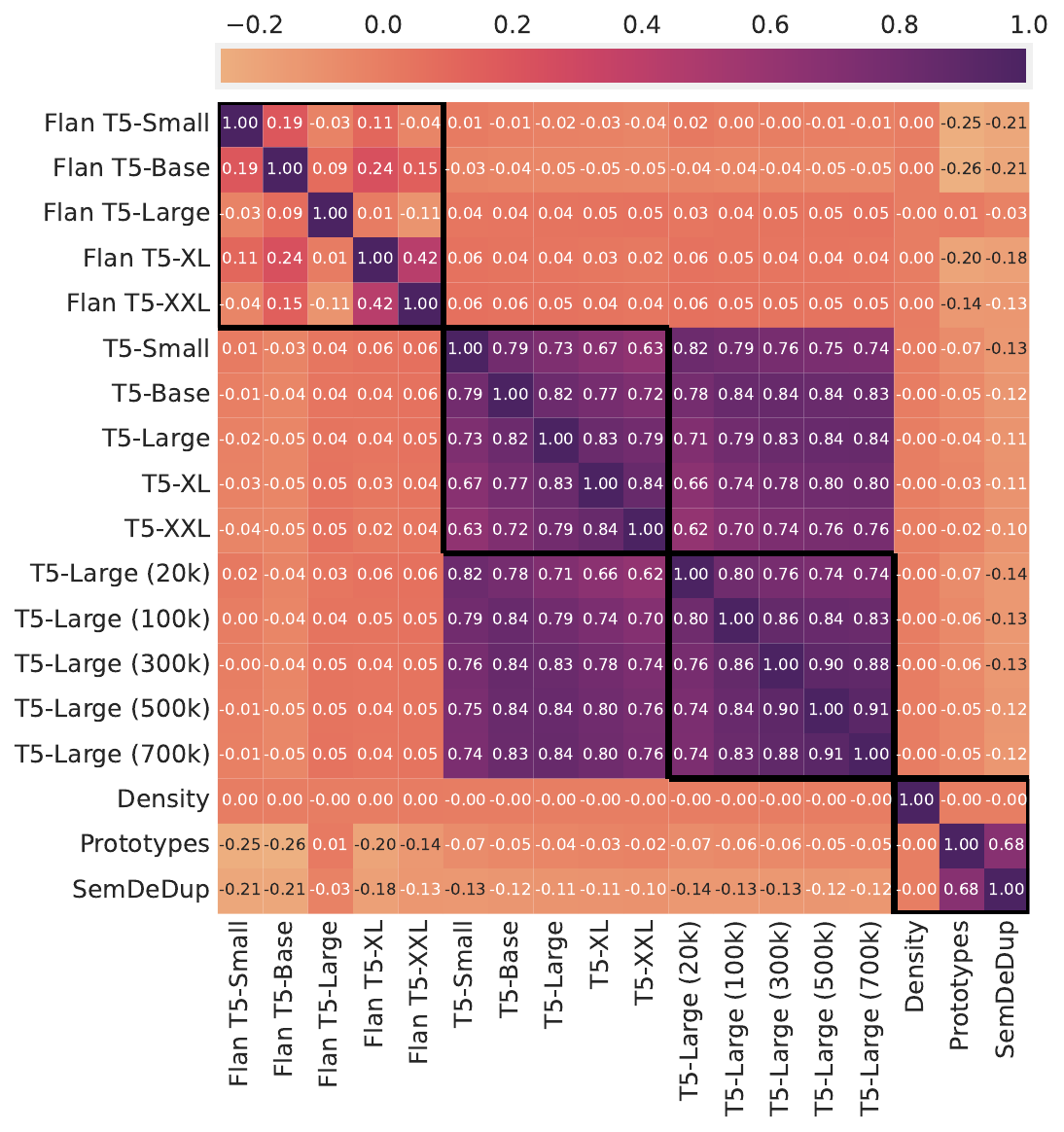} 
    
    \vspace{-0.2cm}
    \caption{
    Kendall's Tau correlation amongst the scores from models in the \askllm (first 5) and perplexity filtering (next 10) frameworks over $500$k randomly selected training samples.
    }
    \vspace{-0.2cm}
    \label{fig:sampler_ranking_correlations}
\end{figure}

\subsection{Do samplers prioritize different examples?}

To understand whether different algorithms prioritize different examples, we sorted examples by score and computed the Kendall Tau rank correlation between samplers (\cref{fig:sampler_ranking_correlations}). We find that samplers differ in significant and interesting ways. For example, the ``T5-Large'' row shows that (i) T5-Large outputs perplexity scores similar to T5-Small early in training, but becomes progressively more nuanced on the path from 20k to 700k training steps, and (ii) perplexity and \askllm select for wildly different criteria, with almost no ranking correlation.

\density prioritizes coverage over de-noising, maintaining the in-distribution test perplexity better than any other strategy (\cref{fig:sampling_ratio_vs_perf_large_perplexity,fig:sampling_ratio_vs_perf_small_perplexity}).
This suggests that coverage sampling preserves the objective function, in contrast with other methods that preferentially select for quality in addition to diversity.

\section{Discussion}

\textbf{Amortized scoring.} The \askllm and perplexity scorers require considerable computation---one LLM inference call for every training sample---which is concerning from both a carbon-emissions and cost perspective~\cite{strubell2019energy}.
However, we argue that the scoring costs are \emph{amortized over many pre-training runs}, which together cost significantly more than the \askllm inference calls~\cite{luccioni2023estimating}.
In practical systems, cheaper samplers / scoring models can also pre-filter examples for our more expensive scorers.
While LLM pre-training is often thought of as a one-time cost, this has historically not been the case.
We therefore view quality scores as a long-term investment.
See \cref{appendix:askkllm_cost} for a deeper discussion about the cost of \askllm scoring.

\textbf{LLM-Based Data Refinement.} Recursively training on model-generated data causes degredation in both diffusion models and LLMs, inciting concerns about whether the internet will remain a viable source of training data~\cite{shumailov2023curse,alemohammad2023self,briesch2023large}.
It is therefore somewhat surprising that LLMs are so effective at deciding which training data to consume.
Our \askllm results raise important questions about whether LLM-based filters can function as an intervention in the self-consumption loop, allowing LLMs to self-improve.

\section{Conclusion}
We studied the performance of sampling algorithms that select high-quality data through highly-capable proxies and maximize coverage through embedding similarity.
Our experiments reveal that LLM-based quality filtering yields a Parteo optimal efficiency tradeoff between data quantity and model quality, with important implications for training cost, self-improvement, and LLM training data curation. 

\section*{Impact Statement} While increased LLM accessibility has well-documented risks, we expect data-efficient pre-training to be a net social good that reduces (amortized) carbon emissions and pre-training cost while improving quality.

\section*{Acknowledgements}
We sincerely thank Xinyun Chen and Kelvin Guu for their insightful feedback on early drafts of this paper.

\bibliography{references}
\bibliographystyle{icml2024}

\appendix
\onecolumn
\newpage
\appendixpage
{
    \hypersetup{linkcolor = [RGB]{0,0,130}}
    \startcontents[sections]
    \printcontents[sections]{l}{1}{\setcounter{tocdepth}{3}}
}
\newpage
\input{appendix}


\end{document}

%% file: defintions.tex
\newcommand{\argmin}[1]{\underset{#1}{\operatorname{arg}\,\operatorname{min}}\;\ }
\newcommand{\expectation}[2]{\underset{#1}{\mathbb{E}} \left[#2\right]}

\newcommand{\density}{\textsc{Density}\xspace}
\newcommand{\askllm}{\textsc{Ask-LLM}\xspace}

\newcommand{\EE}{\operatornamewithlimits{\mathbb{E}}} 

\makeatletter \renewcommand\paragraph{\@startsection{paragraph}{4}{\z@} {2mm \@plus1ex \@minus.2ex} {-0.7em} {\normalfont\normalsize\bfseries}} \makeatother

\newcommand{\listheader}[1]{\item \textbf{#1} }

\newcommand\overstar[1]{\ThisStyle{\ensurestackMath{%
  \setbox0=\hbox{$\SavedStyle#1$}%
  \stackengine{0pt}{\copy0}{\kern.2\ht0\smash{\SavedStyle*}}{O}{c}{F}{T}{S}}}}

\theoremstyle{plain}
\newtheorem{theorem}{Theorem}[section]
\newtheorem{proposition}[theorem]{Proposition}
\newtheorem{lemma}[theorem]{Lemma}
\newtheorem{corollary}[theorem]{Corollary}
\theoremstyle{definition}
\newtheorem{definition}[theorem]{Definition}
\newtheorem{assumption}[theorem]{Assumption}
\theoremstyle{remark}
\newtheorem{remark}[theorem]{Remark}

\newcommand{\STAB}[1]{\begin{tabular}{@{}c@{}}#1\end{tabular}}

\newcommand{\bs}[1]{\ensuremath{\bm{\mathit{#1}}}}

\newcommand{\ie}[0]{\emph{i.e.}\xspace}
\newcommand{\st}[0]{\emph{s.t.}\xspace}
\newcommand{\etc}[0]{\emph{etc.}\xspace}
\newcommand{\eg}[0]{\emph{e.g.}\xspace}
\newcommand{\vs}[0]{\emph{vs.}\xspace}
\newcommand{\wrt}[0]{\emph{w.r.t.}\xspace}

%% file: tables/results.tex
\def\arraystretch{1.1}
\setlength{\tabcolsep}{0.4em} 

\begin{table*}[t!] \centering
  \caption{
  Comparison of sampling algorithms at a fixed sample size. For each sampling strategy, we
  sample the dataset to X\% of the original size and pre-train T5-Small and T5-Large for $524$B tokens. This table is a cross-section of \cref{fig:sampling_ratio_vs_perf} but with more metrics.
  }
  \vspace{0.1cm}
  \label{tab:results}
    \begin{scriptsize}
    \begin{center}
        \begin{tabular}{c|cc|c|cccc|cccc}
            \toprule[1pt]
            \multirow{2.7}{*}{\textbf{LLM}} & \multicolumn{2}{c|}{\textbf{Training config.}} & \multirow{2.7}{*}{Over-scaling (\%)} & \multicolumn{4}{c|}{\textbf{Downstream tasks}} & \multicolumn{4}{c}{\textbf{FLAN Instruction Tuning}} \\[3pt] 
            & Sampler & \# Tokens & & GLUE & SuperGLUE & CNN/DM & SQuAD & MMLU & BBH & Reasoning & QA \\
            
            \midrule[0.75pt]
            
            T5-Small & --- & 184B & 0.0 & 79.9 & 58.6 & \textbf{18.6} & 78.6 & 25.5 & 28.5 & 15.2 & 37.0  \\
            
            \midrule[0.01pt]
            
            T5-Small & Random & 36B ($\equiv$20\%) & -0.2 &	79.9 &	58.3 &	\textbf{18.6} &	78.1 &	26.9 &	27.8 &	15.2 &	38.1 \\
            
            \midrule[0.01pt]
            
            T5-Small & Density & 36B ($\equiv$20\%) & -2.1 &	80.5 &	59.7 &	18.5 &	78.4 &	28.1 &	\textbf{30.3} &	14.5 &	33.4 \\
            
            T5-Small & SemDeDup & 46B ($\equiv$25\%) & -4.5 &	\textbf{80.7} &	59.2 &	18.4 &	77.8 &	28.0 &	26.6 &	14.8 &	\textbf{37.0} \\
            
            T5-Small & Prototypes & 46B ($\equiv$25\%) & -8.0 &	79.7 &	58.8 &	18.5 &	78.0 &	26.8 &	27.7 &	15.7 &	34.2 \\
            
            \midrule[0.01pt]
            
            T5-Small & Perplexity (Small) & 36B ($\equiv$20\%) & -7.8 &	79.9 &	58.4 &	18.4 &	77.5 &	28.1 &	28.2 &	15.0 &	35.0 \\
            
            T5-Small & Ask-LLM (XL) & 36B ($\equiv$20\%) & \textbf{4.2} &	80.3 &	\textbf{59.8} &	\textbf{18.6} &	\textbf{79.1} &	\textbf{29.9} &	28.5 &	\textbf{15.8} &	36.4 \\
            
            \midrule[1pt]
            
            T5-Large & --- & 184B & 0.0 & 88.2 & 82.5 & 20.8 & 86.7 & 40.7 & 33.6 & 21.6 & 73.0 \\
            
            \midrule[0.01pt]
            
            T5-Large & Random & 36B ($\equiv$20\%) & -6.5 &	88.6 &	82.8 &	20.7 &	86.1 &	43.3 &	34.8 &	18.6 &	70.1 \\
            
            \midrule[0.01pt]
            
            T5-Large & Density & 36B ($\equiv$20\%) & 2.8 &	\textbf{88.8} &	82.4 &	20.8 &	86.4 &	41.4 &	35.4 &	19.4 &	72.8 \\
            
            T5-Large & SemDeDup & 46B ($\equiv$25\%) & -20.5 &	88.3 &	81.4 &	20.7 &	86.0 &	41.2 &	\textbf{36.7} &	\textbf{21.8} &	70.2 \\
            
            T5-Large & Prototypes & 46B ($\equiv$25\%) & 0.2 &	88.4 &	82.7 &	20.8 &	87.0 &	40.0 &	35.5 &	17.6 &	71.1 \\
            
            \midrule[0.01pt]
            
            T5-Large & Perplexity (XL) & 36B ($\equiv$20\%) & -32.7 &	87.9 &	81.8 &	20.6 &	85.7 &	38.1 &	33.9 &	20.0 &	69.0 \\
            
            T5-Large & Ask-LLM (XL) & 36B ($\equiv$20\%) & \textbf{33.0} &	\textbf{88.8} &	\textbf{83.0} &	\textbf{21.0} &	\textbf{87.3} &	\textbf{43.6} &	33.0 &	20.0 &	\textbf{77.1} \\
            
            \bottomrule[1pt]
        \end{tabular}
    \end{center}
    \end{scriptsize}
\end{table*}

%% file: appendix.tex
\section{Algorithms} \label{appendix:algorithms}

\subsection{\askllm Sampling} \label{appendix:askkllm_cost}
\input{algorithms/askllm}

\paragraph{Discussion on the cost of \askllm scoring.} 
Even though \askllm sampling results in impressive performance and training efficiency improvements compared to training on the full-dataset (\cref{appendix:results}), the data quality scoring cost might seem prohibitive. On the other hand, on top of the improved results, we argue the following to be compelling points in justifying \askllm's one-time-amortized data scoring cost:
\begin{itemize}[leftmargin=*]
    \item \askllm only requires \emph{forward passes} on the entire dataset. This is much cheaper than (i) training the model itself which requires both forward and backward passes on multiple repetitions of the entire dataset, (ii) gradient-based data-curation techniques \cite{dd_survey, farzi} that also require backward passes, \etc 
    
    \item An additional benefit of the \askllm framework is the ability to leverage memory-efficient, quantized LLM inference setups \cite{llm_int8}. This is strictly not possible, \eg, for pre-training LLMs. Notably, quantization isn't the only \askllm-friendly technique. All the recent (and future) advances in efficient \emph{inference} techniques for LLMs \cite{efficient_inference_blog} directly reduce the amortization cost of the \askllm framework.
    
    \item Another benefit of \askllm is the ability to na\"ively parallelize quality scoring. To be more specific, we can simply scale-up the amount of \emph{small \& independent} inference resources, and run inference calls for various training samples parallely. Note that inference hardware has much smaller requirements compared to, \eg, pre-training or fine-tuning requirements. This is primarily true because of no batch size requirement for inference \vs large batch size requirement while training. This enables scaling-up hardware to happen via a large number of small-compute setups (\eg, 4 interconnected GPUs per node) versus increasing the number of large-compute setups (\eg, $1000$s of interconnected GPUs per node).
    
    \item \askllm also uses strictly less compute compared to teacher-student knowledge distillation based training setups \cite{llm_kd}. This is true simply because knowledge distillation require (i) bigger teacher model's softmax predictions (ii) for each token in our training data. On the other hand, \askllm requires just the score of the token ``\texttt{yes}'' given the prompt.
\end{itemize}

\subsection{\density Sampling}
Our density sampler is adapted from that of~\citet{density}, with a few critical departures:
\begin{itemize}
    \setlength\itemsep{0em}
    \item We use a two-pass procedure that allows for more rigorous theoretical guarantees (and different sampling behavior).
    \item We conduct the density estimation in the model's latent space rather than using Jaccard similarity over $n$-grams.
\end{itemize}
\textbf{Improvements:} Jaccard similarities are sufficient to construct a reasonable sampling distribution for genomics applications, which are significantly more structured than natural language. 
However, this is not the case with text --- we found that sampling based on Jaccard density is no better than random.
For this reason, we must use different kernels ($p$-stable rather than MinHash) and different input representations (embedding rather than $n$-grams).

However, our more interesting departure from~\citet{density} is our two-pass sampling procedure, which changes the behavior of the algorithm and allows for more rigorous theoretical guarantees.
The original method was only able to demonstrate convergence of cluster populations in the sampled dataset.
While this leads to (weak) convergence for some measures of diversity, it also requires strong assumptions about the cluster structure.

\textbf{Theory:} We use a recent result that demonstrates consistent sketch-based estimation of the kernel sum (Theorem 3.3 of ~\citet{liu2023one}), which we paraphrase below.

\begin{lemma}\label{thm:non_private_consistency} 
Let $P(x)$ denote a probability density function. Let $\mathcal{D} \underset{\mathrm{iid}}{\sim} P(x)$ denote a dataset. Let $k(x, y)$ be a positive definite LSH kernel, and let $S$ be the \density score. Then $S(x)$ is a consistent estimator for the kernel sum.
$$S(x) \underset{\mathrm{i.p.}}{\to} \frac{1}{N} \sum_{x_i \in \mathcal{D}} k(x_i, q)$$
with convergence rate $O(\sqrt{ \log R / R})$.
\end{lemma}

If we perform inverse propensity sampling using the score in \cref{thm:non_private_consistency}, we obtain a sampling procedure that outputs a uniformly-distributed sample.

\begin{theorem}\label{thm:converge_to_uniform} 
Let $Q(x)$ be the distribution formed by (i) drawing $N$ samples i.i.d. from a distribution $P$, \eg $\mathcal{D} = \{x_1, ... x_N\} \sim P$, and (ii) keeping $x$ with probability proportional to $\frac{1}{S(x)}$. Under the conditions of Lemma~\ref{thm:non_private_consistency}, $Q(x) \underset{\mathrm{i.p.}}{\to} U(x)$, where $U(x)$ is the uniform distribution.
\end{theorem}
\begin{proof}
Under the conditions of~\citet{wied2012consistency} (specifically, positive-definiteness and $\ell_1$ integrability / bounded domain), the kernel sum is a consistent estimator of the density. That is, the sum converges in probability to $P(x)$.
$$ \frac{1}{N} \sum_{x_i \in \mathcal{D}} k(x_i, q) \underset{\mathrm{i.p.}}{\to} P(x)$$
\cref{thm:non_private_consistency} shows that $S(x)$ converges in probability to the sum (and thus to $P(x)$). By Slutsky's Theorem, $\frac{1}{S(x)} \to \frac{1}{P(x)}$ for all $x$ in the support of the distribution (i.e. $P(x) \neq 0$). The probability of generating $x$ as part of the sample is:
$$Q(x) = \mathrm{Pr}[\mathrm{Select }x\cap \mathrm{Generate }x] = \mathrm{Pr}[\mathrm{Select }x] \mathrm{Pr}[\mathrm{Generate }x] = \frac{1}{S(x)} P(x)$$

Because $\frac{1}{S(x)} \to \frac{c}{P(x)}$ for some constant $c$, we have that $Q(x) \to c$.
\end{proof}

\cref{thm:converge_to_uniform} demonstrates that our \density sampler outputs a uniformly-distributed collection of points over the input space (latent LLM representation space).

\input{algorithms/density}

\textbf{Cost:} Like SemDeDup, D4, and SSL prototypes, our \density sampler requires access to embeddings for each example in the training corpus.
However, by eliminating the expensive clustering step, we eliminate a significant computational overhead.
Our \density sampling routine required just $80$MB of memory and two linear passes through the dataset to score all $364$M embeddings.
This is significantly less expensive than clustering.

\textbf{Tuning:} We also eliminate a large number of hyperparameters, improving tuning. Cluster-based samplers must choose the number of clusters, clustering optimizer and objective, and per-cluster sampling rate or deduplication similarity.
Kernel density estimation, on the other hand, has just \textit{two} hyperparameters: the choice of kernel and the bandwidth.
We did not observe a significant performance variation among different bandwidth and kernel choices (e.g., the L2 and cosine kernels of~\citet{coleman2020sub} perform nearly identically).
This is likely because all positive-definite kernels enjoy strong guarantees on the distribution approximation error~\cite{devroye1983equivalence}.

\section{Data-curation Techniques} \label{appendix:samplers}

\begin{figure}[t!] 
    \centering
    \vspace{-0.2cm}
    \includegraphics[width=0.7\linewidth,trim={0.8cm 0.1cm 0.8cm 1.25cm},clip]{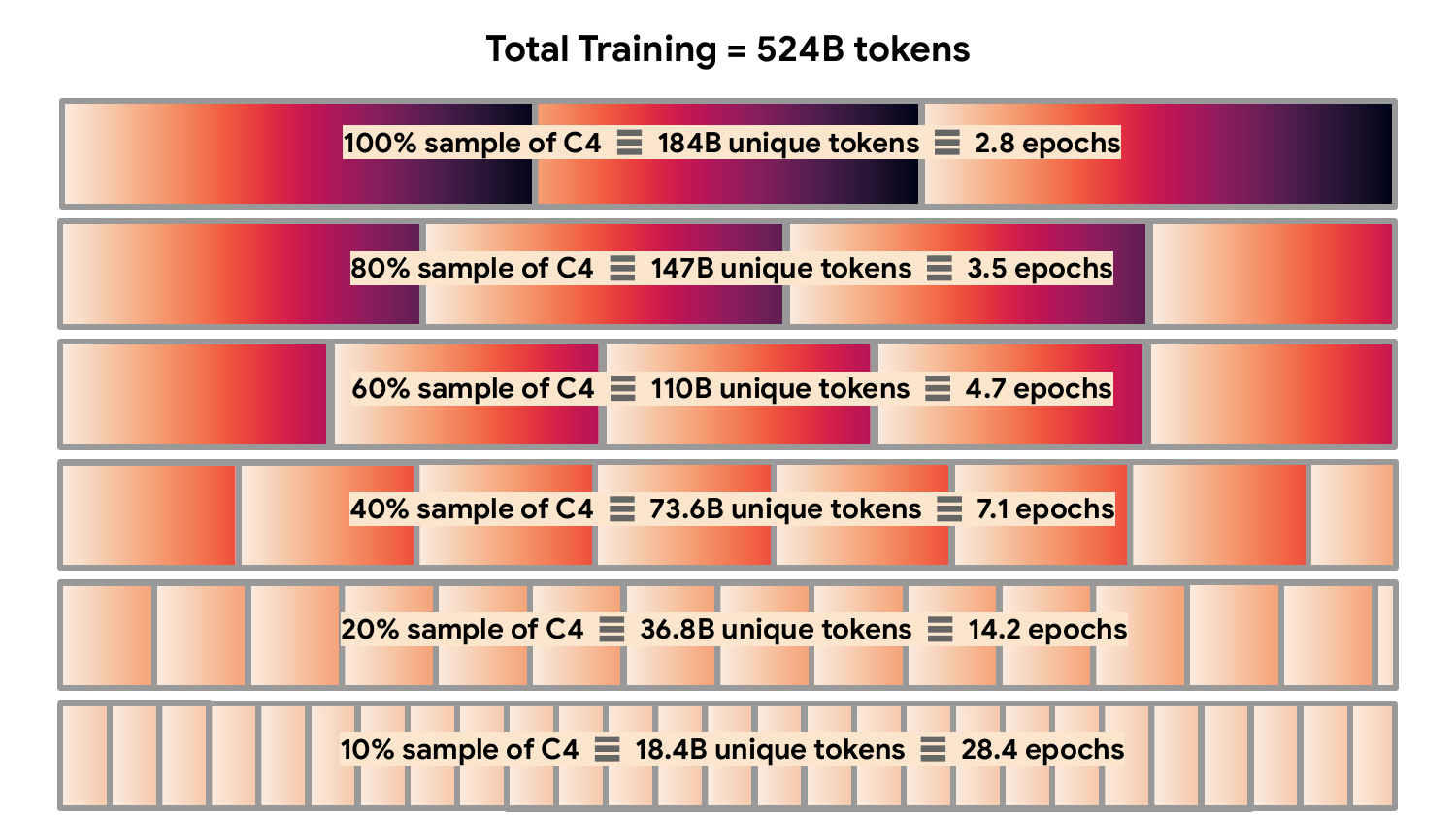} 
    \caption{We consider a setup where all of our models are trained on exactly $524$B tokens, causing us to repeat the same examples for more epochs when we downsample.
    We borrow the format of this graphic explanation from~\citet{muennighoff2023scaling}, who consider a similar setting.
    }
    \vspace{-0.2cm}
    \label{fig:data_epochs}
\end{figure}

\subsection{Random Sampling}
The de-facto standard for obtaining samples of large datasets where we sample training examples uniformly at random. Notably, random sampling has also been accompanied with strong results in a variety of applications in the data-curation literature primarily due to its unbiased sampling \cite{uniform_sampling_good_1, uniform_sampling_good_2}.

\subsection{\density Sampling} \label{appendix:density}
See \cref{sec:density_details} for technical details about the \density sampler. We use Sentence-T5-Base \cite{sentence_t5} as our embedding model for training samples, primarily due to its contrastive training, giving confidence for computing distances amongst its $768$-$\dim$ embeddings. We use the PStable hash \cite{pstable} to hash the embeddings, along with a $[1,000 \times 20,000]$ sketch matrix.

\subsection{SemDeDup}
The key idea is to perform (coverage maximizing) semantic deduplication inside clusters of the original dataset \cite{semdedup}. We re-use the Sentence-T5-Base embeddings of data-points (\cref{appendix:density}), and perform $k$-means clustering to obtain $10,000$ clusters of the entire dataset. 

\subsection{SSL Prototypes}
They key idea is to remove \emph{prototypical} points in a dataset \cite{prototypes}. As a meaningful proxy, this method removes the points closest to cluster centroids of a dataset. 
For brevity, we use the name ``Prototypes'' when reporting our results.
We re-use the same embeddings and clustering for both SemDeDup and Prototypes.

\subsection{Perplexity Filtering}
A popular quality-filtering approach in the literature is to use the perplexity of proxy language models to filter data-points with a high-perplexity under that language model. While the literature historically used small language models for perplexity filtering \cite{wenzek2019ccnet,muennighoff2023scaling}, recent work \cite{marion2023less} suggests improved filtering performance when using LLMs for this task. To this end, we employ perplexity filtering with T5-\{Small, Base, Large, XL, XXL\} models; as well as intermediate checkpoints during the course of training T5-Large: \{$20$k, $100$k, $300$k, $500$k, $700$k\}.

\subsection{\askllm Sampling}
See \cref{sec:askllm_details} for technical details about the \askllm sampler. Since \askllm relies on the reasoning capabilities of instruction-tuned models, we use the Flan-T5-\{Small, Base, Large, XL, XXL\} \cite{flanv2} models for obtaining the quality scores in \askllm.

\section{Downstream Evaluation Tasks} \label{appendix:evals}

\subsection{Perplexity} Defined as the exponentiated average negative log-likelihood of an average sequence in the dataset; we compute the perplexity over the default validation set in C4. Note that C4's validation set is a random sample of the dataset, so it is prone to be of much lower quality than curated sources, and hence, a less reliable indicator of true model quality.

\subsection{HQ Perplexity} As our best effort to devise an inexpensive-to-compute metric that is better aligned with model quality than perplexity on C4's validation set, inspired by the evaluation conducted in \citet{tirumala2023d4}, we construct a \emph{high-quality} validation set from non web-scrape sources. We collate the validation sets from (1) English portion of wiki40b \cite{wiki40b}, (2) realnews and webtext subsets of C4, and (3) news commentary from the LM1B dataset \cite{lm1b}.

\subsection{GLUE} A popular natural language understanding meta-benchmark comprising of eleven different tasks \cite{glue}. Note that we report the average score for all individual tasks, after finetuning on the concatenation of all individual tasks' training sets, as is done in the original T5 implementation.

\subsection{SuperGLUE} A harder meta-benchmark (\vs GLUE) built to further test the natural language understanding abilities of language models \cite{superglue}. Similar to GLUE, we report the average score of all tasks, and conduct fine-tuning on all tasks' concatenated train-set.

\subsection{CNN/DM} We use the CNN/DM dataset \cite{cnndm} for testing our models' abstractive summarization abilities. Like the T5 original setting, we finetune on the train-set, and report the ROUGE-2 scores.

\subsection{SQuAD} A popular dataset \cite{squad} used to evaluate question-answering capabilities of language models, we compare the finetuned performance of our models using exact-match as the metric.

\subsection{FLAN Instruction Tuning} A popular application of LLMs has been instruction-following, and chatting capabilities. To test our model's quality on this front, we finetune our models on the FLANv2 dataset \cite{flanv2}, and test the instruction-tuned models' performance from four fronts: 
\begin{itemize}[leftmargin=*]
    \item $5$-shot MMLU \cite{mmlu}: a popular benchmark consiting of exam questions from $57$ tasks.
    \item $3$-shot Big Bench Hard (BBH) \cite{bigbench}: a popular set of $23$ hardest tasks from big bench.
    \item Reasoning: macro-average $8$-shot performance on GSM8k \cite{gsm8k}, SVAMP \cite{svamp}, ASDIV \cite{asdiv}, and StrategyQA \cite{strategyqa} benchmarks.
    \item QA: macro-average $0$-shot performance on UnifiedQA \cite{unifiedqa}, BoolQ \cite{boolq}, Arc-Easy and Arc-Challenge \cite{arc_eval} benchmarks.
    \item Average: macro-average of all the four benchmarking suites listed above: MMLU, BBH, Reasoning, and Q/A.
\end{itemize}
Please note that all of our reported numbers are based on \emph{single checkpoint} evaluations, \ie, we first select the best checkpoint during FLAN finetuning using the \emph{average} performance on all tasks, and report the individual task performance on that checkpoint.

\section{Additional Results} \label{appendix:results}

\subsection{(\cref{fig:score_distribution}) Quality-score Distribution for Different Samplers} 

For different data curation techniques listed in \cref{appendix:samplers}, we examine the distribution of estimated \emph{data-quality} scores normalized in a way that higher represents better data quality.
\begin{itemize}[leftmargin=*]
    \item For the \density sampler, the plotted score is proportional to the likelihood of the example under the kernel density estimate.
    
    \item For the Prototypes sampler, the plotted score represents the negated cosine similarity of data-point with its assigned cluster centroid.
    
    \item For the SemDeDup sampler, the plotted score represents the negated maximum cosine similarity of a datapoint to all other datapoints in its respective cluster.
    
    \item For the perplexity filtering sampler, the plotted score represents the negated perplexity of a training sample.
    
    \item For the \askllm sampler, the plotted score represents the log probability of the token ``\texttt{yes}'' given the prompt in \cref{fig:ask_llm_prompt}.
\end{itemize}

\begin{figure}[t!] 
    \centering
    \includegraphics[width=\linewidth,trim={0 0 0 0},clip]{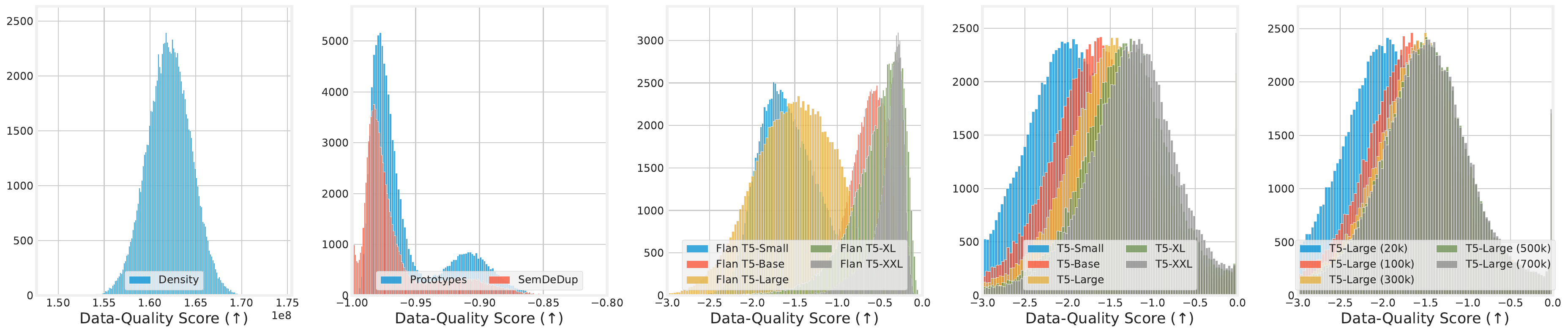} 
    \vspace{-0.2cm}
    \caption{Score distribution of various data curation techniques. The plots for Flan-T5-* models are for \askllm, whereas ones using T5-* models are for perplexity filtering.}
    \vspace{-0.2cm}
    \label{fig:score_distribution}
\end{figure}

\subsection{(\cref{fig:unique_tokens_vs_model_quality_small_0,fig:unique_tokens_vs_model_quality_small_1,fig:unique_tokens_vs_model_quality_small_2,fig:unique_tokens_vs_model_quality_large_0,fig:unique_tokens_vs_model_quality_large_1,fig:unique_tokens_vs_model_quality_large_2,fig:unique_tokens_vs_model_quality_large_3}) Data-quantity \vs Model-quality for Different Samplers} 

For different data curation techniques listed in \cref{appendix:samplers}, we investigate the tradeoff between the sampling rate and the respectively trained model's quality on various downstream evaluations listed in \cref{appendix:evals}. We plot our results in the following figures:
\begin{itemize}[leftmargin=*]
    \item (\cref{fig:unique_tokens_vs_model_quality_small_0}) \textbf{T5-Small, coverage}: Pre-training T5-Small on different amounts of data sampled by \{Random sampling, \density sampling, Self-supervised Prototypes sampling, SemDeDup\}.
    
    \item (\cref{fig:unique_tokens_vs_model_quality_large_0}) \textbf{T5-Large, coverage}: Pre-training T5-Large on different amounts of data sampled by \{Random sampling, \density sampling, Self-supervised Prototypes sampling, SemDeDup\}.
    
    \item (\cref{fig:unique_tokens_vs_model_quality_small_1}) \textbf{T5-Small, \askllm}: Pre-training T5-Small on different amounts of data sampled by \askllm using the \{Flan-T5-Small, Flan-T5-Base, Flan-T5-Large, Flan-T5-XL, Flan-T5-XXL\} scoring models.
    
    \item (\cref{fig:unique_tokens_vs_model_quality_large_1}) \textbf{T5-Large, \askllm}: Pre-training T5-Large on different amounts of data sampled by \askllm using the \{Flan-T5-Small, Flan-T5-Base, Flan-T5-Large, Flan-T5-XL, Flan-T5-XXL\} scoring models.
    
    \item (\cref{fig:unique_tokens_vs_model_quality_small_2}) \textbf{T5-Small, Perplexity filtering}: Pre-training T5-Small on different amounts of data sampled by Perplexity filtering using the \{T5-Small, T5-Base, T5-Large, T5-XL, T5-XXL\} scoring models.
    
    \item (\cref{fig:unique_tokens_vs_model_quality_large_2}) \textbf{T5-Large, Perplexity filtering}: Pre-training T5-Large on different amounts of data sampled by Perplexity filtering using the \{T5-Small, T5-Base, T5-Large, T5-XL, T5-XXL\} scoring models.
    
    \item (\cref{fig:unique_tokens_vs_model_quality_large_3}) \textbf{T5-Large, Perplexity filtering}: Pre-training T5-Large on different amounts of data sampled by Perplexity filtering using the \{$20$k, $100$k, $300$k, $500$k, $700$k\} intermediate checkpoints of T5-Large as data quality scoring models.
\end{itemize}

\begin{figure}[H] 
    \centering
    \includegraphics[width=\linewidth,trim={0 0 0 0},clip]{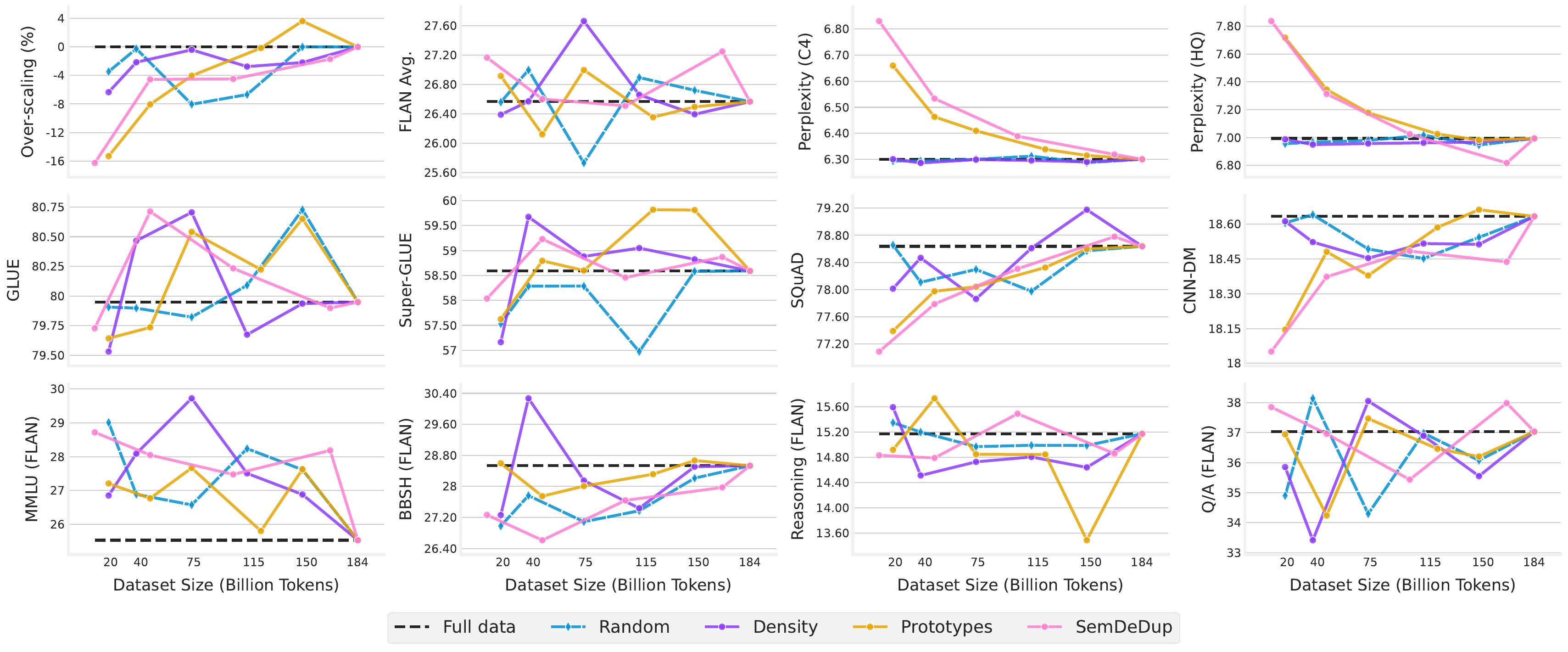} 
    \vspace{-0.2cm}
    \caption{Tradeoff between data quantity and model quality while pre-training T5-Small. Each point in this plot comes from the converged pre-training run over a sampled dataset. See \cref{appendix:evals} for a description about the metrics used in this plot.}
    \vspace{-0.2cm}
    \label{fig:unique_tokens_vs_model_quality_small_0}
\end{figure}

\begin{figure}[H] 
    \centering
    \includegraphics[width=\linewidth,trim={0 0 0 0},clip]{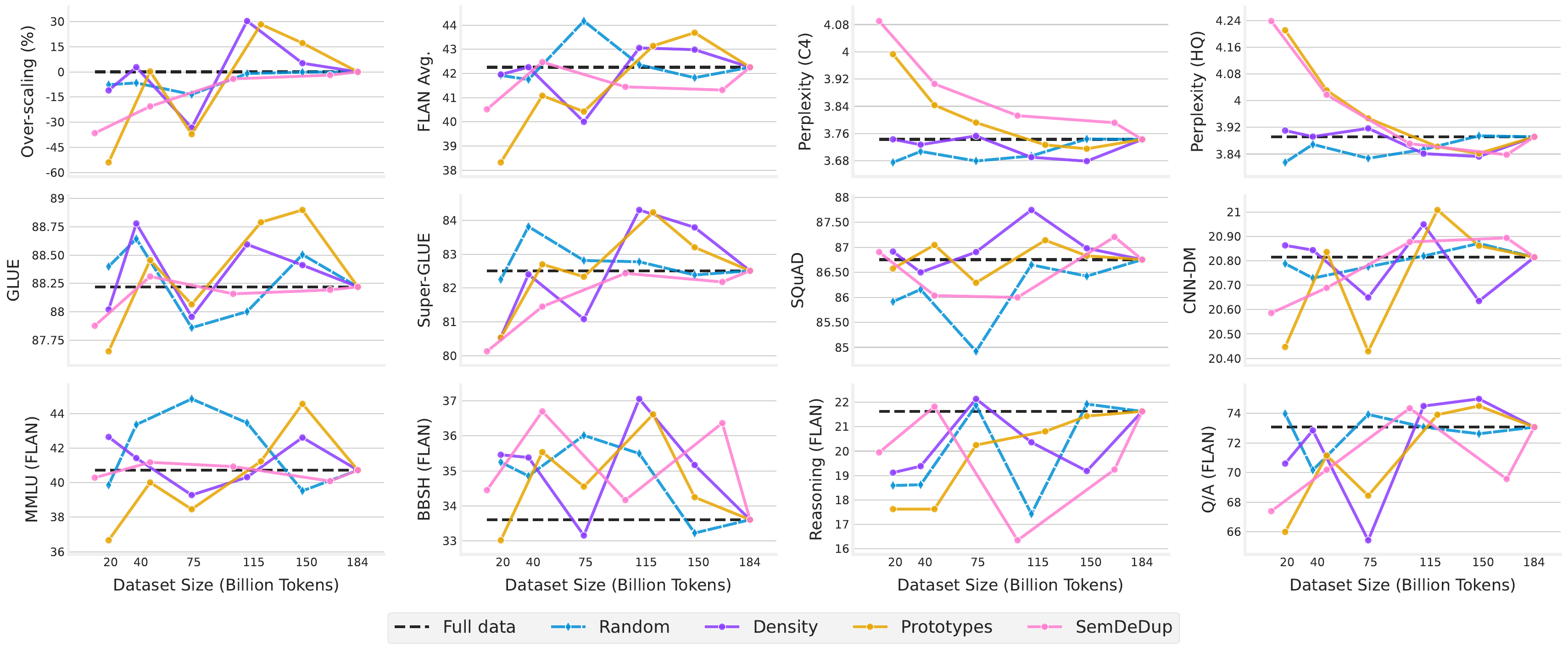} 
    \vspace{-0.2cm}
    \caption{Tradeoff between data quantity and model quality while pre-training T5-Large. Each point in this plot comes from the converged pre-training run over a sampled dataset. See \cref{appendix:evals} for a description about the metrics used in this plot.}
    \vspace{-0.2cm}
    \label{fig:unique_tokens_vs_model_quality_large_0}
\end{figure}

\begin{figure}[H] 
    \centering
    \includegraphics[width=\linewidth,trim={0 0 0 0},clip]{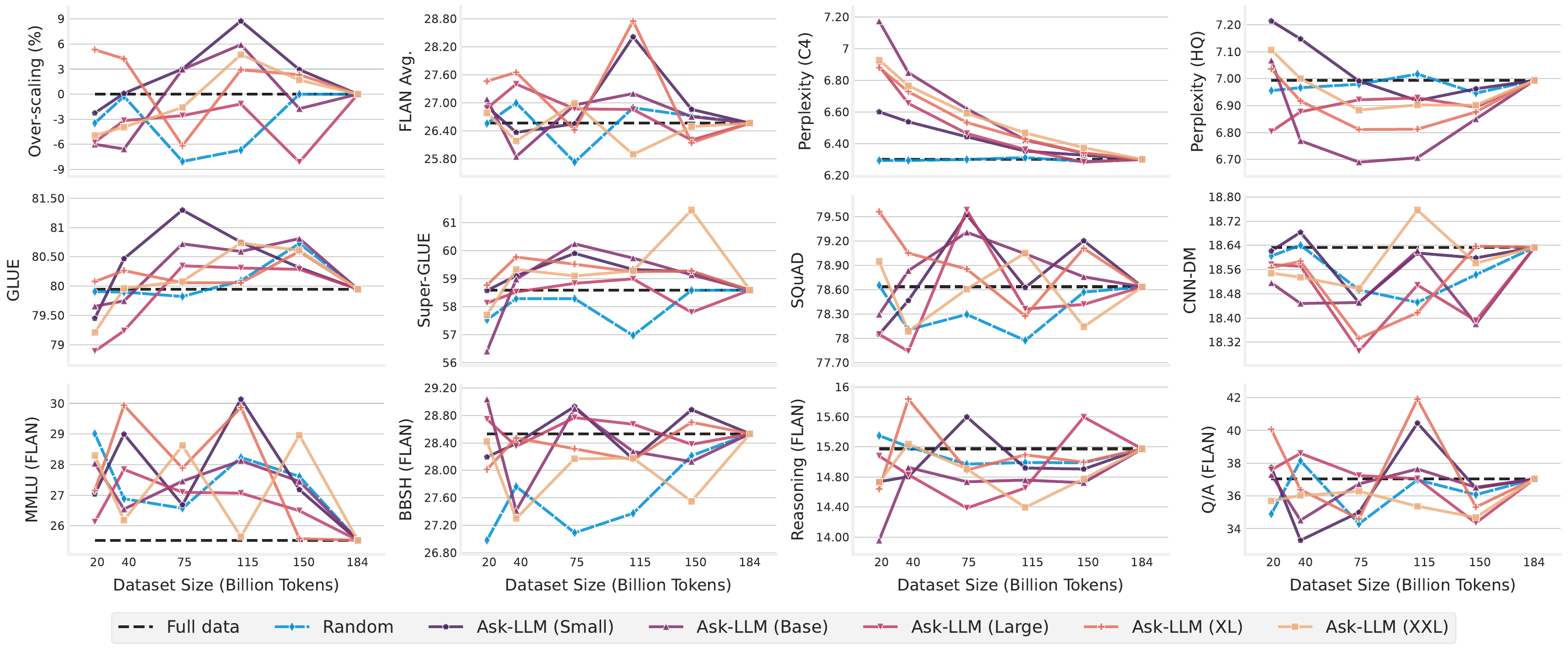} 
    \vspace{-0.2cm}
    \caption{Tradeoff between data quantity and model quality while pre-training T5-Small. Each point in this plot comes from the converged pre-training run over a sampled dataset. See \cref{appendix:evals} for a description about the metrics used in this plot.}
    \vspace{-0.2cm}
    \label{fig:unique_tokens_vs_model_quality_small_1}
\end{figure}

\begin{figure}[H] 
    \centering
    \includegraphics[width=\linewidth,trim={0 0 0 0},clip]{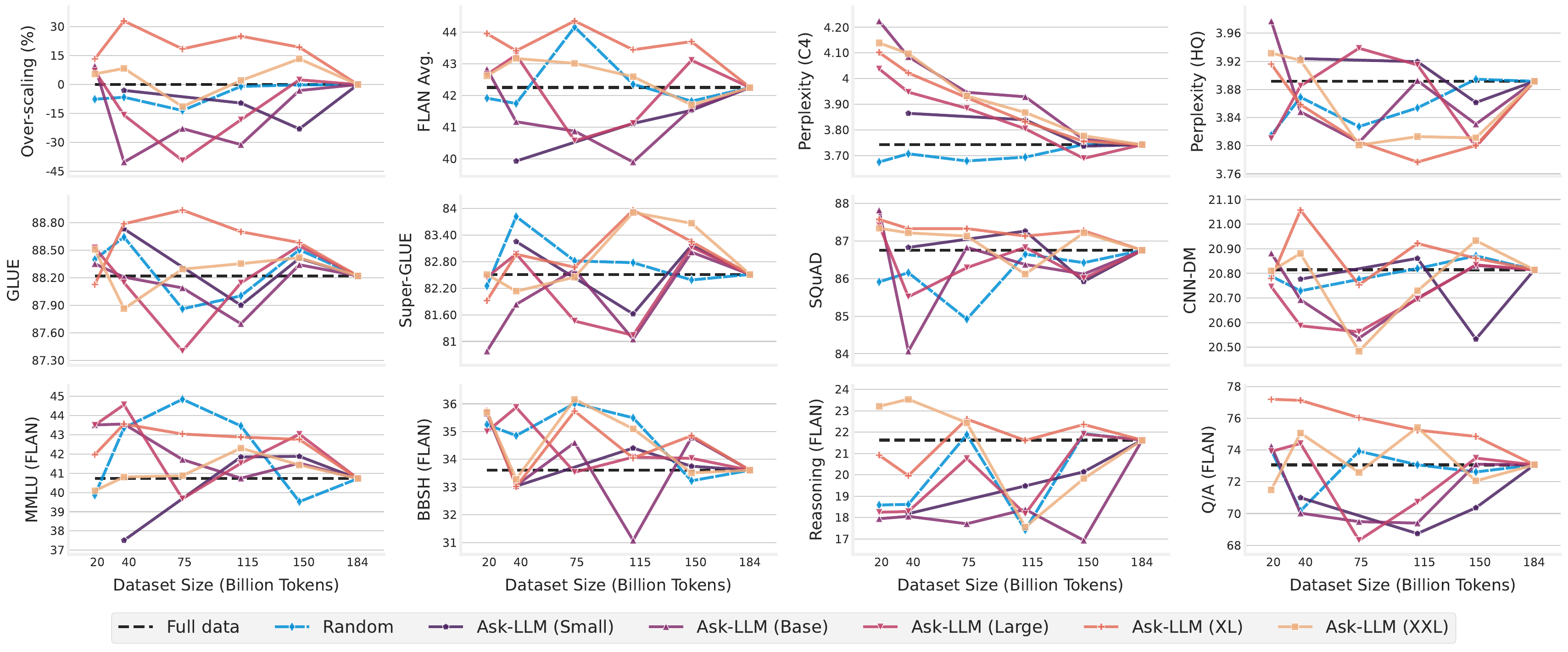} 
    \vspace{-0.2cm}
    \caption{Tradeoff between data quantity and model quality while pre-training T5-Large. Each point in this plot comes from the converged pre-training run over a sampled dataset. See \cref{appendix:evals} for a description about the metrics used in this plot.}
    \vspace{-0.2cm}
    \label{fig:unique_tokens_vs_model_quality_large_1}
\end{figure}

\begin{figure}[H] 
    \centering
    \includegraphics[width=\linewidth,trim={0 0 0 0},clip]{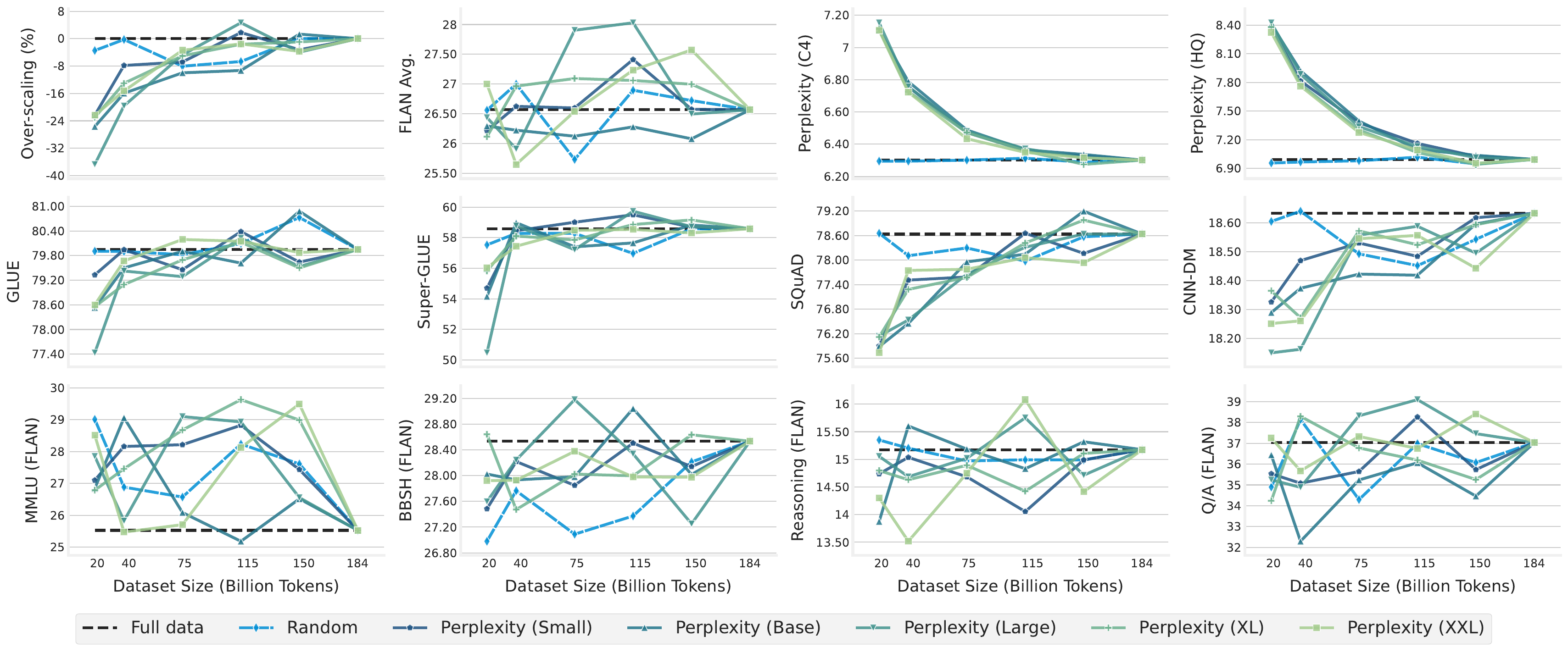} 
    \vspace{-0.2cm}
    \caption{Tradeoff between data quantity and model quality while pre-training T5-Small. Each point in this plot comes from the converged pre-training run over a sampled dataset. See \cref{appendix:evals} for a description about the metrics used in this plot.}
    \vspace{-0.2cm}
    \label{fig:unique_tokens_vs_model_quality_small_2}
\end{figure}

\begin{figure}[H] 
    \centering
    \includegraphics[width=\linewidth,trim={0 0 0 0},clip]{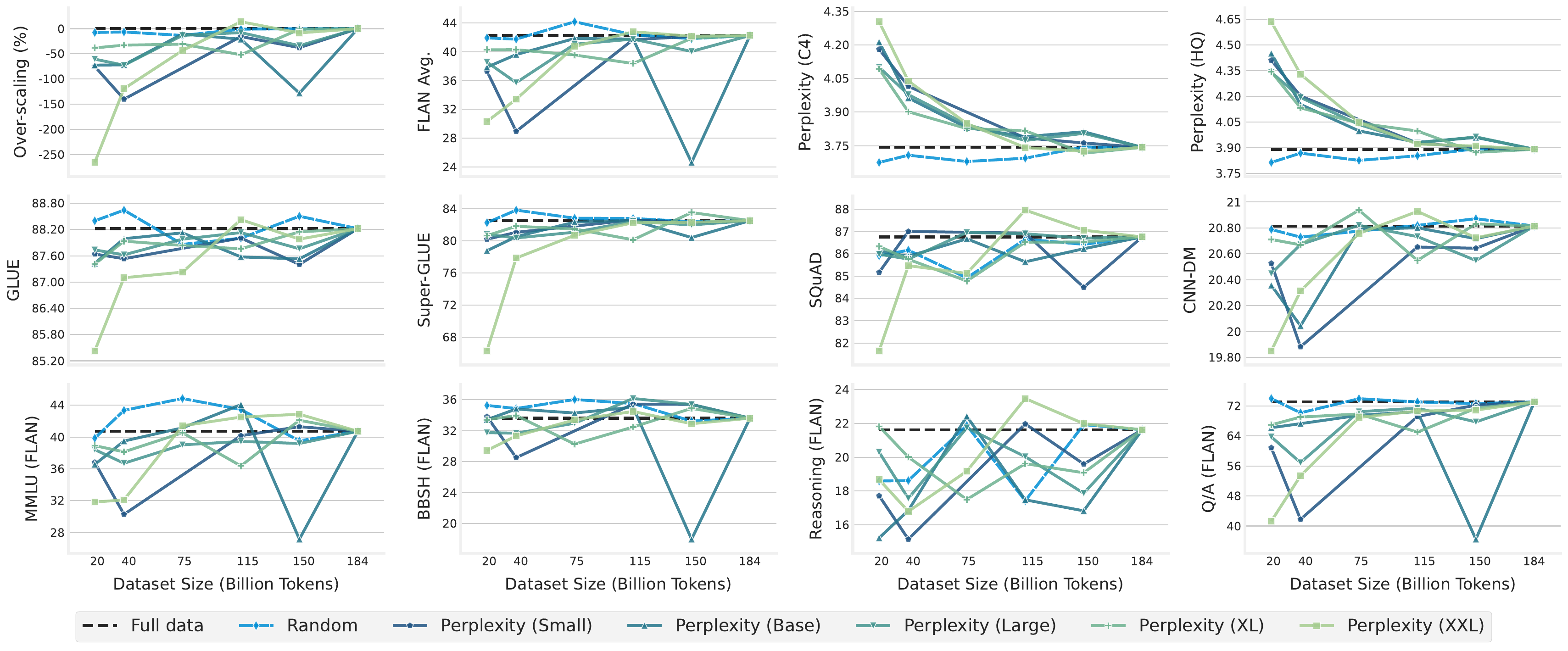} 
    \vspace{-0.2cm}
    \caption{Tradeoff between data quantity and model quality while pre-training T5-Large. Each point in this plot comes from the converged pre-training run over a sampled dataset. See \cref{appendix:evals} for a description about the metrics used in this plot.}
    \vspace{-0.2cm}
    \label{fig:unique_tokens_vs_model_quality_large_2}
\end{figure}

\begin{figure}[H] 
    \centering
    \includegraphics[width=\linewidth,trim={0 0 0 0},clip]{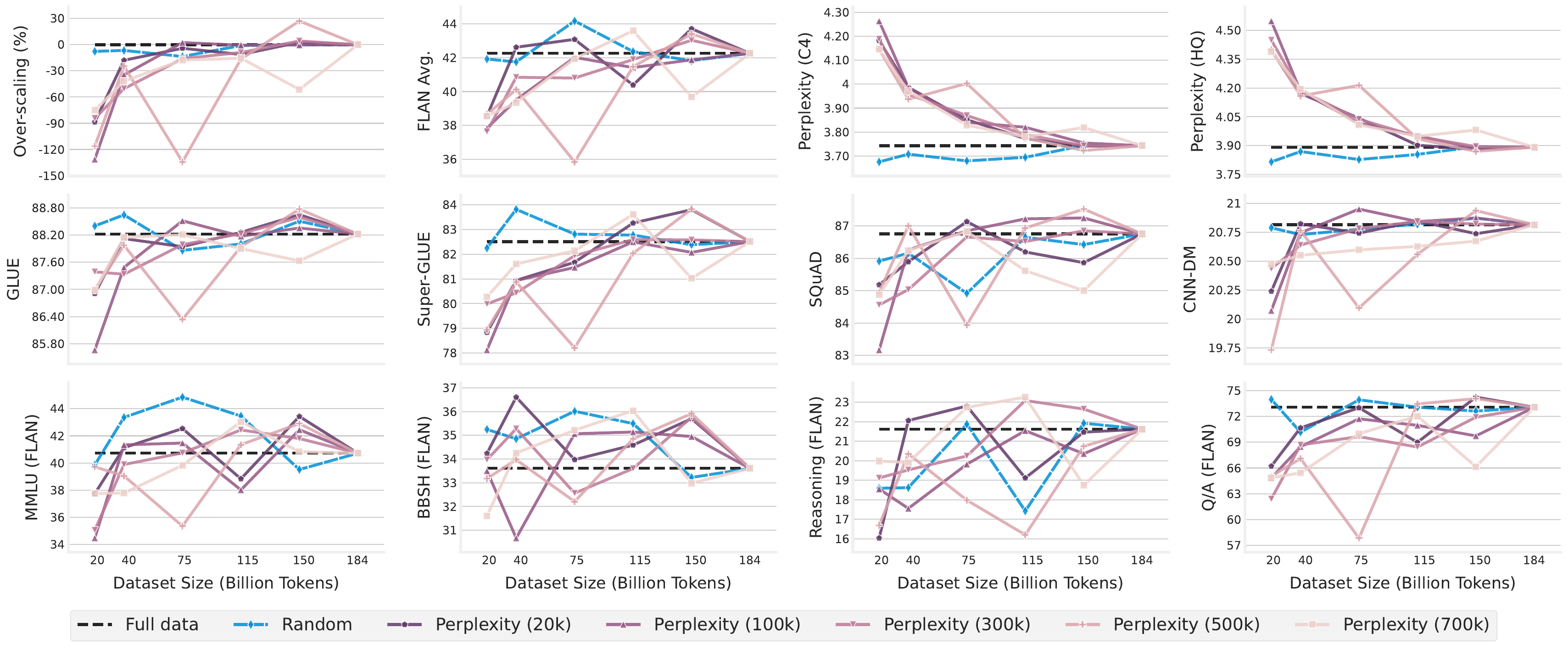} 
    \vspace{-0.2cm}
    \caption{Tradeoff between data quantity and model quality while pre-training T5-Large. Each point in this plot comes from the converged pre-training run over a sampled dataset. See \cref{appendix:evals} for a description about the metrics used in this plot.}
    \vspace{-0.2cm}
    \label{fig:unique_tokens_vs_model_quality_large_3}
\end{figure}

\subsection{(\cref{fig:training_tokens_vs_quality_per_max_epoch_small_0,fig:training_tokens_vs_quality_per_max_epoch_small_1,fig:training_tokens_vs_quality_per_max_epoch_small_2,fig:training_tokens_vs_quality_per_max_epoch_large_0,fig:training_tokens_vs_quality_per_max_epoch_large_1,fig:training_tokens_vs_quality_per_max_epoch_large_2,fig:training_tokens_vs_quality_per_max_epoch_large_3}) Quality of Fresh \vs Repeated Tokens for Different Samplers} 

We investigate the data-efficiency 
for different data curation techniques listed in \cref{appendix:samplers} over various downstream evaluations listed in \cref{appendix:evals}, when stratifying by the maximum number of repetitions allowed over the sampled dataset.
We plot our results in the following figures:
\begin{itemize}[leftmargin=*]
    \item (\cref{fig:training_tokens_vs_quality_per_max_epoch_small_0}) \textbf{T5-Small, coverage}: Average data-efficiency of pre-training T5-Small on data sampled by \{Random sampling, \density sampling, Self-supervised Prototypes sampling, SemDeDup\}, stratified by the maxmimum number of allowed repetitions over the sampled dataset.
    
    \item (\cref{fig:training_tokens_vs_quality_per_max_epoch_large_0}) \textbf{T5-Large, coverage}: Average data-efficiency of pre-training T5-Large on data sampled by \{Random sampling, \density sampling, Self-supervised Prototypes sampling, SemDeDup\}, stratified by the maxmimum number of allowed repetitions over the sampled dataset.
    
    \item (\cref{fig:training_tokens_vs_quality_per_max_epoch_small_1}) \textbf{T5-Small, \askllm}: Average data-efficiency of pre-training T5-Small on data sampled by \askllm using the \{Flan-T5-Small, Flan-T5-Base, Flan-T5-Large, Flan-T5-XL, Flan-T5-XXL\} scoring models, stratified by the maxmimum number of allowed repetitions over the sampled dataset.
    
    \item (\cref{fig:training_tokens_vs_quality_per_max_epoch_large_1}) \textbf{T5-Large, \askllm}: Average data-efficiency of pre-training T5-Large on data sampled by \askllm using the \{Flan-T5-Small, Flan-T5-Base, Flan-T5-Large, Flan-T5-XL, Flan-T5-XXL\} scoring models, stratified by the maxmimum number of allowed repetitions over the sampled dataset.
    
    \item (\cref{fig:training_tokens_vs_quality_per_max_epoch_small_2}) \textbf{T5-Small, Perplexity filtering}: Average data-efficiency of pre-training T5-Small on data sampled by Perplexity filtering using the \{T5-Small, T5-Base, T5-Large, T5-XL, T5-XXL\} scoring models, stratified by the maxmimum number of allowed repetitions over the sampled dataset.
    
    \item (\cref{fig:training_tokens_vs_quality_per_max_epoch_large_2}) \textbf{T5-Large, Perplexity filtering}: Average data-efficiency of pre-training T5-Large on data sampled by Perplexity filtering using the \{T5-Small, T5-Base, T5-Large, T5-XL, T5-XXL\} scoring models, stratified by the maxmimum number of allowed repetitions over the sampled dataset.
    
    \item (\cref{fig:training_tokens_vs_quality_per_max_epoch_large_3}) \textbf{T5-Large, Perplexity filtering}: Average data-efficiency of pre-training T5-Large on data sampled by Perplexity filtering using the \{$20$k, $100$k, $300$k, $500$k, $700$k\} intermediate checkpoints of T5-Large as data quality scoring models, stratified by the maxmimum number of allowed repetitions over the sampled dataset.
\end{itemize}

\begin{figure}[H] 
    \centering
    \includegraphics[width=\linewidth,trim={0 0 0 0},clip]{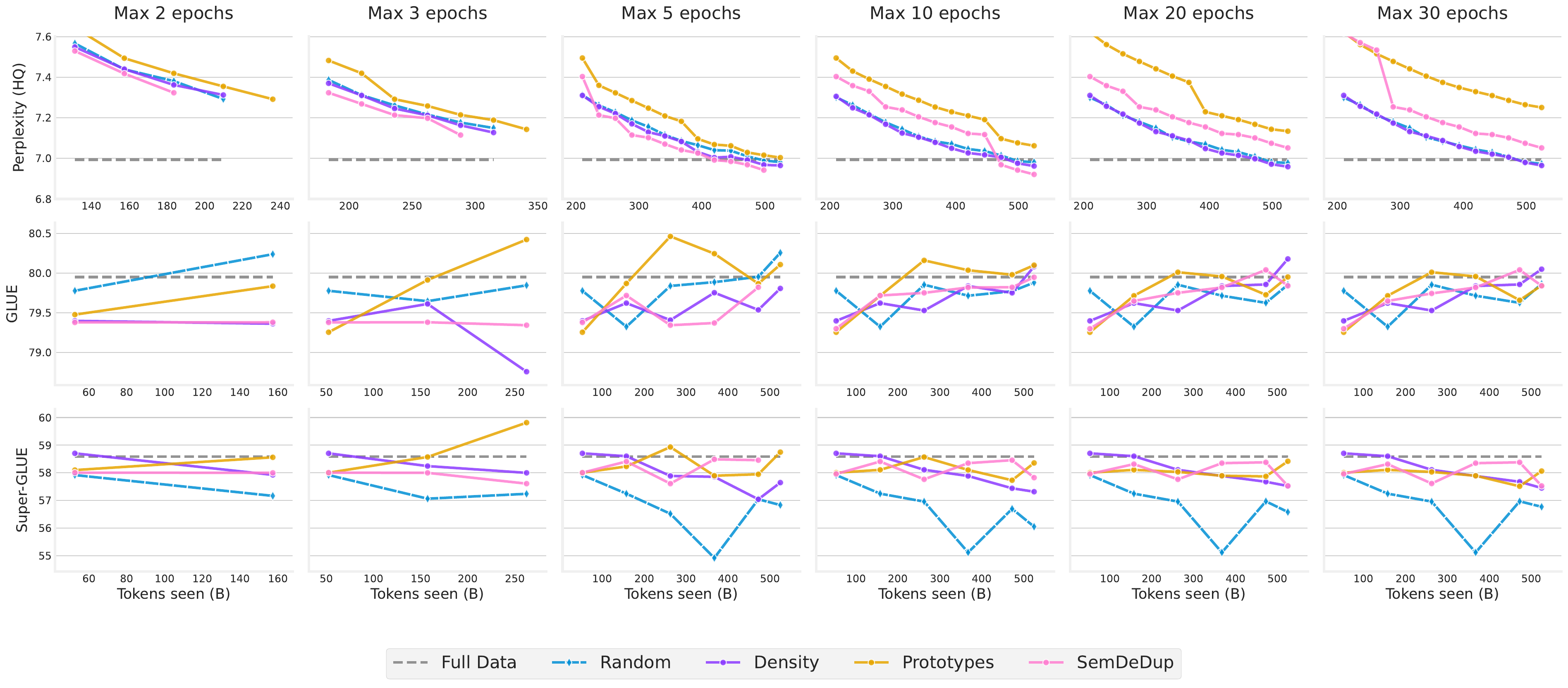} 
    \vspace{-0.2cm}
    \caption{
    Average data-efficiency of pre-training T5-Small on sampled data, stratified by maximum number of allowed repetitions on the sampled dataset. Each point in this plot represents the performance of an intermediate checkpoint \emph{averaged} over all sampling ratios, as long as the maximum allowed repetitions have not been reached. See \cref{appendix:evals} for a description about the metrics used in this plot.
    }
    \vspace{-0.2cm}
    \label{fig:training_tokens_vs_quality_per_max_epoch_small_0}
\end{figure}

\begin{figure}[H] 
    \centering
    \includegraphics[width=\linewidth,trim={0 0 0 0},clip]{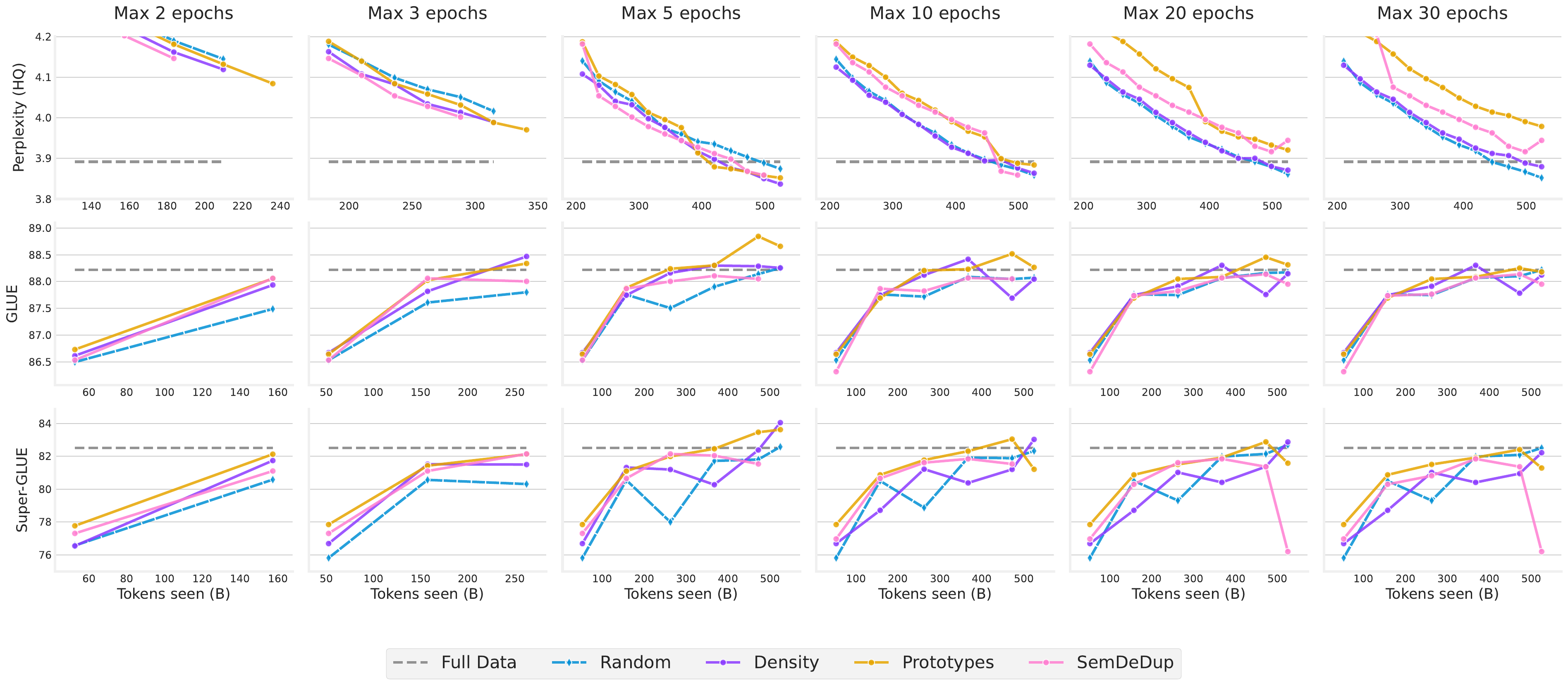} 
    \vspace{-0.2cm}
    \caption{Average data-efficiency of pre-training T5-Large on sampled data, stratified by maximum number of allowed repetitions on the sampled dataset. Each point in this plot represents the performance of an intermediate checkpoint \emph{averaged} over all sampling ratios, as long as the maximum allowed repetitions have not been reached. See \cref{appendix:evals} for a description about the metrics used in this plot.}
    \vspace{-0.2cm}
    \label{fig:training_tokens_vs_quality_per_max_epoch_large_0}
\end{figure}

\begin{figure}[H] 
    \centering
    \includegraphics[width=\linewidth,trim={0 0 0 0},clip]{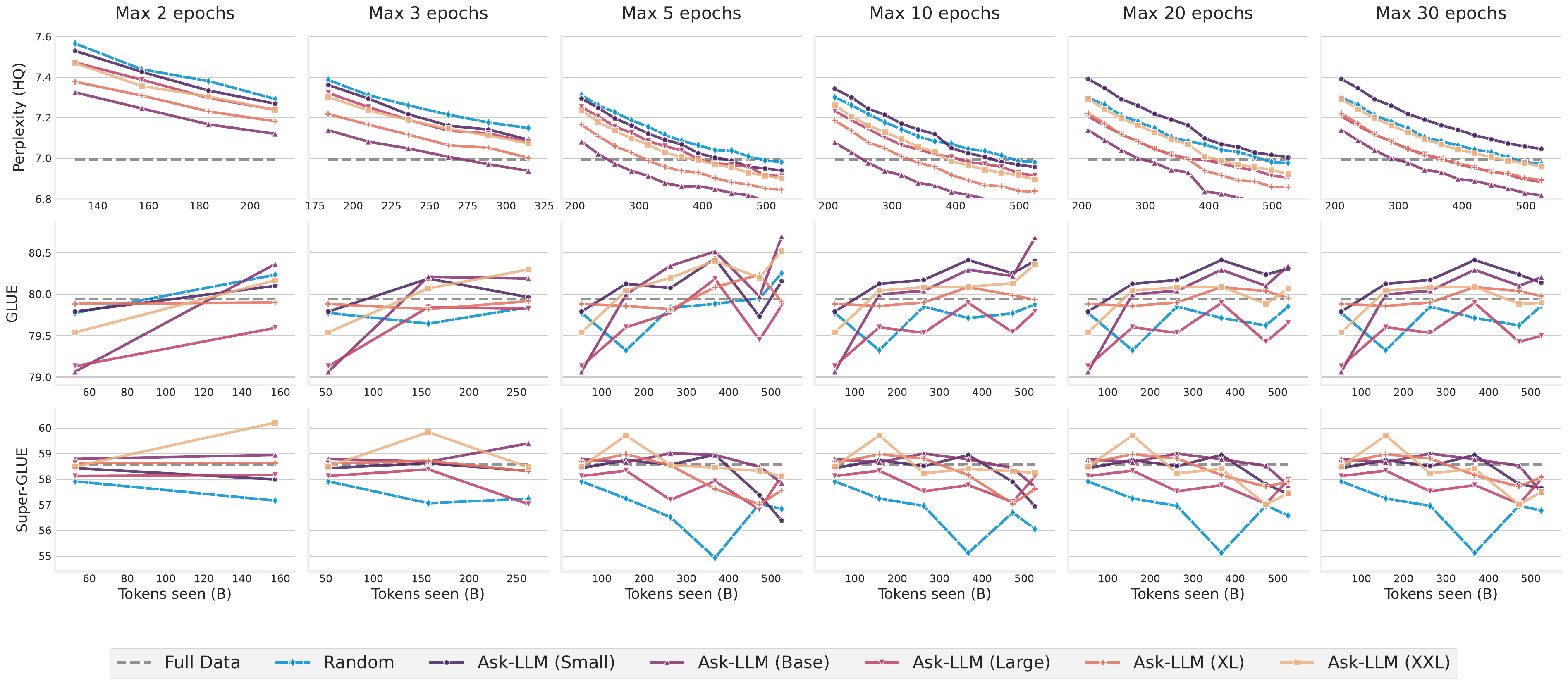} 
    \vspace{-0.2cm}
    \caption{Average data-efficiency of pre-training T5-Small on sampled data, stratified by maximum number of allowed repetitions on the sampled dataset. Each point in this plot represents the performance of an intermediate checkpoint \emph{averaged} over all sampling ratios, as long as the maximum allowed repetitions have not been reached. See \cref{appendix:evals} for a description about the metrics used in this plot.}
    \vspace{-0.2cm}
    \label{fig:training_tokens_vs_quality_per_max_epoch_small_1}
\end{figure}

\begin{figure}[H] 
    \centering
    \includegraphics[width=\linewidth,trim={0 0 0 0},clip]{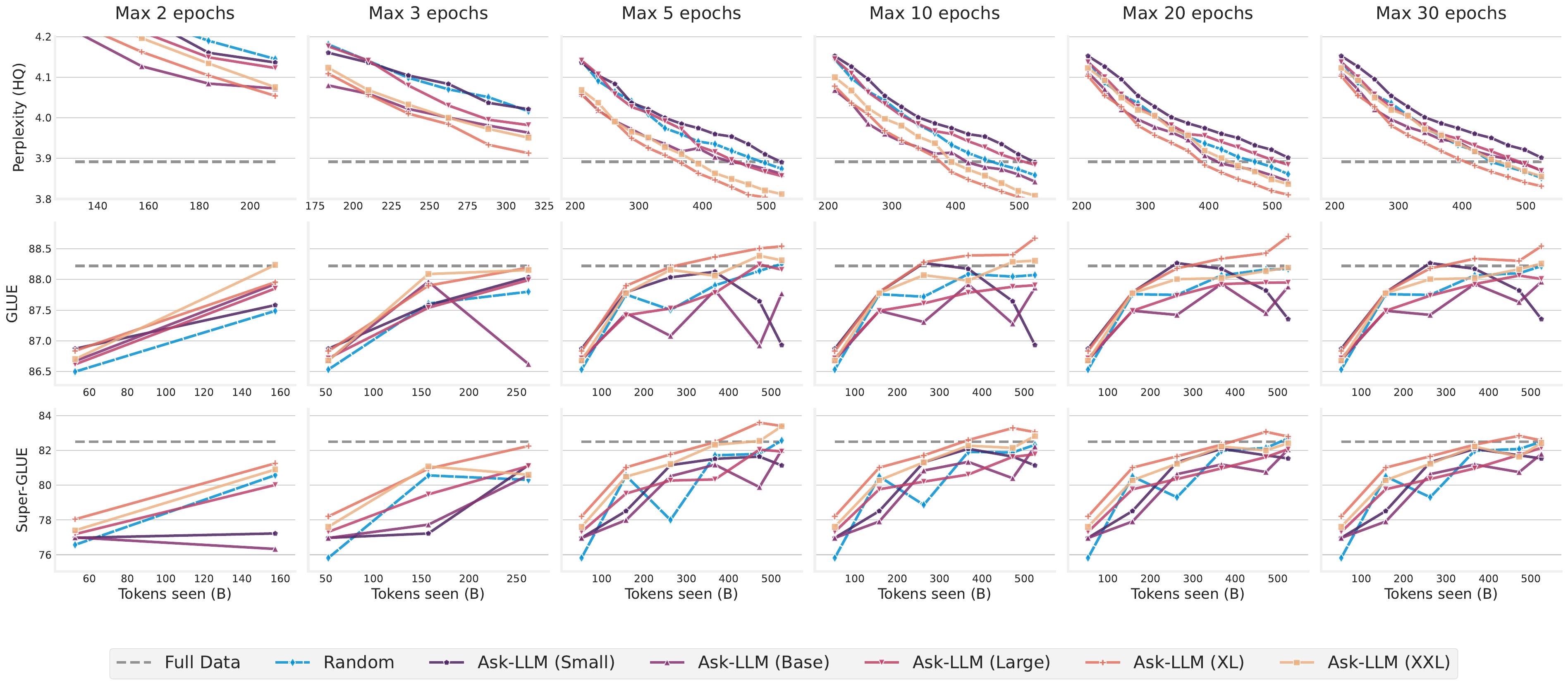} 
    \vspace{-0.2cm}
    \caption{Average data-efficiency of pre-training T5-Large on sampled data, stratified by maximum number of allowed repetitions on the sampled dataset. Each point in this plot represents the performance of an intermediate checkpoint \emph{averaged} over all sampling ratios, as long as the maximum allowed repetitions have not been reached. See \cref{appendix:evals} for a description about the metrics used in this plot.}
    \vspace{-0.2cm}
    \label{fig:training_tokens_vs_quality_per_max_epoch_large_1}
\end{figure}

\begin{figure}[H] 
    \centering
    \includegraphics[width=\linewidth,trim={0 0 0 0},clip]{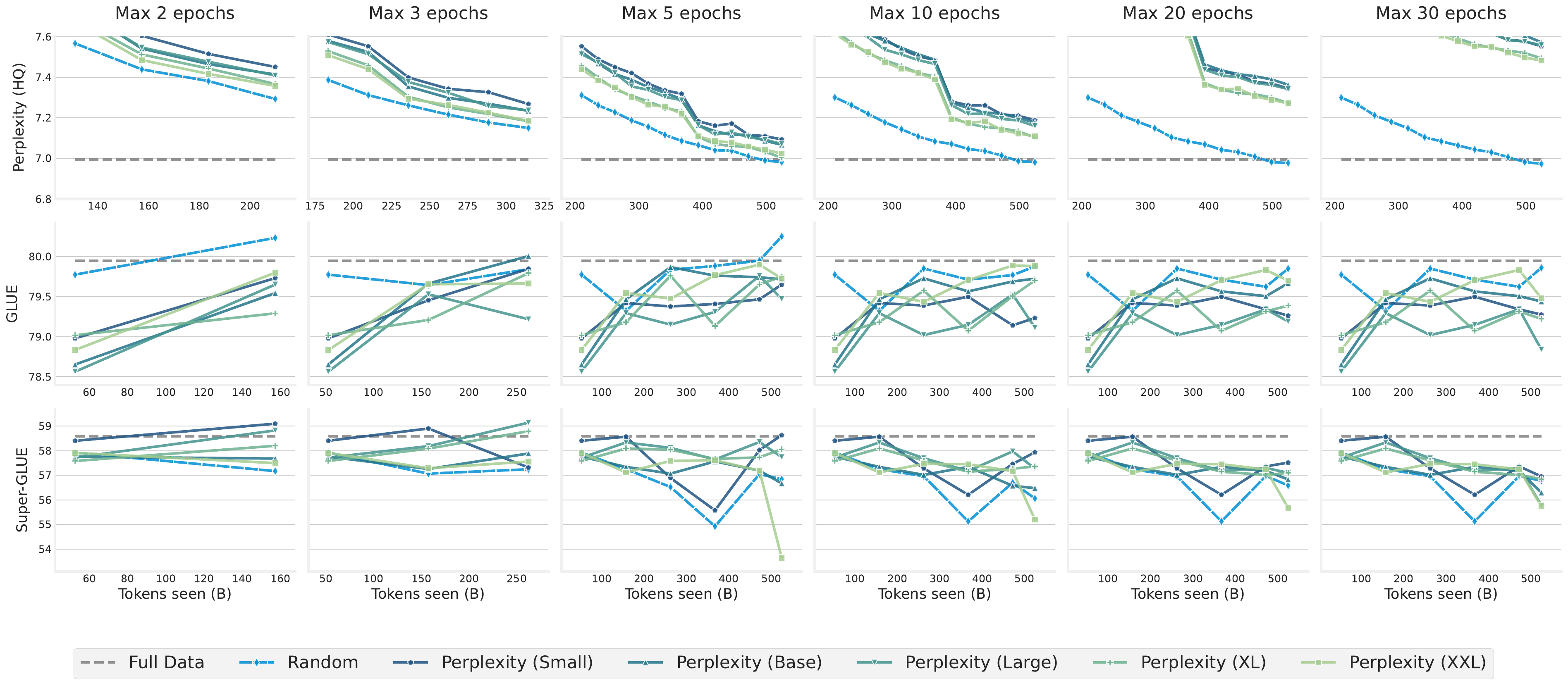} 
    \vspace{-0.2cm}
    \caption{Average data-efficiency of pre-training T5-Small on sampled data, stratified by maximum number of allowed repetitions on the sampled dataset. Each point in this plot represents the performance of an intermediate checkpoint \emph{averaged} over all sampling ratios, as long as the maximum allowed repetitions have not been reached. See \cref{appendix:evals} for a description about the metrics used in this plot.}
    \vspace{-0.2cm}
    \label{fig:training_tokens_vs_quality_per_max_epoch_small_2}
\end{figure}

\begin{figure}[H] 
    \centering
    \includegraphics[width=\linewidth,trim={0 0 0 0},clip]{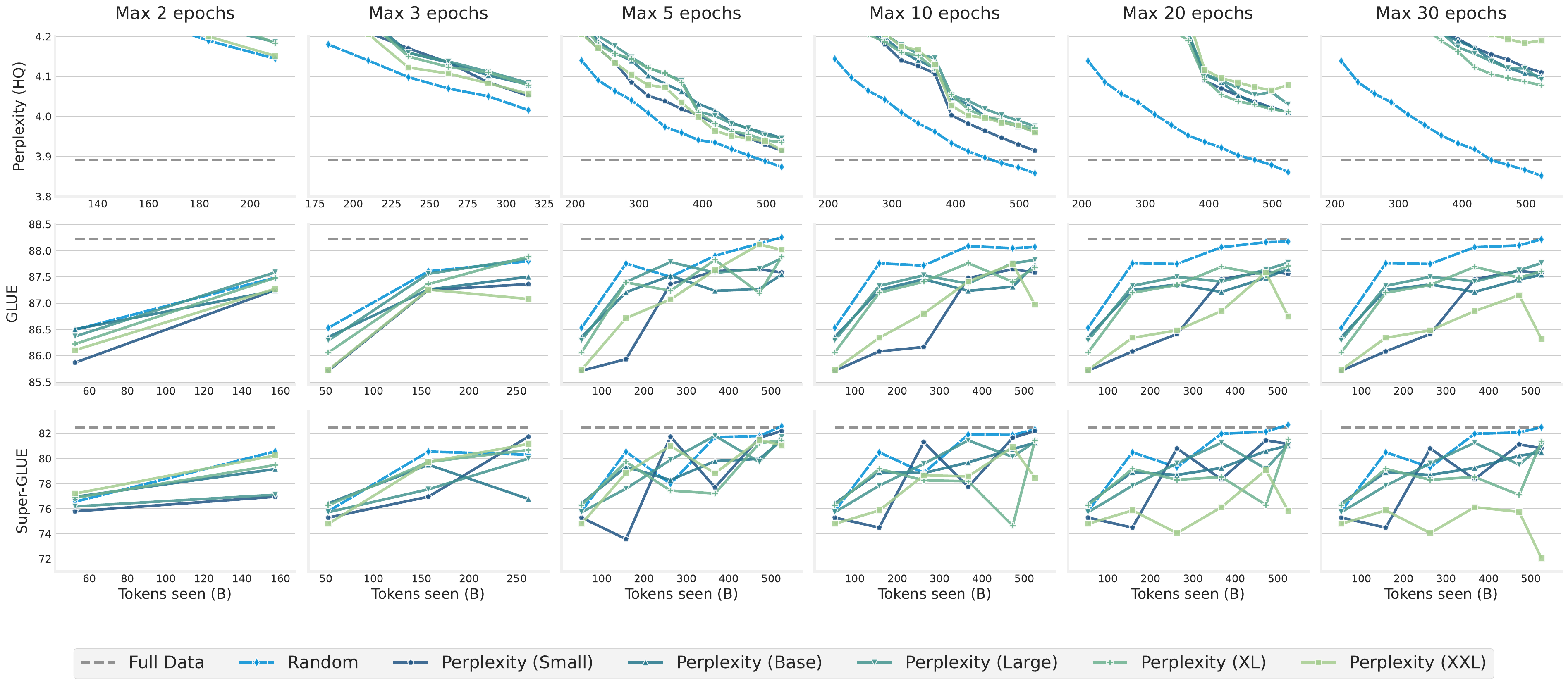} 
    \vspace{-0.2cm}
    \caption{Average data-efficiency of pre-training T5-Large on sampled data, stratified by maximum number of allowed repetitions on the sampled dataset. Each point in this plot represents the performance of an intermediate checkpoint \emph{averaged} over all sampling ratios, as long as the maximum allowed repetitions have not been reached. See \cref{appendix:evals} for a description about the metrics used in this plot.}
    \vspace{-0.2cm}
    \label{fig:training_tokens_vs_quality_per_max_epoch_large_2}
\end{figure}

\begin{figure}[H] 
    \centering
    \includegraphics[width=\linewidth,trim={0 0 0 0},clip]{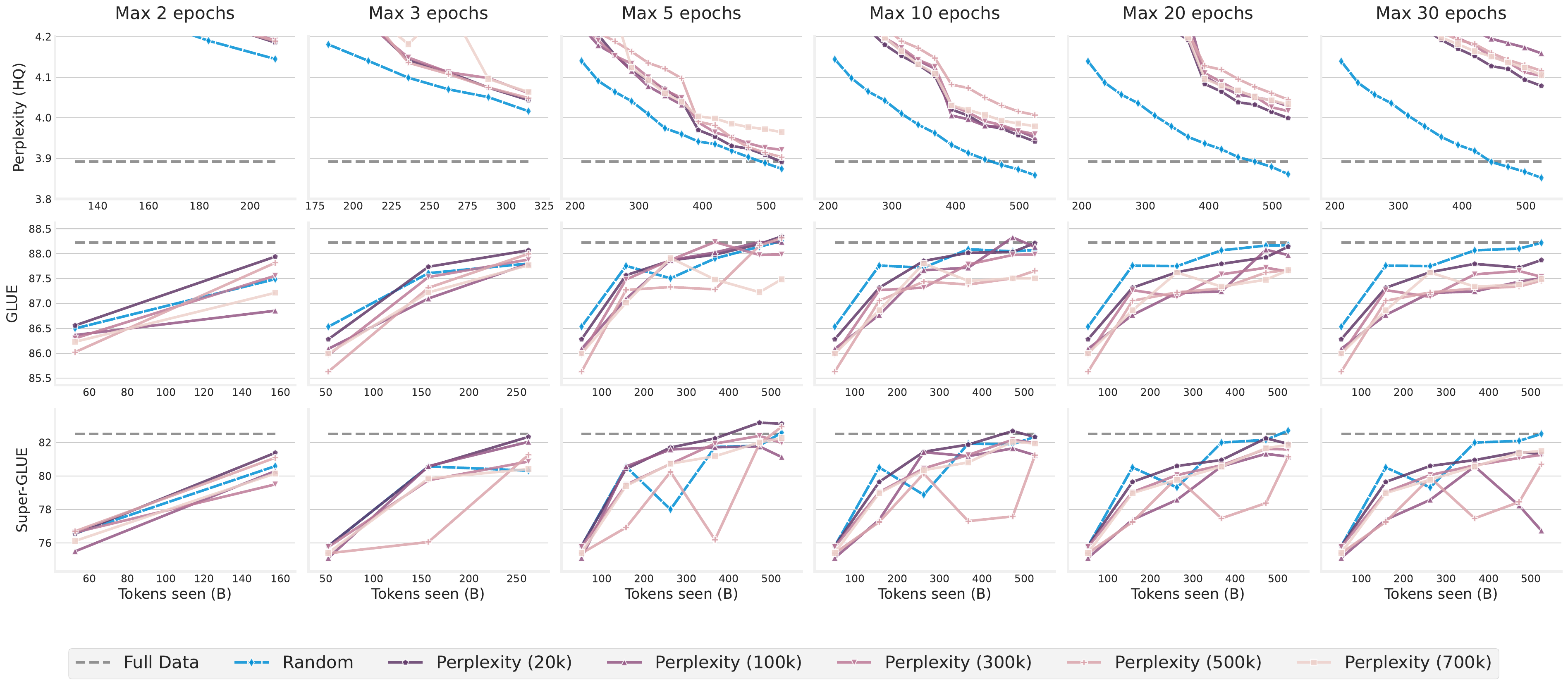} 
    \vspace{-0.2cm}
    \caption{Average data-efficiency of pre-training T5-Large on sampled data, stratified by maximum number of allowed repetitions on the sampled dataset. Each point in this plot represents the performance of an intermediate checkpoint \emph{averaged} over all sampling ratios, as long as the maximum allowed repetitions have not been reached. See \cref{appendix:evals} for a description about the metrics used in this plot.}
    \vspace{-0.2cm}
    \label{fig:training_tokens_vs_quality_per_max_epoch_large_3}
\end{figure}

\subsection{(\cref{fig:training_tokens_vs_model_quality_small_0,fig:training_tokens_vs_model_quality_small_1,fig:training_tokens_vs_model_quality_small_2,fig:training_tokens_vs_model_quality_large_0,fig:training_tokens_vs_model_quality_large_1,fig:training_tokens_vs_model_quality_large_2,fig:training_tokens_vs_model_quality_large_3}) Data-efficiency of Different Samplers} 

We investigate the data-efficiency 
for different data curation techniques listed in \cref{appendix:samplers} over various downstream evaluations listed in \cref{appendix:evals}, when stratifying by the sampling ratio \emph{or} the size of the sampled dataset.
We plot our results in the following figures:

\begin{itemize}[leftmargin=*]
    \item (\cref{fig:training_tokens_vs_model_quality_small_0}) \textbf{T5-Small, coverage}: Data-efficiency of pre-training T5-Small on data sampled by \{Random sampling, \density sampling, Self-supervised Prototypes sampling, SemDeDup\}, stratified by the sampling ratio.
    
    \item (\cref{fig:training_tokens_vs_model_quality_large_0}) \textbf{T5-Large, coverage}: Data-efficiency of pre-training T5-Large on data sampled by \{Random sampling, \density sampling, Self-supervised Prototypes sampling, SemDeDup\}, stratified by the sampling ratio.
    
    \item (\cref{fig:training_tokens_vs_model_quality_small_1}) \textbf{T5-Small, \askllm}: Data-efficiency of pre-training T5-Small on data sampled by \askllm using the \{Flan-T5-Small, Flan-T5-Base, Flan-T5-Large, Flan-T5-XL, Flan-T5-XXL\} scoring models, stratified by the sampling ratio.
    
    \item (\cref{fig:training_tokens_vs_model_quality_large_1}) \textbf{T5-Large, \askllm}: Data-efficiency of pre-training T5-Large on data sampled by \askllm using the \{Flan-T5-Small, Flan-T5-Base, Flan-T5-Large, Flan-T5-XL, Flan-T5-XXL\} scoring models, stratified by the sampling ratio.
    
    \item (\cref{fig:training_tokens_vs_model_quality_small_2}) \textbf{T5-Small, Perplexity filtering}: Data-efficiency of pre-training T5-Small on data sampled by Perplexity filtering using the \{T5-Small, T5-Base, T5-Large, T5-XL, T5-XXL\} scoring models, stratified by the sampling ratio.
    
    \item (\cref{fig:training_tokens_vs_model_quality_large_2}) \textbf{T5-Large, Perplexity filtering}: Data-efficiency of pre-training T5-Large on data sampled by Perplexity filtering using the \{T5-Small, T5-Base, T5-Large, T5-XL, T5-XXL\} scoring models, stratified by the sampling ratio.
    
    \item (\cref{fig:training_tokens_vs_model_quality_large_3}) \textbf{T5-Large, Perplexity filtering}: Data-efficiency of pre-training T5-Large on data sampled by Perplexity filtering using the \{$20$k, $100$k, $300$k, $500$k, $700$k\} intermediate checkpoints of T5-Large as data quality scoring models, stratified by the sampling ratio.
\end{itemize}

\begin{figure}[H] 
    \centering
    \includegraphics[width=\linewidth,trim={0 0 0 0},clip]{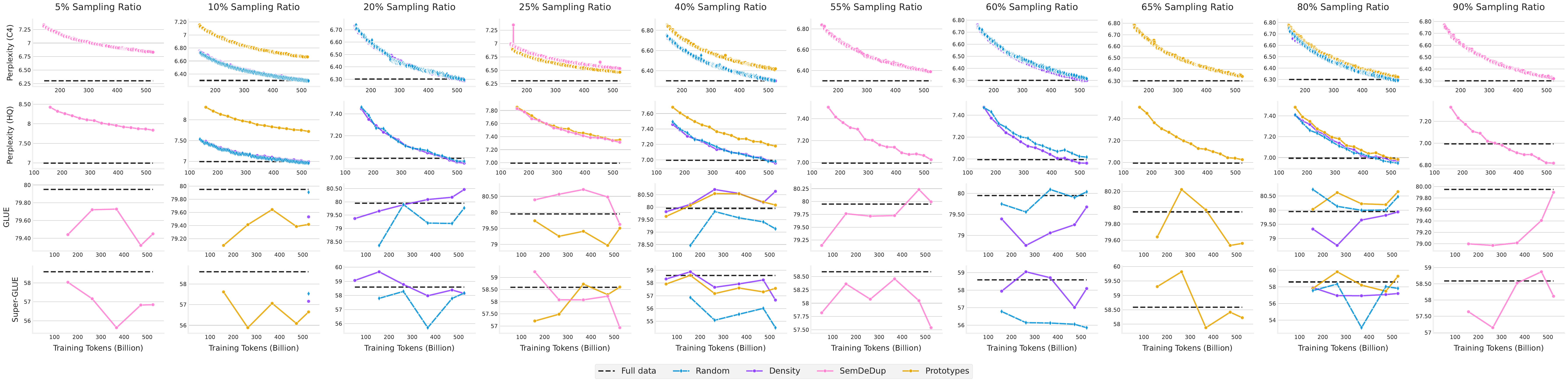} 
    \vspace{-0.2cm}
    \caption{
    Data efficiency comparison of different samplers while training T5-Small for various sampling ratios. Each point in this plot is the performance of an intermediate checkpoint during the course of training on sampled data.
    }
    \vspace{-0.2cm}
    \label{fig:training_tokens_vs_model_quality_small_0}
\end{figure}

\begin{figure}[H] 
    \centering
    \includegraphics[width=\linewidth,trim={0 0 0 0},clip]{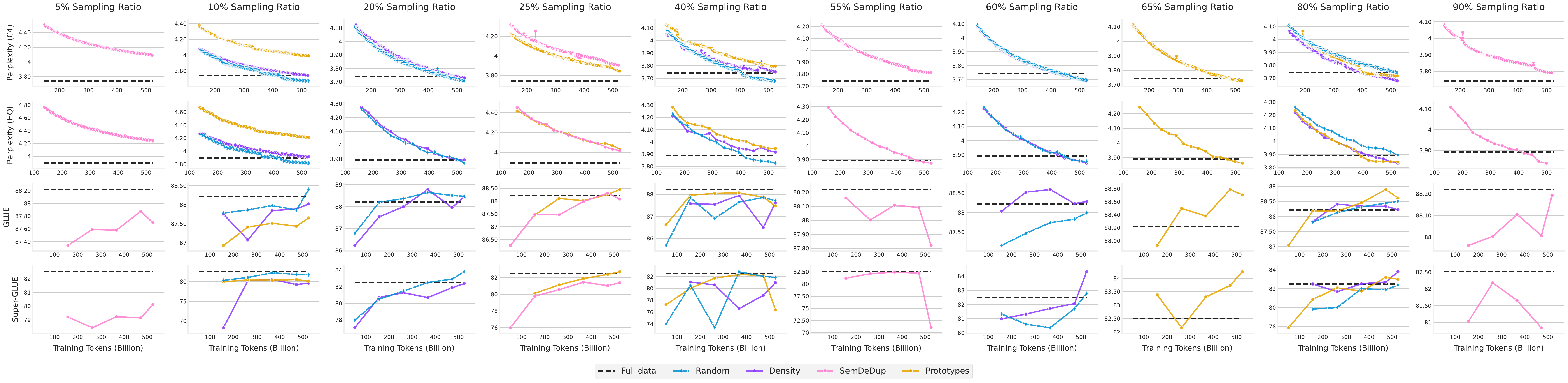} 
    \vspace{-0.2cm}
    \caption{Data efficiency comparison of different samplers while training T5-Large for various sampling ratios. Each point in this plot is the performance of an intermediate checkpoint during the course of training on sampled data.}
    \vspace{-0.2cm}
    \label{fig:training_tokens_vs_model_quality_large_0}
\end{figure}

\newpage

\begin{figure}[H] 
    \centering
    \includegraphics[width=\linewidth,trim={0 0 0 0},clip]{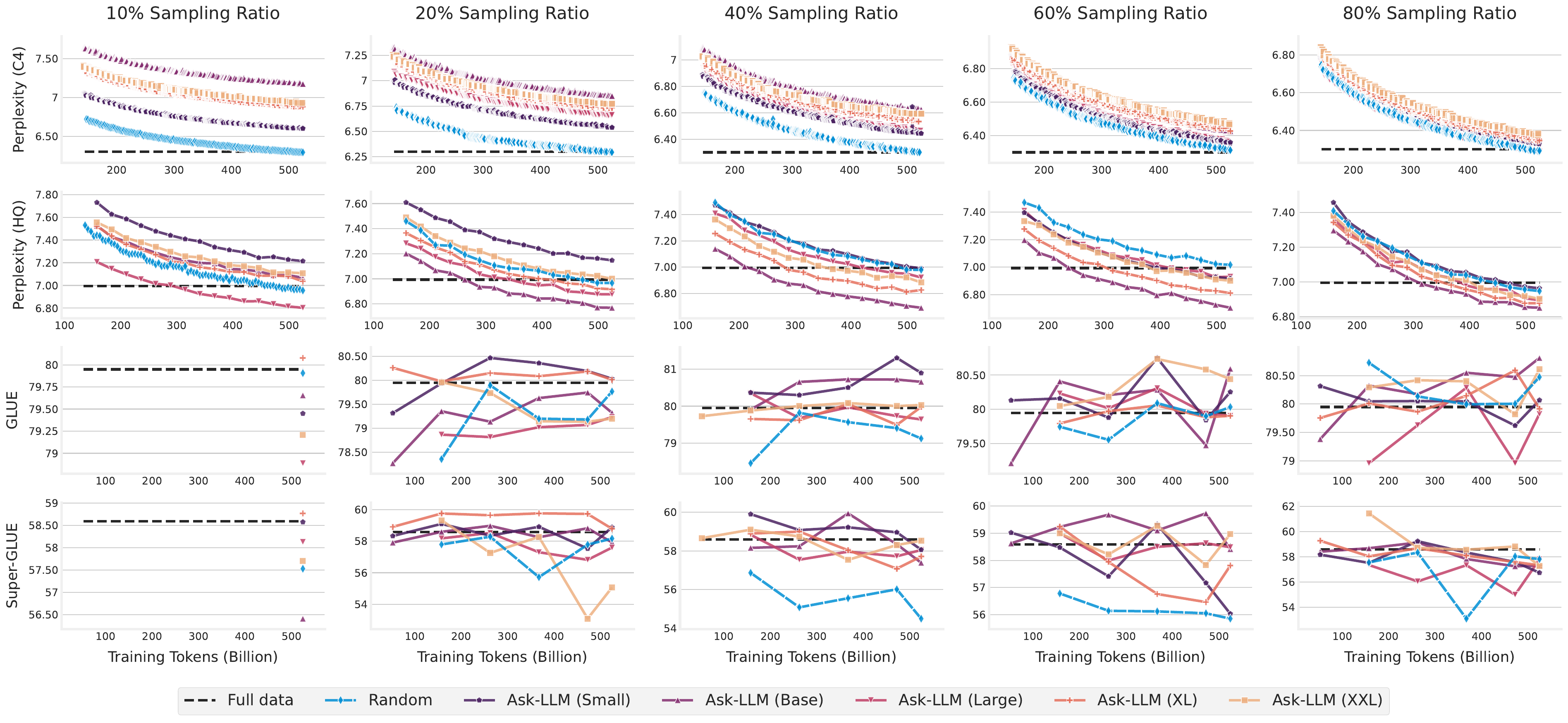} 
    \vspace{-0.2cm}
    \caption{Data efficiency comparison of different samplers while training T5-Small for various sampling ratios. Each point in this plot is the performance of an intermediate checkpoint during the course of training on sampled data.}
    \vspace{-0.2cm}
    \label{fig:training_tokens_vs_model_quality_small_1}
\end{figure}

\begin{figure}[H] 
    \centering
    \includegraphics[width=\linewidth,trim={0 0 0 0},clip]{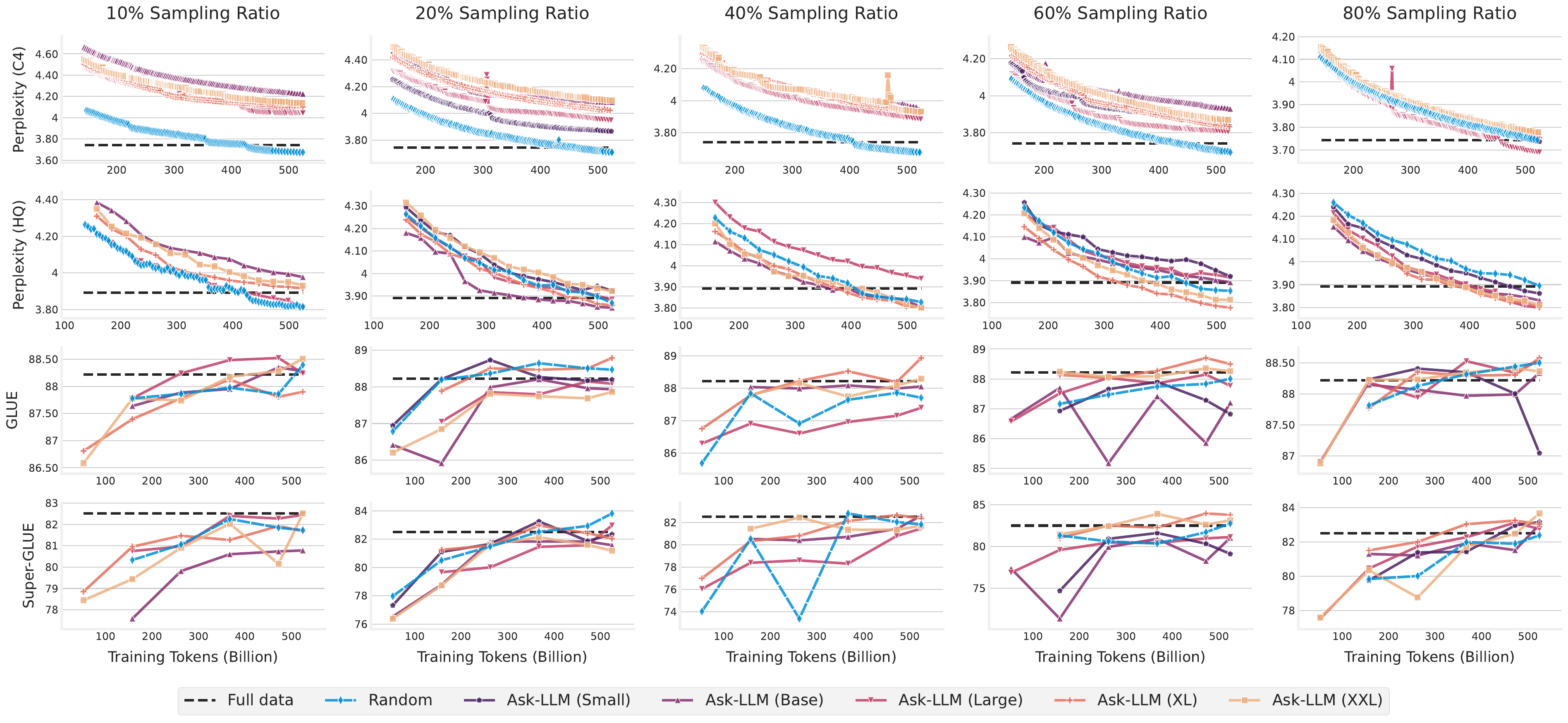} 
    \vspace{-0.2cm}
    \caption{Data efficiency comparison of different samplers while training T5-Large for various sampling ratios. Each point in this plot is the performance of an intermediate checkpoint during the course of training on sampled data.}
    \vspace{-0.2cm}
    \label{fig:training_tokens_vs_model_quality_large_1}
\end{figure}

\begin{figure}[H] 
    \centering
    \includegraphics[width=\linewidth,trim={0 0 0 0},clip]{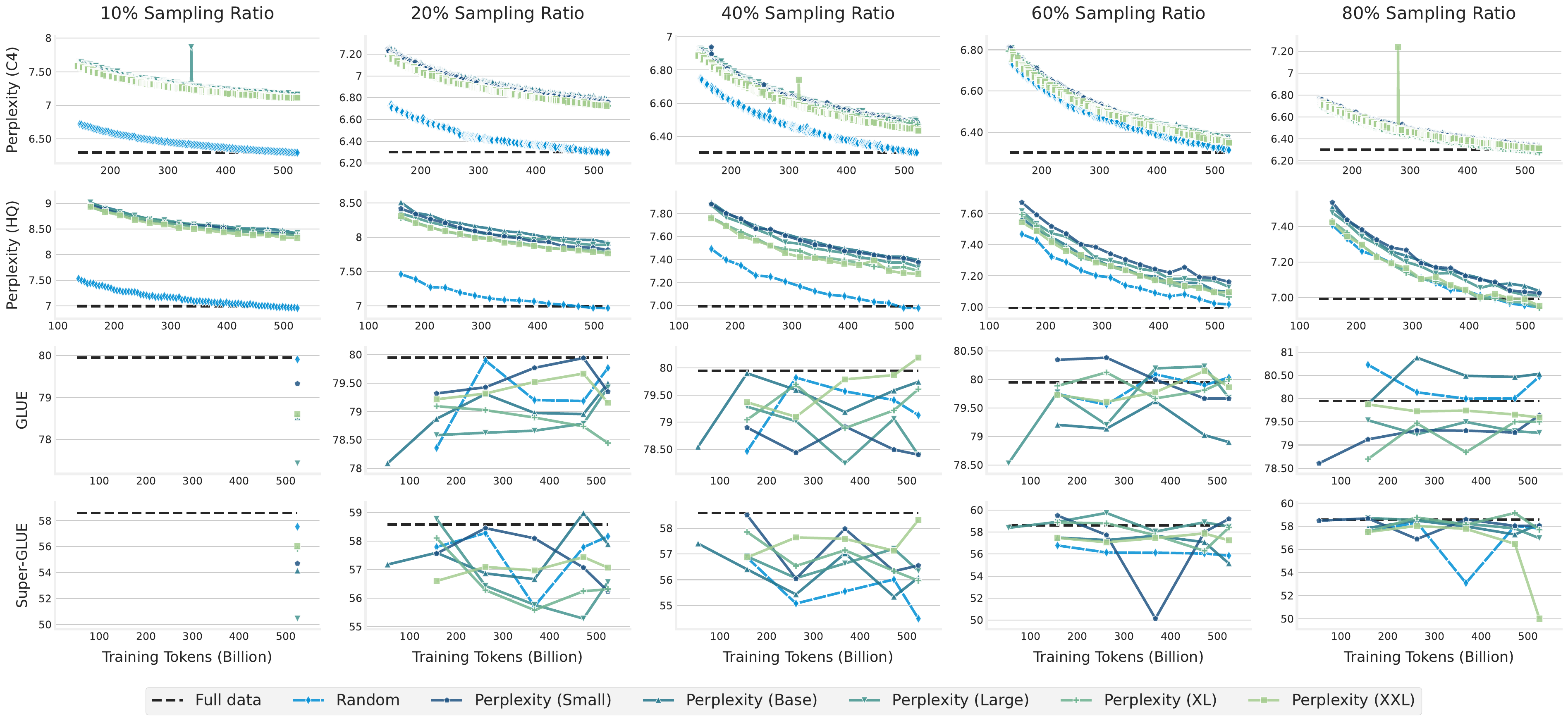} 
    \vspace{-0.2cm}
    \caption{Data efficiency comparison of different samplers while training T5-Small for various sampling ratios. Each point in this plot is the performance of an intermediate checkpoint during the course of training on sampled data.}
    \vspace{-0.2cm}
    \label{fig:training_tokens_vs_model_quality_small_2}
\end{figure}

\begin{figure}[H] 
    \centering
    \includegraphics[width=\linewidth,trim={0 0 0 0},clip]{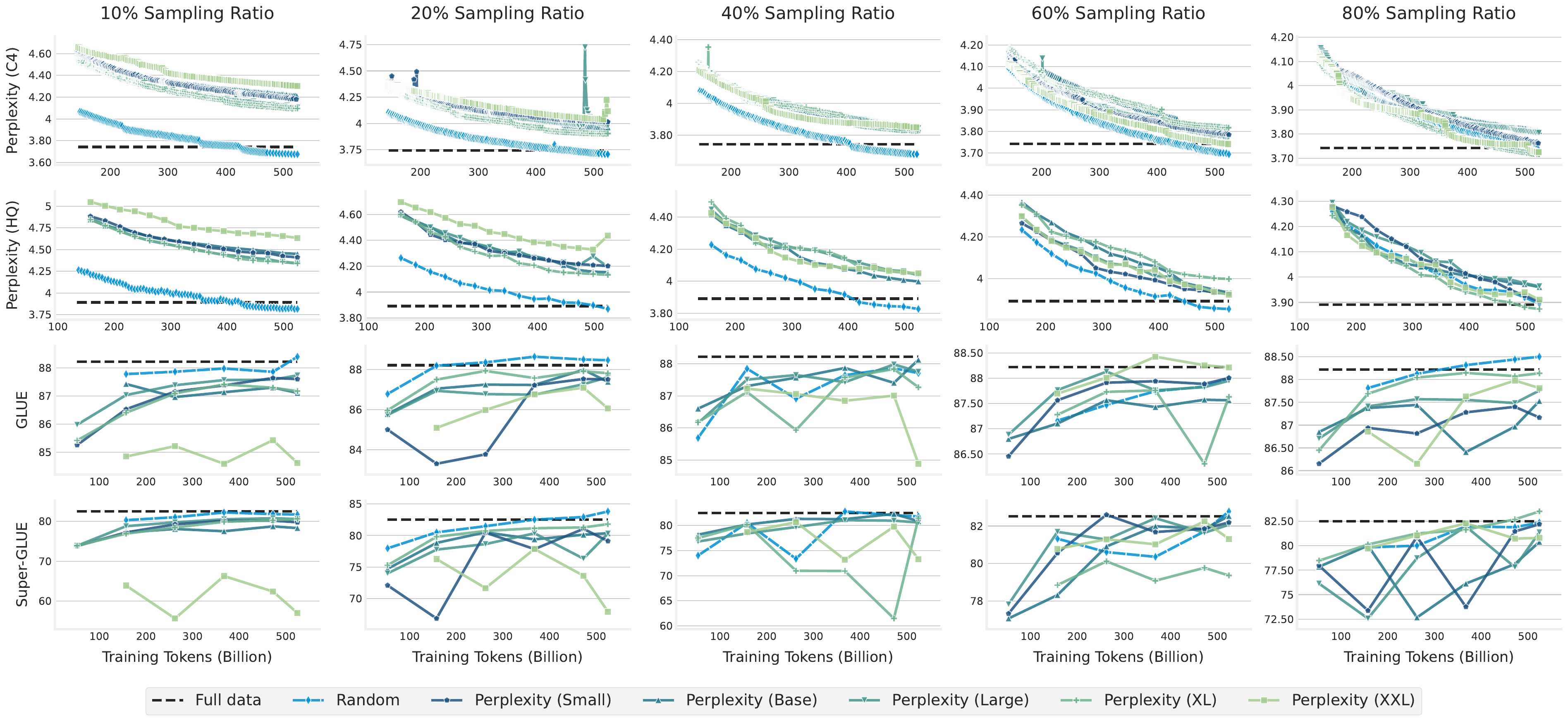} 
    \vspace{-0.2cm}
    \caption{Data efficiency comparison of different samplers while training T5-Large for various sampling ratios. Each point in this plot is the performance of an intermediate checkpoint during the course of training on sampled data.}
    \vspace{-0.2cm}
    \label{fig:training_tokens_vs_model_quality_large_2}
\end{figure}

\begin{figure}[H] 
    \centering
    \includegraphics[width=\linewidth,trim={0 0 0 0},clip]{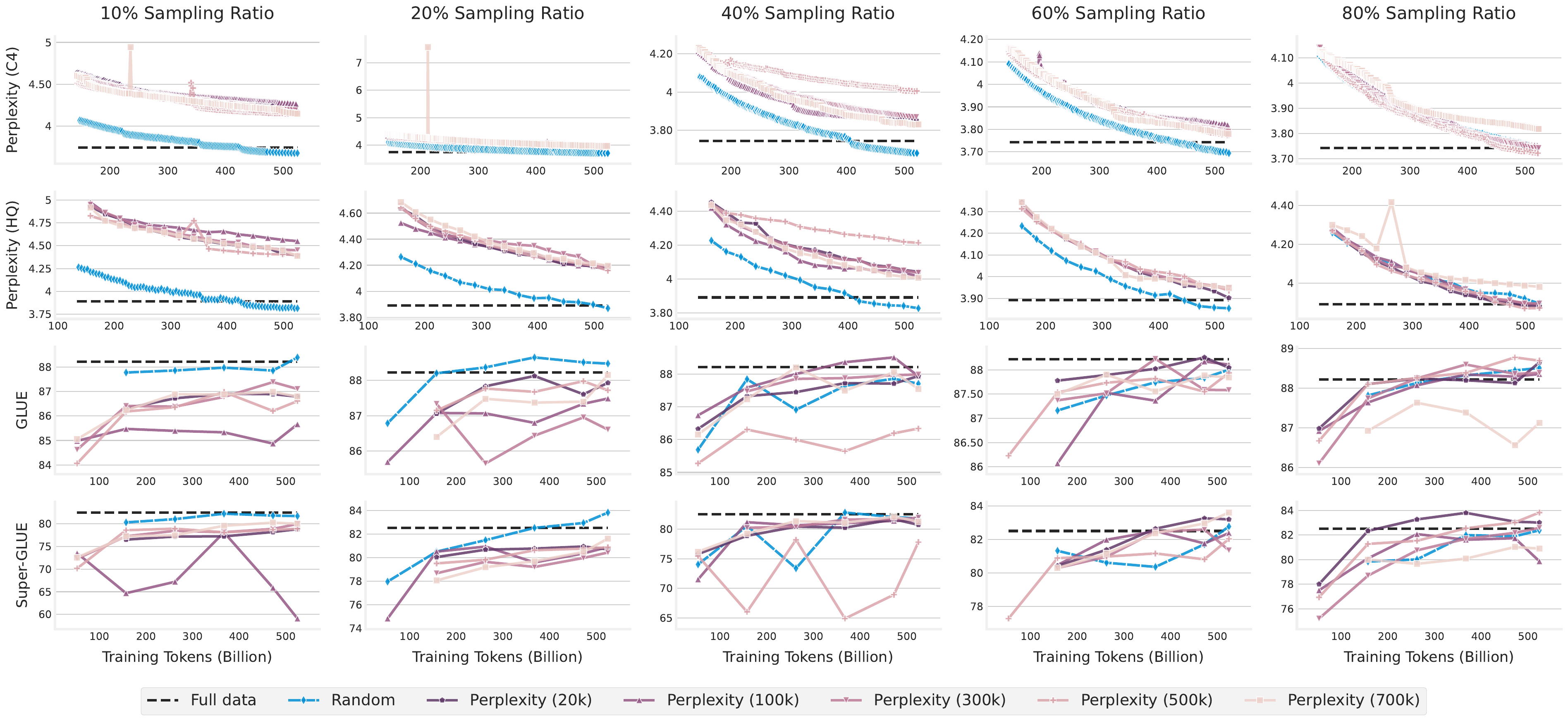} 
    \vspace{-0.6cm}
    \caption{Data efficiency comparison of different samplers while training T5-Large for various sampling ratios. Each point in this plot is the performance of an intermediate checkpoint during the course of training on sampled data.}
    \vspace{-0.3cm}
    \label{fig:training_tokens_vs_model_quality_large_3}
\end{figure}

\section{Qualitative Results} \label{appendix:qualitative_results}

In this section we look at some qualitative training samples, sorted according to various criteria of data-quality scores. Along with the textual content of each training sample, we also list the estimated data-quality percentile for \askllm and perplexity filtering samplers, \ie, the percentile of the given data-point's quality score amongst the entire training set. A high percentile represents that the sampler estimates this training sample to have higher quality compared to other training samples in the dataset. We manually don't include any NSFW examples to the best of our knowledge.

\subsection{High-quality Samples Identified by \askllm}
\label{appendix:ask_llm_high_quality}

We look at the training samples that \emph{all} \askllm scoring models, on average, think are good (\ie, have a high percentile). To the best of our understanding, the overarching conclusions we make by observing these qualitative samples are:
\begin{itemize}[leftmargin=*]
    \item \askllm doesn't seem to have any length bias for good examples.
    
    \item \askllm can accurately tag high-quality training samples that contain a lot of proper nouns and named entities. Perplexity filtering gets these kind of samples wrong.
    
    \item Even looking at this slice of only the highest-quality data tagged by \askllm, perplexity filtering scores don't seem to correlate well with \askllm scores as suggested by \cref{fig:sampler_ranking_correlations}.
\end{itemize}

\input{examples/good_askllm}

\subsection{Low-quality Samples Identified by \askllm}
\label{sec:low_quality_by_LLM}

We look at the training samples that \emph{all} \askllm scoring models, on average, think are bad (\ie, have a low percentile). To the best of our understanding, the overarching conclusions we make by observing these qualitative samples are:
\begin{itemize}[leftmargin=*]
    \item \askllm doesn't seem to have any length bias for bad examples.
    
    \item \askllm filters hateful or toxic examples that might hurt LLM training.
    
    \item \askllm rejects non-contextual samples, \eg, having only questions with no answers, repeated non-sensical content, \etc Notably, perplexity filtering performs bad in these cases, as these low quality examples tend to have a low perplexity score.
\end{itemize}

\input{examples/bad_askllm}

\subsection{Increasing-quality Samples Identified by \askllm}

We look at the training samples that \askllm scoring models \emph{disagree on} as we go from Flan-T5-Small $\rightarrow$ Flan-T5-XXL. Specifically, we look at training samples that Flan-T5-Small thinks are of low quality, whereas Flan-T5-XXL thinks otherwise.
To the best of our understanding, our overarching conclusions by observing these qualitative samples are:
\begin{itemize}[leftmargin=*]
    \item Larger scoring models in \askllm are able to identify training samples containing \emph{tail-end} of knowledge, \eg, rare world-events, rare named entities, \etc 
    
    \item The increasing quality trend going from Flan-T5-Small $\rightarrow$ Flan-T5-XXL isn't correlated with the quality scoring model size in perplexity filtering.
\end{itemize}

\input{examples/increasing_askllm}

\subsection{Decreasing-quality Samples Identified by \askllm}

We look at the training samples that \askllm scoring models \emph{disagree on} as we go from Flan-T5-Small $\rightarrow$ Flan-T5-XXL. Specifically, we look at training samples that Flan-T5-XXL thinks are of low quality, whereas Flan-T5-Small thinks otherwise.
To the best of our understanding, our overarching conclusions by observing these qualitative samples are:
\begin{itemize}[leftmargin=*]
    \item Smaller quality-scoring models sometimes mislabel non-informative training samples, that contain, \eg, non-informative content, or repeated content. 
    
    \item The decreasing quality trend going from Flan-T5-Small $\rightarrow$ Flan-T5-XXL isn't correlated with the quality scoring model size in perplexity filtering.
\end{itemize}

\input{examples/decreasing_askllm}

%% file: algorithms/askllm.tex
\renewcommand{\algorithmiccomment}[1]{\hfill\texttt{// #1}}

\begin{algorithm}[H]
    \centering
    \caption{\askllm Sampling}
    \label{alg:askllm}
    \begin{algorithmic}[1]
    \STATE {\bfseries Input:} Dataset $\mathcal{D} = \{x_1,x_2,\cdots,x_N\}$ s.t. $x_i \in \mathcal{X}$ is the training sample in plain-text, sample size $k$, scoring model $\mathcal{M} : \mathcal{X}; \mathcal{X} \mapsto \mathbb{R}$
    
    \STATE Initialize list of scores $S = []$.
    \FOR{$n=1 \rightarrow N$}
        \STATE $\mathrm{prompt}_n \leftarrow \operatorname{make\_prompt}(x_n)$ \COMMENT{Make \askllm prompts as in \cref{fig:ask_llm_prompt}}
        \STATE Append $\mathcal{M}($``\texttt{yes}'' | $\mathrm{prompt}_n)$ to $S$ \COMMENT{Use $\mathcal{M}$ to score $x_n$}
    \ENDFOR
    
    
    \STATE {\bfseries Output:} Select $k$ elements from $\mathcal{D}$ with top-$k$ scores in $S$, without replacement.
    \end{algorithmic}
\end{algorithm}

%% file: algorithms/density.tex
\renewcommand{\algorithmiccomment}[1]{\hfill\texttt{// #1}}

\begin{algorithm}[H]
    \centering
    \caption{Inverse Propensity Sampling (IPS) via Kernel Density Estimation (KDE)}
    \label{alg:density}
    \begin{algorithmic}[1]
   \STATE {\bfseries Input:} Dataset $\mathcal{D} = \{x_1,x_2,\cdots,x_N\}$ of embedings, sample size $k$, kernel $k$ with corresponding locality-sensitive hash family $\mathcal{H}$ (see \citet{coleman2020sub}), hash range $B$, rows $R$, random seed $s$
    \STATE {\bfseries Initialize:} KDE sketch $ \mathcal{S} \leftarrow 0^{R \times B}$
   \STATE Generate $R$ independent hash functions ${h_1,\dots,h_R}$ from $\mathcal{H}$ with range $B$ and random seed $s$.
   \FOR[Construct KDE estimator for $D$.]{$ n= 1 \rightarrow N$}
    \FOR[Add $x_n$ to the KDE estimator.]{$r = 1 \rightarrow R$}
      \STATE $\mathcal{S}_{r, h_r(x_n)} += 1$
     \ENDFOR
    \ENDFOR
    \STATE Initialize list of scores $S = []$.
    \FOR[Score each example $x_n$]{$ n= 1 \rightarrow N$}
    \STATE $\mathrm{score} = 0$
    \FOR[Compute approximate KDE using $\mathcal{S}$]{$r = 1 \rightarrow R$}
      \STATE $\mathrm{score} += \mathcal{S}[r, h_r(x_n)]$
     \ENDFOR
     \STATE Append $\mathrm{score} / R$ to $S$
    \ENDFOR
    \STATE {\bfseries Output:} Select $k$ elements from $\mathcal{D}$ with probability $p = \frac{S}{\sum S}$ without replacement.
    \end{algorithmic}
\end{algorithm}

%% file: examples/good_askllm.tex
\newtcolorbox[auto counter]{tbox}[3][]{%
colframe=#2,colback=#2!8,coltitle=#2!20!black,#1,
title={Example \thetcbcounter: #3}
}

\setlength{\tabcolsep}{0.7em}
\def\hrulefill{\leavevmode\leaders\hrule height 1.5pt\hfill\kern0pt}

\begin{tbox}[label=example:good1]{ForestGreen}{Estimated Data-Quality (Percentile -- Higher is better)}
\centering
\begin{tabular}{ccccc|ccccc}
    \multicolumn{5}{c|}{\textbf{\askllm}} & \multicolumn{5}{c}{\textbf{Perplexity Filtering}} \\[2pt]
     Small & Base & Large & XL & XXL & Small & Base & Large & XL & XXL  \\
     \midrule
     93.33\% & 88.21\% & 88.11\% & 100.0\% & 99.99\% & 50.29\% & 30.34\% & 32.56\% & 31.61\% & 25.62\%
\end{tabular}

\vspace{0.3cm}
\hspace{0.5cm}\hrulefill\hspace{0.2cm} \floweroneleft\floweroneright \hspace{0.2cm}\hrulefill\hspace{0.5cm}
\vspace{0.3cm}

What constitutes overtime for a part-time employee?
Question: What is overtime for a part-time employee?
Overtime for a part-time employee is time that is beyond the part-time employee’s ordinary hours of work or outside the agreed number of hours of work, as specified in their employment contract.
\end{tbox}

\begin{tbox}[label=example:good2]{ForestGreen}{Estimated Data-Quality (Percentile -- Higher is better)}
\centering
\begin{tabular}{ccccc|ccccc}
    \multicolumn{5}{c|}{\textbf{\askllm}} & \multicolumn{5}{c}{\textbf{Perplexity Filtering}} \\[2pt]
     Small & Base & Large & XL & XXL & Small & Base & Large & XL & XXL  \\
     \midrule
     99.86\% & 98.54\% & 96.4\% & 96.3\% & 96.67\% & 46.2\% & 54.65\% & 46.2\% & 49.85\% & 20.33\%
\end{tabular}

\vspace{0.3cm}
\hspace{0.5cm}\hrulefill\hspace{0.2cm} \floweroneleft\floweroneright \hspace{0.2cm}\hrulefill\hspace{0.5cm}
\vspace{0.3cm}

Viva La Vegan! - Can a Vegan Lifestyle Help to Get Rid of Ocean Dead Zones?
Can a Vegan Lifestyle Help to Get Rid of Ocean Dead Zones?
A dead zone is an area at the bottom of the ocean that is oxygen depleted and cannot maintain any marine life. The biggest cause of these dead zones is an overflow of fertilizers, sewage and industrial pollutants being pumped into rivers all over the world.
Thankfully dead zones can be reversed and living a vegan lifestyle can help enormously and I’ll show you how.
What are Ocean Dead Zones?

......

Vegans don’t want to harm the planet. On the contrary they want to save it and what better way than living with nature instead of against it and helping the planet in ways we probably never even realised, like helping to reverse our oceans dead zones.
Next time you think about buying something you don’t need, or eating food that is highly processed or non-organic, spare a thought for the largely unknown dead zones and how overconsumption and an unnatural lifestyle is slowly killing both you and them.
\end{tbox} 

\begin{tbox}[label=example:good3]{ForestGreen}{Estimated Data-Quality (Percentile -- Higher is better)}
\centering
\begin{tabular}{ccccc|ccccc}
    \multicolumn{5}{c|}{\textbf{\askllm}} & \multicolumn{5}{c}{\textbf{Perplexity Filtering}} \\[2pt]
     Small & Base & Large & XL & XXL & Small & Base & Large & XL & XXL  \\
     \midrule
     98.81\% & 98.96\% & 95.42\% & 99.53\% & 99.56\% & 88.1\% & 80.99\% & 77.13\% & 65.89\% & 73.79\%
\end{tabular}

\vspace{0.3cm}
\hspace{0.5cm}\hrulefill\hspace{0.2cm} \floweroneleft\floweroneright \hspace{0.2cm}\hrulefill\hspace{0.5cm}
\vspace{0.3cm}

Question: Is it necessary to dredge ponds and lakes in the upper coastal region of South Carolina?
Answer: It is necessary to dredge ponds and lakes in South Carolina, in the upper coastal region of South Carolina. Each lake and each pond is a different environment and as years pass, these environments accumulate a lot of sediment. They tend to fill in with storm water runoff, they tend from natural leafy materials—whether it be grass clippings, leafy materials, storm water fun off, sand, silt, sediment, muck, mire. All of these produce in the bottoms of pond beds and lake beds. So it is absolutely necessary to do an evaluation every so many years to determine whether or not you need to remove the sediment that’s accumulated.
\end{tbox}

\begin{tbox}[label=example:good4]{ForestGreen}{Estimated Data-Quality (Percentile -- Higher is better)}
\centering
\begin{tabular}{ccccc|ccccc}
    \multicolumn{5}{c|}{\textbf{\askllm}} & \multicolumn{5}{c}{\textbf{Perplexity Filtering}} \\[2pt]
     Small & Base & Large & XL & XXL & Small & Base & Large & XL & XXL  \\
     \midrule
     88.93\% & 92.16\% & 90.3\% & 95.14\% & 93.44\% & 26.83\% & 34.32\% & 32.98\% & 31.14\% & 28.35\%
\end{tabular}

\vspace{0.3cm}
\hspace{0.5cm}\hrulefill\hspace{0.2cm} \floweroneleft\floweroneright \hspace{0.2cm}\hrulefill\hspace{0.5cm}
\vspace{0.3cm}

However, it’s a long and challenging way to mass production.
New Tesla Model 3 is an electric game-changer worth \$35,000 and comes in classic black color. A single masterpiece in black now belongs to Tesla’s CEO and co-founder Elon Musk.
Why not mass market yet? Company has a quite complicated reason. Tesla needs to make sure that it can build, deliver and service enormous numbers of these awesome electric cars without sacrificing quality.
Tesla will present 30 first cars at a launch celebration dated on July 28.
100 cars with production speed 3 cars per day dated for August.
1,500 cars will be ready for September.

...

Owners of new Teslas will also enjoy exquisite aerodynamic wheel face.
An itemized list of the Tesla Model 3’s features, specs, and pricing is expected to be revealed on July 28, at the car’s launch party.
5.6 seconds is what it gets the Model 3 to go from zero to 60 miles per hour, as May news says. Hot, right? It accelerates even faster than the base model BMW 3 Series or the famous Mercedes-Benz C Class, which are leaders in the compact luxury space.
A single charge will allow minimum 215 miles of single drive.
The roof in Model 3 is made almost entirely of glass, providing an incredible sense of space and infinity. Moreover, it blocks UV rays and manages the level of heat.
\end{tbox}

\begin{tbox}[label=example:good5]{ForestGreen}{Estimated Data-Quality (Percentile -- Higher is better)}
\centering
\begin{tabular}{ccccc|ccccc}
    \multicolumn{5}{c|}{\textbf{\askllm}} & \multicolumn{5}{c}{\textbf{Perplexity Filtering}} \\[2pt]
     Small & Base & Large & XL & XXL & Small & Base & Large & XL & XXL  \\
     \midrule
     89.28\% & 98.11\% & 98.93\% & 98.7\% & 96.32\% & 26.24\% & 19.14\% & 26.25\% & 26.05\% & 24.29\%
\end{tabular}

\vspace{0.3cm}
\hspace{0.5cm}\hrulefill\hspace{0.2cm} \floweroneleft\floweroneright \hspace{0.2cm}\hrulefill\hspace{0.5cm}
\vspace{0.3cm}

Landmines. Every month, 1200 people are maimed, and a further 800 killed throughout the world due to landmines. Landmine removal efforts are clearing about 100,000 mines a year, but at rate it will still be over 1000 years to get them all. The cost of clearing them is huge, with estimates in excess of \$50 billion. Worse still, for every 5000 mines cleared, one person will die in the process.

...

Hopefully the work that people like Vandiver and Tan can be built upon and further progress can be made in the fight to clear the world of landmines. The video below shows a group of minesweepers working with the kits- and it is clear even watching them that the level of understanding as to how the mine operates is already improving- giving them the knowledge they need to safely diffuse any mines they encounter.
\end{tbox}

\begin{tbox}[label=example:good6]{ForestGreen}{Estimated Data-Quality (Percentile -- Higher is better)}
\centering
\begin{tabular}{ccccc|ccccc}
    \multicolumn{5}{c|}{\textbf{\askllm}} & \multicolumn{5}{c}{\textbf{Perplexity Filtering}} \\[2pt]
     Small & Base & Large & XL & XXL & Small & Base & Large & XL & XXL  \\
     \midrule
     87.79\% & 98.52\% & 90.11\% & 91.65\% & 88.09\% & 19.72\% & 17.88\% & 21.13\% & 16.95\% & 11.92\%
\end{tabular}

\vspace{0.3cm}
\hspace{0.5cm}\hrulefill\hspace{0.2cm} \floweroneleft\floweroneright \hspace{0.2cm}\hrulefill\hspace{0.5cm}
\vspace{0.3cm}

By all measures a successful chemical engineering undergraduate at Oregon Agricultural College, and wanting very much to continue his education and earn his PhD in chemistry, Linus Pauling wrote to several graduate programs across the country, inquiring in particular about fellowships. Though he had proven himself to be prodigious talent as a student and, already, as a teacher, Pauling’s location in Corvallis didn’t carry a great deal of cache with the country’s elite institutions. And given his family’s shaky financial health, some measure of institutional funding was going to be required if he were to advance in the academy.

...

During his sparse free time, Pauling wrote letter after letter to his girlfriend, Ava Helen Miller, who remained in Corvallis to continue work on her Home Economics degree at OAC. Having expressed a desire to marry at least twice before Linus left for California, only to be rebuffed by their families, the two decided in their letters that they would absolutely be wed once Pauling had finished his first year of classes and just prior to his resumption of more construction work during the summer. Their plan came to fruition in Salem, Oregon on June 17, 1923, and Ava Helen moved to Pasadena that fall to accompany her new husband during his second year as a graduate student.
\end{tbox}

\begin{tbox}[label=example:good7]{ForestGreen}{Estimated Data-Quality (Percentile -- Higher is better)}
\centering
\begin{tabular}{ccccc|ccccc}
    \multicolumn{5}{c|}{\textbf{\askllm}} & \multicolumn{5}{c}{\textbf{Perplexity Filtering}} \\[2pt]
     Small & Base & Large & XL & XXL & Small & Base & Large & XL & XXL  \\
     \midrule
     87.08\% & 89.33\% & 95.26\% & 99.13\% & 99.94\% & 98.09\% & 97.52\% & 98.83\% & 97.39\% & 97.38\%
\end{tabular}

\vspace{0.3cm}
\hspace{0.5cm}\hrulefill\hspace{0.2cm} \floweroneleft\floweroneright \hspace{0.2cm}\hrulefill\hspace{0.5cm}
\vspace{0.3cm}

Bonelli, N.; Giordano, S.; Procissi, G. Enif-Lang: A Specialized Language for Programming Network Functions on Commodity Hardware. J. Sens. Actuator Netw. 2018, 7, 34.
Bonelli N, Giordano S, Procissi G. Enif-Lang: A Specialized Language for Programming Network Functions on Commodity Hardware. Journal of Sensor and Actuator Networks. 2018; 7(3):34.
Bonelli, Nicola; Giordano, Stefano; Procissi, Gregorio. 2018. "Enif-Lang: A Specialized Language for Programming Network Functions on Commodity Hardware." J. Sens. Actuator Netw. 7, no. 3: 34.
\end{tbox}

\begin{tbox}[label=example:good8]{ForestGreen}{Estimated Data-Quality (Percentile -- Higher is better)}
\centering
\begin{tabular}{ccccc|ccccc}
    \multicolumn{5}{c|}{\textbf{\askllm}} & \multicolumn{5}{c}{\textbf{Perplexity Filtering}} \\[2pt]
     Small & Base & Large & XL & XXL & Small & Base & Large & XL & XXL  \\
     \midrule
     96.41\% & 86.03\% & 97.38\% & 95.91\% & 90.8\% & 34.7\% & 44.8\% & 56.87\% & 60.15\% & 77.25\%
\end{tabular}

\vspace{0.3cm}
\hspace{0.5cm}\hrulefill\hspace{0.2cm} \floweroneleft\floweroneright \hspace{0.2cm}\hrulefill\hspace{0.5cm}
\vspace{0.3cm}

"What is your number one secret to productivity?"
In recording their responses, Kruse came across some fascinating suggestions. What follows are some of my favorites.
They focus on minutes, not hours. Most people default to hour and half-hour blocks on their calendar; highly successful people know that there are 1,440 minutes in every day and that there is nothing more valuable than time. Money can be lost and made again, but time spent can never be reclaimed. As legendary Olympic gymnast Shannon Miller told Kevin, "To this day, I keep a schedule that is almost minute by minute." You must master your minutes to master your life.

...

Energy is everything. You can't make more minutes in the day, but you can increase your energy to increase your attention, focus, and productivity. Highly successful people don't skip meals, sleep, or breaks in the pursuit of more, more, more. Instead, they view food as fuel, sleep as recovery, and breaks as opportunities to recharge in order to get even more done.
Author of \#1 bestselling book, Emotional Intelligence 2.0, and president of TalentSmart, world's leading provider of emotional intelligence.
\end{tbox}

%% file: examples/bad_askllm.tex

\setlength{\tabcolsep}{0.7em}
\def\hrulefill{\leavevmode\leaders\hrule height 1.5pt\hfill\kern0pt}

\begin{tbox}[label=example:bad1]{Bittersweet}{Estimated Data-Quality (Percentile -- Higher is better)}
\centering
\begin{tabular}{ccccc|ccccc}
    \multicolumn{5}{c|}{\textbf{\askllm}} & \multicolumn{5}{c}{\textbf{Perplexity Filtering}} \\[2pt]
     Small & Base & Large & XL & XXL & Small & Base & Large & XL & XXL  \\
     \midrule
     0.01\% & 0.01\% & 0.01\% & 0.0\% & 0.0\% & 40.46\% & 25.66\% & 27.42\% & 25.6\% & 28.12\%
\end{tabular}

\vspace{0.3cm}
\hspace{0.5cm}\hrulefill\hspace{0.2cm} \floweroneleft\floweroneright \hspace{0.2cm}\hrulefill\hspace{0.5cm}
\vspace{0.3cm}

Release name : Juiced2.Hot.Import.Nights-Multi5-RELOADED. ? Format : iso Juiced 2: HIN evolves the current street racing scene, letting players experience PC Repack DiRT Rally v1.1 ? Black Box Bears Cant Drift PC torrent uploaded.  ?  Juiced 2 ? ?  ?? ? ???? ???? ? ??? ?  ?? ? ? ? ?  ????! .

...

HIN evolves the current street racing scene, letting players experience the culture of the real-life HIN tour, the nation?s largest lifestyle custom. Juiced 2 Hot Import Nights Torrent. Bittorrent 729.64 MB. Juiced 2 Hot Import Nights Download free torrent at Largest Bittorrent Source with Several Listed Files. Now you can upload screenshots or other images (cover scans, disc scans,...
\end{tbox}

\begin{tbox}[label=example:bad2]{Bittersweet}{Estimated Data-Quality (Percentile -- Higher is better)}
\centering
\begin{tabular}{ccccc|ccccc}
    \multicolumn{5}{c|}{\textbf{\askllm}} & \multicolumn{5}{c}{\textbf{Perplexity Filtering}} \\[2pt]
     Small & Base & Large & XL & XXL & Small & Base & Large & XL & XXL  \\
     \midrule
     5.41\% & 3.86\% & 0.49\% & 0.8\% & 6.24\% & 62.97\% & 75.91\% & 86.3\% & 85.26\% & 88.11\%
\end{tabular}

\vspace{0.3cm}
\hspace{0.5cm}\hrulefill\hspace{0.2cm} \floweroneleft\floweroneright \hspace{0.2cm}\hrulefill\hspace{0.5cm}
\vspace{0.3cm}

You were a good daughter the first day or two. Now, you are only showing the worst sides of yourself. I can only be sad and disappointed in you.
\end{tbox}

\begin{tbox}[label=example:bad3]{Bittersweet}{Estimated Data-Quality (Percentile -- Higher is better)}
\centering
\begin{tabular}{ccccc|ccccc}
    \multicolumn{5}{c|}{\textbf{\askllm}} & \multicolumn{5}{c}{\textbf{Perplexity Filtering}} \\[2pt]
     Small & Base & Large & XL & XXL & Small & Base & Large & XL & XXL  \\
     \midrule
     1.08\% & 0.41\% & 6.16\% & 2.46\% & 1.44\% & 35.97\% & 24.13\% & 31.46\% & 51.15\% & 38.19\%
\end{tabular}

\vspace{0.3cm}
\hspace{0.5cm}\hrulefill\hspace{0.2cm} \floweroneleft\floweroneright \hspace{0.2cm}\hrulefill\hspace{0.5cm}
\vspace{0.3cm}

Kids can help you enrich your life? Be a better person? Learn to think about someone else? Apparently whoever said these things has never had children because from everything we have seen and experienced, kids are flat out horrible. College can’t come fast enough.
\end{tbox}

\begin{tbox}[label=example:bad4]{Bittersweet}{Estimated Data-Quality (Percentile -- Higher is better)}
\centering
\begin{tabular}{ccccc|ccccc}
    \multicolumn{5}{c|}{\textbf{\askllm}} & \multicolumn{5}{c}{\textbf{Perplexity Filtering}} \\[2pt]
     Small & Base & Large & XL & XXL & Small & Base & Large & XL & XXL  \\
     \midrule
     1.89\% & 3.58\% & 3.11\% & 6.02\% & 0.09\% & 18.09\% & 22.8\% & 25.61\% & 19.14\% & 47.01\%
\end{tabular}

\vspace{0.3cm}
\hspace{0.5cm}\hrulefill\hspace{0.2cm} \floweroneleft\floweroneright \hspace{0.2cm}\hrulefill\hspace{0.5cm}
\vspace{0.3cm}

EventsThis is how you can go ice skating with real penguinsGrab your tickets before they sell out!
Can you spot anyone you know in these fun pics?
EventsHow do I get tickets for Wimbledon 2018?
\end{tbox}

\begin{tbox}[label=example:bad5]{Bittersweet}{Estimated Data-Quality (Percentile -- Higher is better)}
\centering
\begin{tabular}{ccccc|ccccc}
    \multicolumn{5}{c|}{\textbf{\askllm}} & \multicolumn{5}{c}{\textbf{Perplexity Filtering}} \\[2pt]
     Small & Base & Large & XL & XXL & Small & Base & Large & XL & XXL  \\
     \midrule
     2.17\% & 1.11\% & 3.75\% & 2.0\% & 5.31\% & 92.49\% & 89.88\% & 86.79\% & 97.04\% & 96.78\%
\end{tabular}

\vspace{0.3cm}
\hspace{0.5cm}\hrulefill\hspace{0.2cm} \floweroneleft\floweroneright \hspace{0.2cm}\hrulefill\hspace{0.5cm}
\vspace{0.3cm}

That I don’t make you happy?
We can start all over some day?
Somewhere, are you dreaming of me?
Won’t you come back home to me?
\end{tbox}

\begin{tbox}[label=example:bad6]{Bittersweet}{Estimated Data-Quality (Percentile -- Higher is better)}
\centering
\begin{tabular}{ccccc|ccccc}
    \multicolumn{5}{c|}{\textbf{\askllm}} & \multicolumn{5}{c}{\textbf{Perplexity Filtering}} \\[2pt]
     Small & Base & Large & XL & XXL & Small & Base & Large & XL & XXL  \\
     \midrule
     0.06\% & 0.04\% & 0.08\% & 0.11\% & 0.07\% & 68.86\% & 51.15\% & 44.08\% & 35.81\% & 19.28\%
\end{tabular}

\vspace{0.3cm}
\hspace{0.5cm}\hrulefill\hspace{0.2cm} \floweroneleft\floweroneright \hspace{0.2cm}\hrulefill\hspace{0.5cm}
\vspace{0.3cm}

  ? ,  ? ,  ? ,  ? ,  ?   ? ,  ?   ? . (1395).  ?   ?   ?   ?   ?   ?   ?   ?   ?   ?   ?   ?   ?   ?   ?   ?   ? .  ?   ?   ?   ? , 26(2), 145-159.  ?   ? ;  ?   ? ;  ?   ?   ?   ? . " ?   ?   ?   ?   ?   ?   ?   ?   ?   ?   ?   ?   ?   ?   ?   ?   ? ".  ?   ?   ?   ? , 26, 2, 1395, 145-159.  ? ,  ? ,  ? ,  ? ,  ?   ? ,  ?   ? . (1395). ' ?   ?   ?   ?   ?   ?   ?   ?   ?   ?   ?   ?   ?   ?   ?   ?   ? ',  ?   ?   ?   ? , 26(2), pp. 145-159.  ? ,  ? ,  ? ,  ? ,  ?   ? ,  ?   ? .  ?   ?   ?   ?   ?   ?   ?   ?   ?   ?   ?   ?   ?   ?   ?   ?   ? .  ?   ?   ?   ? , 1395; 26(2): 145-159.  ?   ?   ?   ?   ?   ? ? ?   ?   ?   ?   ?   ?  BHT  ?   ? ? ?   ?   ?   ?  DPPH ?   ?   ?   ?   ?   ?   ?   ?   ?   ?   ?   ? .  ?   ?   ?   ?   ?   ?   ?   ?   ?   ?   ?   ?  (HPMC)  ?   ?   ?   ?   ?   ?   ?   ?   ?   ?   ?   ?   ?   ?   ?   ?   ?   ? ? ?   ?   ?   

...

Effect of the plasticizer on permeability, mechanical resistance and thermal behaviour of composite coating films. Powder Technology 238:14-19. Martos MV, Mohamady MA, Fern?ndez?L?pez J, Abd ElRazik KA, Omer EA, P?rez?Alvarez JA and Sendra E, 2011. In vitro antioxidant and antibacterial activities of essentials oils obtained from Egyptian aromatic plants. Food Control 22: 1715?1722. Phoopuritham P, Thongngam M, Yoksan R and Suppakul P, 2011. Antioxidant Properties of Selected Plant Extracts and Application in Packaging as Antioxidant Cellulose?Based Films for Vegetable Oil. Packaging Technology and Science 25: 125?136. Rojas?Gra? MA, Avena?Bustillos RJ, Olsen C, Friedman M, Henika PR, Martin?Belloso O, Pan Zh and McHughTH, 2007. Effects...
\end{tbox}

\begin{tbox}[label=example:bad7]{Bittersweet}{Estimated Data-Quality (Percentile -- Higher is better)}
\centering
\begin{tabular}{ccccc|ccccc}
    \multicolumn{5}{c|}{\textbf{\askllm}} & \multicolumn{5}{c}{\textbf{Perplexity Filtering}} \\[2pt]
     Small & Base & Large & XL & XXL & Small & Base & Large & XL & XXL  \\
     \midrule
     0.01\% & 0.02\% & 0.02\% & 0.01\% & 0.0\% & 59.41\% & 36.81\% & 23.01\% & 12.95\% & 17.24\%
\end{tabular}

\vspace{0.3cm}
\hspace{0.5cm}\hrulefill\hspace{0.2cm} \floweroneleft\floweroneright \hspace{0.2cm}\hrulefill\hspace{0.5cm}
\vspace{0.3cm}

Showing results for tags 'A3arma\_start'. I have a Error mesage "Addon 'A3\_epoch\_server' requires addon 'A3\_epoch\_config'" why is that and how can i fix this? When i click Ok i get this My Start.cmd losk like this: arma3server.exe [email protected];@EpochHive; -config=C: ? arma 3 ? SC ? config.cfg -ip=192.168.71.234 -port=2301 -profiles=SC -cfg=C: ? arma 3 ? SC ? basic.cfg -name=SC This is my RPT file: ===================================================================== == C: ? arma 3 ? arma3server.exe == arma3server.exe [email protected];@EpochHive; -config=C: ? arma 3 ? SC ? 

...

2:05:23 Updating base class ->RscListBox, by a3 ? ui\_f ? config.bin/RscIGUIListBox/ 2:05:23 Updating base class ->RscListNBox, by a3 ? ui\_f ? config.bin/RscIGUIListNBox/ 2:05:23 Updating base class ->RscText, by a3 ? ui\_f ? config.bin/RscBackground/ 2:05:23 Updating base class ->RscText, by a3 ? ui\_f ? config.bin/RscBackgroundGUI/ 2:05:23 Updating base class ->RscPicture, by a3 ? ui\_f ? config.bin/RscBackgroundGUILeft/ 2:05:23 Updating base class ->RscPicture, by a3 ? ui\_f ? config.bin/RscBackgroundGUIRight/ 2:05:23 Updating base class ->RscPicture, by a3 ? ui\_f ? config.bin/RscBackgroundGUIBottom/ 2:05:23 Updating base class ->RscText, by a3...
\end{tbox}

\begin{tbox}[label=example:bad8]{Bittersweet}{Estimated Data-Quality (Percentile -- Higher is better)}
\centering
\begin{tabular}{ccccc|ccccc}
    \multicolumn{5}{c|}{\textbf{\askllm}} & \multicolumn{5}{c}{\textbf{Perplexity Filtering}} \\[2pt]
     Small & Base & Large & XL & XXL & Small & Base & Large & XL & XXL  \\
     \midrule
     0.47\% & 3.79\% & 1.93\% & 1.08\% & 10.22\% & 51.15\% & 46.92\% & 63.04\% & 44.77\% & 41.35\%
\end{tabular}

\vspace{0.3cm}
\hspace{0.5cm}\hrulefill\hspace{0.2cm} \floweroneleft\floweroneright \hspace{0.2cm}\hrulefill\hspace{0.5cm}
\vspace{0.3cm}

10 February 2019 I have 2 houses (joint - me \& my wife) in my name and 2 land (plots). Recently sold one of flat (100\% cheque payment).
Can I reinvest the Capital gains arriving out of sale in purchasing a flat?
Note: I had reinvested earlier on (4 years ago) the similar captial gains to buy land from a house sale.
\end{tbox}

%% file: examples/increasing_askllm.tex
\setlength{\tabcolsep}{0.7em}
\def\hrulefill{\leavevmode\leaders\hrule height 1.5pt\hfill\kern0pt}

\begin{tbox}[label=example:increasing1]{Cerulean}{Estimated Data-Quality (Percentile -- Higher is better)}
\centering
\begin{tabular}{ccccc|ccccc}
    \multicolumn{5}{c|}{\textbf{\askllm}} & \multicolumn{5}{c}{\textbf{Perplexity Filtering}} \\[2pt]
     Small & Base & Large & XL & XXL & Small & Base & Large & XL & XXL  \\
     \midrule
     7.67\% & 30.45\% & 57.41\% & 78.17\% & 97.41\% & 15.56\% & 31.02\% & 24.14\% & 50.59\% & 49.64\%
\end{tabular}

\vspace{0.3cm}
\hspace{0.5cm}\hrulefill\hspace{0.2cm} \floweroneleft\floweroneright \hspace{0.2cm}\hrulefill\hspace{0.5cm}
\vspace{0.3cm}

The historic city of Manchester now features one of the most interesting public art installations that art lovers have ever witnessed. Design studio, Acrylicize installed five giant lamps in Piccadilly Place that represent the many historic periods that the city has gone through, including; Art Deco, Art Nouveau, Victorian, mid-century, and contemporary.
The installation is without any doubt, a great piece of art but unlike other artworks, these are absolutely functional as well. Each lamp provides the many visitors with seating, shelter, light and even heat in the winters. The admirers can also witness the historic stories of Manchester via graphic illustrations on the lamps.
\end{tbox} 

\begin{tbox}[label=example:increasing2]{Cerulean}{Estimated Data-Quality (Percentile -- Higher is better)}
\centering
\begin{tabular}{ccccc|ccccc}
    \multicolumn{5}{c|}{\textbf{\askllm}} & \multicolumn{5}{c}{\textbf{Perplexity Filtering}} \\[2pt]
     Small & Base & Large & XL & XXL & Small & Base & Large & XL & XXL  \\
     \midrule
     10.48\% & 31.26\% & 54.17\% & 84.17\% & 97.93\% & 30.52\% & 39.49\% & 35.79\% & 30.89\% & 25.39\%
\end{tabular}

\vspace{0.3cm}
\hspace{0.5cm}\hrulefill\hspace{0.2cm} \floweroneleft\floweroneright \hspace{0.2cm}\hrulefill\hspace{0.5cm}
\vspace{0.3cm}

The Cokin Yellow and Pink Center Spot filter has a clear center and diffused yellow and pink edges. Theses diffused edges will be produce blur while leaving the center sharp. The filter effect is directly influenced by the f-stop and the focal length. A lens shot at f/1.4 will see a greater blurring effect than f/8.0 and a 85mm lens will see more blur than a 28mm. Additionally, a longer focal length lens will visually increase the size of the center spot area because it sees less of the filter area.
\end{tbox} 

\begin{tbox}[label=example:increasing3]{Cerulean}{Estimated Data-Quality (Percentile -- Higher is better)}
\centering
\begin{tabular}{ccccc|ccccc}
    \multicolumn{5}{c|}{\textbf{\askllm}} & \multicolumn{5}{c}{\textbf{Perplexity Filtering}} \\[2pt]
     Small & Base & Large & XL & XXL & Small & Base & Large & XL & XXL  \\
     \midrule
     7.05\% & 20.29\% & 38.23\% & 50.38\% & 63.94\% & 22.41\% & 14.8\% & 12.69\% & 20.68\% & 8.62\%
\end{tabular}

\vspace{0.3cm}
\hspace{0.5cm}\hrulefill\hspace{0.2cm} \floweroneleft\floweroneright \hspace{0.2cm}\hrulefill\hspace{0.5cm}
\vspace{0.3cm}

Provide hoist coverage and 200 degree rotation for individual use in bays, along walls, or columns of plants, or as a supplement to an overhead crane or monorail system. This jib has the advantage of providing maximum lift for the hoist, since it can be installed very close to the underside of the lowest ceiling obstruction. It is composed of a vertical mast mounted to 2 brackets on a wall or vertical building beam with a boom that cantilevers out, perpendicular from the wall at the top.
\end{tbox} 

\begin{tbox}[label=example:increasing4]{Cerulean}{Estimated Data-Quality (Percentile -- Higher is better)}
\centering
\begin{tabular}{ccccc|ccccc}
    \multicolumn{5}{c|}{\textbf{\askllm}} & \multicolumn{5}{c}{\textbf{Perplexity Filtering}} \\[2pt]
     Small & Base & Large & XL & XXL & Small & Base & Large & XL & XXL  \\
     \midrule
     20.76\% & 45.81\% & 60.22\% & 73.95\% & 84.14\% & 2.98\% & 2.94\% & 3.49\% & 2.51\% & 2.09\%
\end{tabular}

\vspace{0.3cm}
\hspace{0.5cm}\hrulefill\hspace{0.2cm} \floweroneleft\floweroneright \hspace{0.2cm}\hrulefill\hspace{0.5cm}
\vspace{0.3cm}

The mighty Adyar River that flows through Chennai has a tale to tell.
Arun Krishnamurthy, founder, Environmentalist Foundation of India has documented the origin of the river, the journey and the culmination all captured in images aimed at sensitizing citizens of Chennai to a treasure that they are being denied.
Titled Urban Waters, the photo exhibition on Adyar river will bring out Adyar’s rich history, fine ecology, urban exploitation and her innate beauty through framed images. The exhibition is organised at Max Mueller Bhavan in Chennai.
Goethe Institut, Max Mueller Bhavan is at 4, 5th Street, Rutland Gate, Chennai.
\end{tbox} 

\begin{tbox}[label=example:increasing5]{Cerulean}{Estimated Data-Quality (Percentile -- Higher is better)}
\centering
\begin{tabular}{ccccc|ccccc}
    \multicolumn{5}{c|}{\textbf{\askllm}} & \multicolumn{5}{c}{\textbf{Perplexity Filtering}} \\[2pt]
     Small & Base & Large & XL & XXL & Small & Base & Large & XL & XXL  \\
     \midrule
     4.27\% & 22.22\% & 47.57\% & 82.58\% & 92.4\% & 6.34\% & 4.77\% & 3.89\% & 8.75\% & 7.55\%
\end{tabular}

\vspace{0.3cm}
\hspace{0.5cm}\hrulefill\hspace{0.2cm} \floweroneleft\floweroneright \hspace{0.2cm}\hrulefill\hspace{0.5cm}
\vspace{0.3cm}

The Pendaries Village Skyline Subdivision is located near both the Santa Fe National Forest and the Pecos Wilderness in North Central New Mexico. It has the charm of small town New Mexico, perhaps even more so than its better known nearby sister cities. It offers a unique opportunity for people wishing to enjoy the quiet beauty of Northern New Mexico.
\end{tbox} 

\begin{tbox}[label=example:increasing6]{Cerulean}{Estimated Data-Quality (Percentile -- Higher is better)}
\centering
\begin{tabular}{ccccc|ccccc}
    \multicolumn{5}{c|}{\textbf{\askllm}} & \multicolumn{5}{c}{\textbf{Perplexity Filtering}} \\[2pt]
     Small & Base & Large & XL & XXL & Small & Base & Large & XL & XXL  \\
     \midrule
     22.09\% & 66.57\% & 76.56\% & 85.51\% & 96.98\% & 20.8\% & 24.82\% & 17.42\% & 18.65\% & 15.55\%
\end{tabular}

\vspace{0.3cm}
\hspace{0.5cm}\hrulefill\hspace{0.2cm} \floweroneleft\floweroneright \hspace{0.2cm}\hrulefill\hspace{0.5cm}
\vspace{0.3cm}

Anderson .Paak’s new album, Oxnard, is a nod to the Southern California city where Anderson grew up. It is the Grammy-nominated artist’s third studio album and the first to be released on Dr. Dre’s label Aftermath Entertainment. Oxnard includes his latest single, Tints featuring Kendrick Lamar along with album features from J Cole, Pusha T and many more. This is the album he dreamed of making in high school, when he was listening to Jay-Z’s The Blueprint, The Game’s The Documentary, and Kanye West’s The College Dropout.
The classic fourth album from the rap-god Eminem.
\end{tbox} 

\begin{tbox}[label=example:increasing7]{Cerulean}{Estimated Data-Quality (Percentile -- Higher is better)}
\centering
\begin{tabular}{ccccc|ccccc}
    \multicolumn{5}{c|}{\textbf{\askllm}} & \multicolumn{5}{c}{\textbf{Perplexity Filtering}} \\[2pt]
     Small & Base & Large & XL & XXL & Small & Base & Large & XL & XXL  \\
     \midrule
     0.98\% & 24.84\% & 53.36\% & 88.98\% & 98.18\% & 2.3\% & 1.48\% & 2.03\% & 2.1\% & 3.07\%
\end{tabular}

\vspace{0.3cm}
\hspace{0.5cm}\hrulefill\hspace{0.2cm} \floweroneleft\floweroneright \hspace{0.2cm}\hrulefill\hspace{0.5cm}
\vspace{0.3cm}

The Disknet is a networking solution which uses the external floppy drive port of the Amiga. It uses the same coax cabling as 10Base2 Ethernet (RG-58U/50Ohm) but is NOT compatible and is capable of transferring at around 45k/sec.
The Disknet may be the same device as the AmigaLink, but this has not been confirmed.
\end{tbox}

%% file: examples/decreasing_askllm.tex
\setlength{\tabcolsep}{0.7em}
\def\hrulefill{\leavevmode\leaders\hrule height 1.5pt\hfill\kern0pt}

\begin{tbox}{BurntOrange}{Estimated Data-Quality (Percentile -- Higher is better)}
\centering
\begin{tabular}{ccccc|ccccc}
    \multicolumn{5}{c|}{\textbf{\askllm}} & \multicolumn{5}{c}{\textbf{Perplexity Filtering}} \\[2pt]
     Small & Base & Large & XL & XXL & Small & Base & Large & XL & XXL  \\
     \midrule
     64.05\% & 46.39\% & 35.92\% & 25.29\% & 9.63\% & 4.3\% & 10.21\% & 3.47\% & 3.34\% & 3.35\%
\end{tabular}

\vspace{0.3cm}
\hspace{0.5cm}\hrulefill\hspace{0.2cm} \floweroneleft\floweroneright \hspace{0.2cm}\hrulefill\hspace{0.5cm}
\vspace{0.3cm}

one filled with goodwill and cheer.
who have supported me thru the year.
I wouldn't be changing careers.
instead of on strange people's rears.
Wishes You a Healthy, Happy Holidays!
Ah, how the mighty have fallen!
And a Merry fave to you ... and a happy new rear.
From one Xmas humor story to another, enjoyed this!
Thanks Jack \& Susan! Doug, I checked him out–wonderful stuff! Will pass along the good word.
Fun and funny--as always! Thanks for the cheer!
I can only fave this once, but I've looked at it repeatedly over what has been a bizarre week-- and each time you've given me a laugh. That's a gift Bob and I'm grateful! Best of holidays to you and a great New Year!
\end{tbox} 

\begin{tbox}{BurntOrange}{Estimated Data-Quality (Percentile -- Higher is better)}
\centering
\begin{tabular}{ccccc|ccccc}
    \multicolumn{5}{c|}{\textbf{\askllm}} & \multicolumn{5}{c}{\textbf{Perplexity Filtering}} \\[2pt]
     Small & Base & Large & XL & XXL & Small & Base & Large & XL & XXL  \\
     \midrule
     91.25\% & 71.8\% & 53.1\% & 24.11\% & 4.53\% & 32.4\% & 36.56\% & 46.53\% & 48.19\% & 54.84\%
\end{tabular}

\vspace{0.3cm}
\hspace{0.5cm}\hrulefill\hspace{0.2cm} \floweroneleft\floweroneright \hspace{0.2cm}\hrulefill\hspace{0.5cm}
\vspace{0.3cm}

I hear people saying that vinyl records have a better sound quality than CDs or even DVDs. A mini LP is a CD version of something that was originally released as a 12" (12 inch) vinyl LP. In many cases the packaging is superior to, or at least. Vitalogy; Studio album by Pearl Jam; Released: Vinyl: November 22, 1994 CD: December 6, 1994: Recorded: November 1993 – October 1994: Studio: Bad Animals Studio.
Browse best sellers, new releases, AutoRip CDs and vinyl records, deals, vinyl Audio CD. 7.99. From A Room: Volume 1. Chris Stapleton. Audio. The one and only CD, DVD, VIDEO, DJ, VINYL, ERO store. Search our full catalog. Recordstore.co.uk. The UK's leading online record store. Buy new and exclusive signed bundles, CDs, LPs, Merchandise and box sets. Recordstore Day, every. Vinyl Records to CD Conversion - Cheapest on the net! High-quality, standards-compliant CD-Audio of your favorite vinyl records, saved for posterity. Custom CD, DVD Vinyl Packaging You're just a click away from a gorgeous, retail-ready CD or DVD in professional disc packaging. We also offer a full-range of Vinyl. 

...

Buy with confidence as the. Mar 4, 2017 Despite the decline in mainstream CD usage, some consumers still have CD recording needs for radio, vinyl and other formats. Here are our. 12 results . You can finally burn your cassettes and vinyl records to CD with Crosley's Memory Master II CD Recorder. Just play your cassette or record One Nation is back after the Sold Out New Years Eve event with yet another From its esoteric origins releasing field recordings of steam engines on vinyl to our latest critically acclaimed Ultradisc UHR™ SACDs, Mobile Fidelity Sound. How much are worth and valued your rare and collectable vinyl and cd by searching on Music Price Guide archive. Heel veel CD, LP, Vinyl SACD op voorraad, snelle levertijden en altijd superscherp geprijsd en lage verzendkosten, voor 17:00 besteld morgen Some of the greatest music ever made isn t available digitally, on mp3, or on CD; but rather is only available on vinyl. Moreover, if you already have purchased.
\end{tbox} 

\begin{tbox}{BurntOrange}{Estimated Data-Quality (Percentile -- Higher is better)}
\centering
\begin{tabular}{ccccc|ccccc}
    \multicolumn{5}{c|}{\textbf{\askllm}} & \multicolumn{5}{c}{\textbf{Perplexity Filtering}} \\[2pt]
     Small & Base & Large & XL & XXL & Small & Base & Large & XL & XXL  \\
     \midrule
     96.67\% & 76.07\% & 47.33\% & 30.0\% & 7.97\% & 32.02\% & 21.27\% & 24.31\% & 25.77\% & 23.7\%
\end{tabular}

\vspace{0.3cm}
\hspace{0.5cm}\hrulefill\hspace{0.2cm} \floweroneleft\floweroneright \hspace{0.2cm}\hrulefill\hspace{0.5cm}
\vspace{0.3cm}

A brilliant performance by Year 6 based on The Lion King. Brilliant singing and acting from everyone, congratulations Year 6! A big thank you to all the staff that helped with everything from costumes, set design, make up and directing. A wonderful commemoration of the seven years that Year 6 students have spent at The Good Shepherd. Thank you to all of the parents and staff for attending this celebration and we wish all of the children continued success in their new schools and hope they continue to do themselves proud. Well done to Foundation for showing us what it is to be good friends! This week we have been looking at all the countries in the world that speak Spanish as their native language, there are 21! So throughout school we spent a day learning lots of wonderful things about our chosen country. We looked at maps, flags, famous people, food and so much more! Below is a little glimpse into our fabulous week. 

...

Click on the links to take a look at some of the brilliant things we got up to! Faith in Families is a charity based here in Nottingham who believe, as we do, that all children have the right to grow up as part of a loving and nurturing family and they provide services for children and families. We learnt lots about adoption and what it can mean for children and their family. We learnt about Fairtrade and all the fantastic work they do around the world. We also discovered lots of products that we did not know were Fairtrade. There was also a sell out Fairtrade food sale, well done everyone! Year 2 have been able to show off our brilliant new high visibility jackets! Now we will be able to stay safe and visible on any out of school trips. We are very lucky to have these donated by Walton \& Allen. Thank you! Click on the high visibility jacket to take a look at our super jackets! Year 4 have wowed us with their acting skills in a brilliant performance of Ali Baba - well done Year 4! Year...
\end{tbox} 

\begin{tbox}{BurntOrange}{Estimated Data-Quality (Percentile -- Higher is better)}
\centering
\begin{tabular}{ccccc|ccccc}
    \multicolumn{5}{c|}{\textbf{\askllm}} & \multicolumn{5}{c}{\textbf{Perplexity Filtering}} \\[2pt]
     Small & Base & Large & XL & XXL & Small & Base & Large & XL & XXL  \\
     \midrule
     90.79\% & 75.97\% & 58.89\% & 18.06\% & 3.0\% & 13.65\% & 16.88\% & 17.85\% & 14.36\% & 13.67\%
\end{tabular}

\vspace{0.3cm}
\hspace{0.5cm}\hrulefill\hspace{0.2cm} \floweroneleft\floweroneright \hspace{0.2cm}\hrulefill\hspace{0.5cm}
\vspace{0.3cm}

Search result for " For Sale "
We supply Germany made embalming powder in small quantities from 1 kg at affordable prices. We have white and pink 100\% hot and 98\% pink in stock. Call us on +27786893835 for details. EMBALMING..
EMBALMING POWDER CALL +27786893835 Hager Werken Embalming Compound Pink Powder call +27786893835 in General items from Germany Embalming compound in powder form both PINK and WHITE Radio active..
Sierra Residences Type B, Sg Ara near PISA, Factory,Air-port Sierra Residences (ID: 5695) ================== Monthly Rent: RM 1,000 BU: 1182 sq.ft. Newly Renovated/NOT Furnished - 3..
Very Strategic and Highly Potential LAND 9.7 Acres Converted Residential Land For Sale in Taman Melawati !!!!! Taman Melawati development land , Titile : Freehold, non bumi land. Status:..
I am a Certified Private Loan Lender, Do you need a Fast and Guarantee loan to pay your bills or start up a Business? I offer both local and international loan services to meet your financial needs..

...

Introducing our mining company to you for a very fruitful business transaction. we are a miners who have come together to upgrade our production through the introduction of modern technology and..
Commercial land for sale. Location near to Premium Outlet. Size = 32 acres Good land shape and very suitable for development. Selling price RM 60 per sf. Interested party kindly contact..
Keterangan : * Tanah yang rata dan sangat startegik untuk buat rumah kediaman/rumah rehat (homestay), atau untuk rumah penginapan sendirian/Percutian (vacation home) * Tanah lot tepi berdekatan..
Limited gated Semi D at Sri petaling,fully furnish with lift and move in condition.newly buit,modern,spacius and practical.Prime location for own stay,good gated security and easy access to few main..
Land for sale in MELAKA ! Price : RM 65 per sq fit (or roughly U\$D 17 per sq fit ) Size : 53000 sf Property type ï¼šfreehold housing land Location : Jalan Laksamana Cheng Ho,Â ..
\end{tbox} 

\begin{tbox}{BurntOrange}{Estimated Data-Quality (Percentile -- Higher is better)}
\centering
\begin{tabular}{ccccc|ccccc}
    \multicolumn{5}{c|}{\textbf{\askllm}} & \multicolumn{5}{c}{\textbf{Perplexity Filtering}} \\[2pt]
     Small & Base & Large & XL & XXL & Small & Base & Large & XL & XXL  \\
     \midrule
     94.72\% & 87.31\% & 78.07\% & 13.77\% & 6.51\% & 5.75\% & 9.63\% & 13.12\% & 17.51\% & 17.12\%
\end{tabular}

\vspace{0.3cm}
\hspace{0.5cm}\hrulefill\hspace{0.2cm} \floweroneleft\floweroneright \hspace{0.2cm}\hrulefill\hspace{0.5cm}
\vspace{0.3cm}

FIFA 20 CONFIRMED TRANSFERS SUMMER 2019 \& RUMOURS | w/ ALEX SANDRO BALE \& NEYMAR JR. TO BARCELONA!!
Top 10 Worst Transfers In Football History!
70 CONFIRMED TRANSFERS JANUARY 2019 ------------------------ Thank You For Watching --------------------------------- * Like + Subscribe * =================.
FIFA 20 | CONFIRMED TRANSFERS SUMMER 2019 \& RUMOURS | w ZIDANE COUTINHO \& RONALDO BACK TO R.MADRID!
REBUILDING REAL MADRID | DREAM TEAM LINEUP 2019-2020 | POTENTIAL TRANSFERS | w/ NEYMAR \& RONALDO!
FIFA 20 | CONFIRMED TRANSFERS SUMMER 2019 \& RUMOURS | w BALE FEKIR UMTITI \& NEYMAR £300M TO MADRID!
SUBSCRIBE http://bit.ly/SoccerAMSub Dean from 442oons is back with his list of the top 5 deals that were done on transfer deadline day. Do you agree with ..
FIFA 20 | CONFIRMED TRANSFERS SUMMER 2019 \& RUMOURS | w STERLING JAMES AUBAMEYANG \& GRIEZMANN!
SUBSCRIBE to FOOTBALL DAILY: http://bit.ly/fdsubscribe Last week we broke down our best signings of the summer so far. Now lets expose the worst!
Top 150 confirmed transfers / signings of the summer transfer window 2018 ft. Ronaldo, Mbappe, Mahrez, Vidal, Courtois... THANK FOR WATCHING!
FIFA 20 | CONFIRMED TRANSFERS SUMMER 2019 \& RUMOURS | w/ POGBA SANCHO THIAGO \& MESSI TO INTER!!
\end{tbox} 

\begin{tbox}{BurntOrange}{Estimated Data-Quality (Percentile -- Higher is better)}
\centering
\begin{tabular}{ccccc|ccccc}
    \multicolumn{5}{c|}{\textbf{\askllm}} & \multicolumn{5}{c}{\textbf{Perplexity Filtering}} \\[2pt]
     Small & Base & Large & XL & XXL & Small & Base & Large & XL & XXL  \\
     \midrule
     86.25\% & 69.2\% & 61.9\% & 46.57\% & 19.99\% & 76.61\% & 71.91\% & 94.86\% & 92.93\% & 94.99\%
\end{tabular}

\vspace{0.3cm}
\hspace{0.5cm}\hrulefill\hspace{0.2cm} \floweroneleft\floweroneright \hspace{0.2cm}\hrulefill\hspace{0.5cm}
\vspace{0.3cm}

Phone 1300 616 202 if you’re looking for a trustworthy, experienced and licensed Plumber Leopold. We know that getting plumbing repairs in Leopold can be a pain and you’ve got better things to do than look for a plumber. Clearwater Plumbing and Maintenance will save you from any unnecessary hassle and expense for a Plumber Leopold.
We make sure that wherever you need a Plumber Leopold, Clearwater Plumbing and Maintenance will assist you with your plumbing worries. Plumbing problems with your taps, toilets, gas, hot water and drains are painful enough. You don’t need the extra stress of finding a Plumber Leopold that you can trust. And what about all of those plumbers in Leopold who don’t clean up after themselves, leaving mud and materials all over your home? Our professional team are different!

...

Do you have hot water system repairs Leopold. We have highly experienced plumbers who know how to fix hot water systems Leopold. There can be many possible reasons why your hot water system Leopold is broken. Our Leopold plumbers are reliable, fast and know hot to diagnose problems. Our hot water system repairs Leopold plumbers are trained and qualified. To book an appointment, please call 1300 616 202. We will do our best to get a plumber to you in Leopold as soon as possible.
If you notice that there is water leaking from the bottom of your hot water system in Leopold, chances are the system is completely broken. In this scenario, you will need to replace your hot water system in Leopold. Our team of plumbers can help you to choose what hot water system you will need.
\end{tbox} 

\begin{tbox}{BurntOrange}{Estimated Data-Quality (Percentile -- Higher is better)}
\centering
\begin{tabular}{ccccc|ccccc}
    \multicolumn{5}{c|}{\textbf{\askllm}} & \multicolumn{5}{c}{\textbf{Perplexity Filtering}} \\[2pt]
     Small & Base & Large & XL & XXL & Small & Base & Large & XL & XXL  \\
     \midrule
     82.64\% & 75.2\% & 63.2\% & 29.51\% & 8.94\% & 78.34\% & 82.07\% & 91.01\% & 87.78\% & 88.02\%
\end{tabular}

\vspace{0.3cm}
\hspace{0.5cm}\hrulefill\hspace{0.2cm} \floweroneleft\floweroneright \hspace{0.2cm}\hrulefill\hspace{0.5cm}
\vspace{0.3cm}

You can now configure the minimum TLS protocol level for client connections and connections to other servers. Refer to the following page for more information: Advanced TLS.
You can now set an Integrated Capture Point (ICP) to stopped mode by changing the state of the corresponding configuration object to disabled; changing the state to enabled restarts the inbound cycle of the ICP.
You can now set the minimum TLS protocol level for the Web Service Capture Point by configuring the option <sec-protocol> in the section <settings> of the Capture Point object.

...

Support for the following databases. See the Supported Operating Environment: eServices page for more detailed information and a list of all supported databases.
No special procedure is required to upgrade to release 8.5.201.05.
Retrieved from "https://docs.genesys.com/Documentation:RN:mm-ixn-svr85rn:mm-ixn-svr8520105:8.5.x (2019-04-21 22:59:48)"
This page was last modified on November 8, 2018, at 08:48.
\end{tbox} 

\begin{tbox}{BurntOrange}{Estimated Data-Quality (Percentile -- Higher is better)}
\centering
\begin{tabular}{ccccc|ccccc}
    \multicolumn{5}{c|}{\textbf{\askllm}} & \multicolumn{5}{c}{\textbf{Perplexity Filtering}} \\[2pt]
     Small & Base & Large & XL & XXL & Small & Base & Large & XL & XXL  \\
     \midrule
     62.21\% & 54.71\% & 35.73\% & 22.64\% & 6.76\% & 64.82\% & 85.95\% & 94.65\% & 93.35\% & 85.29\%
\end{tabular}

\vspace{0.3cm}
\hspace{0.5cm}\hrulefill\hspace{0.2cm} \floweroneleft\floweroneright \hspace{0.2cm}\hrulefill\hspace{0.5cm}
\vspace{0.3cm}

are willing to provide you with perfect services and striding for Display Stand For Boutique , Display Stand for Boutique , Display Stand for Phone , Our product quality is one of the major concerns and has been produced to meet the customer's standards. "Customer services and relationship" is another important area which we understand good communication and relationships with our customers is the most significant power to run it as a long term business.
"We have quite a few great team customers very good at internet marketing, QC, and dealing with kinds of troublesome trouble while in the output approach for Display Stand For Boutique , Display Stand for Boutique , Display Stand for Phone , We set a strict quality control system. We've got return and exchange policy and you can exchange within 7 days after receive the wigs if it is in new station and we service repairing free for our solutions. You should feel free to contact us for further information and we are going to give you competitive price list then.
\end{tbox}

%% file: main.bbl
\begin{thebibliography}{90}
\providecommand{\natexlab}[1]{#1}
\providecommand{\url}[1]{\texttt{#1}}
\expandafter\ifx\csname urlstyle\endcsname\relax
  \providecommand{\doi}[1]{doi: #1}\else
  \providecommand{\doi}{doi: \begingroup \urlstyle{rm}\Url}\fi

\bibitem[Abbas et~al.(2023)Abbas, Tirumala, Simig, Ganguli, and
  Morcos]{semdedup}
Abbas, A., Tirumala, K., Simig, D., Ganguli, S., and Morcos, A.~S.
\newblock Semdedup: Data-efficient learning at web-scale through semantic
  deduplication.
\newblock \emph{arXiv preprint arXiv:2303.09540}, 2023.

\bibitem[Agarwal et~al.(2023)Agarwal, Vieillard, Stanczyk, Ramos, Geist, and
  Bachem]{llm_kd}
Agarwal, R., Vieillard, N., Stanczyk, P., Ramos, S., Geist, M., and Bachem, O.
\newblock Gkd: Generalized knowledge distillation for auto-regressive sequence
  models.
\newblock \emph{arXiv preprint arXiv:2306.13649}, 2023.

\bibitem[Alemohammad et~al.(2023)Alemohammad, Casco-Rodriguez, Luzi, Humayun,
  Babaei, LeJeune, Siahkoohi, and Baraniuk]{alemohammad2023self}
Alemohammad, S., Casco-Rodriguez, J., Luzi, L., Humayun, A.~I., Babaei, H.,
  LeJeune, D., Siahkoohi, A., and Baraniuk, R.~G.
\newblock Self-consuming generative models go mad.
\newblock \emph{arXiv preprint arXiv:2307.01850}, 2023.

\bibitem[Anil et~al.(2023)Anil, Dai, Firat, Johnson, Lepikhin, Passos, Shakeri,
  Taropa, Bailey, and et~al.]{palm2}
Anil, R., Dai, A.~M., Firat, O., Johnson, M., Lepikhin, D., Passos, A.,
  Shakeri, S., Taropa, E., Bailey, P., and et~al., Z.~C.
\newblock Palm 2 technical report, 2023.

\bibitem[Ayed \& Hayou(2023{\natexlab{a}})Ayed and Hayou]{ayed2023data}
Ayed, F. and Hayou, S.
\newblock Data pruning and neural scaling laws: fundamental limitations of
  score-based algorithms.
\newblock \emph{arXiv preprint arXiv:2302.06960}, 2023{\natexlab{a}}.

\bibitem[Ayed \& Hayou(2023{\natexlab{b}})Ayed and
  Hayou]{uniform_sampling_good_1}
Ayed, F. and Hayou, S.
\newblock Data pruning and neural scaling laws: fundamental limitations of
  score-based algorithms.
\newblock \emph{Transactions on Machine Learning Research}, 2023{\natexlab{b}}.
\newblock ISSN 2835-8856.
\newblock URL \url{https://openreview.net/forum?id=iRTL4pDavo}.

\bibitem[Borsos et~al.(2020)Borsos, Mutny, and Krause]{bilevel_coresets}
Borsos, Z., Mutny, M., and Krause, A.
\newblock Coresets via bilevel optimization for continual learning and
  streaming.
\newblock \emph{Advances in Neural Information Processing Systems},
  33:\penalty0 14879--14890, 2020.

\bibitem[Briesch et~al.(2023)Briesch, Sobania, and Rothlauf]{briesch2023large}
Briesch, M., Sobania, D., and Rothlauf, F.
\newblock Large language models suffer from their own output: An analysis of
  the self-consuming training loop.
\newblock \emph{arXiv preprint arXiv:2311.16822}, 2023.

\bibitem[Chelba et~al.(2013)Chelba, Mikolov, Schuster, Ge, Brants, Koehn, and
  Robinson]{lm1b}
Chelba, C., Mikolov, T., Schuster, M., Ge, Q., Brants, T., Koehn, P., and
  Robinson, T.
\newblock One billion word benchmark for measuring progress in statistical
  language modeling.
\newblock \emph{arXiv preprint arXiv:1312.3005}, 2013.

\bibitem[Chen et~al.(2012)Chen, Welling, and Smola]{chen2012super}
Chen, Y., Welling, M., and Smola, A.
\newblock Super-samples from kernel herding.
\newblock \emph{arXiv preprint arXiv:1203.3472}, 2012.

\bibitem[Chitta et~al.(2021)Chitta, {\'A}lvarez, Haussmann, and
  Farabet]{chitta2021training}
Chitta, K., {\'A}lvarez, J.~M., Haussmann, E., and Farabet, C.
\newblock Training data subset search with ensemble active learning.
\newblock \emph{IEEE Transactions on Intelligent Transportation Systems},
  23\penalty0 (9):\penalty0 14741--14752, 2021.

\bibitem[Clark et~al.(2019)Clark, Lee, Chang, Kwiatkowski, Collins, and
  Toutanova]{boolq}
Clark, C., Lee, K., Chang, M.-W., Kwiatkowski, T., Collins, M., and Toutanova,
  K.
\newblock Boolq: Exploring the surprising difficulty of natural yes/no
  questions.
\newblock \emph{arXiv preprint arXiv:1905.10044}, 2019.

\bibitem[Clark et~al.(2018)Clark, Cowhey, Etzioni, Khot, Sabharwal, Schoenick,
  and Tafjord]{arc_eval}
Clark, P., Cowhey, I., Etzioni, O., Khot, T., Sabharwal, A., Schoenick, C., and
  Tafjord, O.
\newblock Think you have solved question answering? try arc, the ai2 reasoning
  challenge.
\newblock \emph{arXiv preprint arXiv:1803.05457}, 2018.

\bibitem[Cobbe et~al.(2021)Cobbe, Kosaraju, Bavarian, Chen, Jun, Kaiser,
  Plappert, Tworek, Hilton, Nakano, et~al.]{gsm8k}
Cobbe, K., Kosaraju, V., Bavarian, M., Chen, M., Jun, H., Kaiser, L., Plappert,
  M., Tworek, J., Hilton, J., Nakano, R., et~al.
\newblock Training verifiers to solve math word problems.
\newblock \emph{arXiv preprint arXiv:2110.14168}, 2021.

\bibitem[Coleman \& Shrivastava(2020)Coleman and Shrivastava]{coleman2020sub}
Coleman, B. and Shrivastava, A.
\newblock Sub-linear race sketches for approximate kernel density estimation on
  streaming data.
\newblock In \emph{Proceedings of The Web Conference 2020}, pp.\  1739--1749,
  2020.

\bibitem[Coleman et~al.(2022)Coleman, Geordie, Chou, Elworth, Treangen, and
  Shrivastava]{density}
Coleman, B., Geordie, B., Chou, L., Elworth, R.~L., Treangen, T., and
  Shrivastava, A.
\newblock One-pass diversified sampling with application to terabyte-scale
  genomic sequence streams.
\newblock In \emph{International Conference on Machine Learning}, pp.\
  4202--4218. PMLR, 2022.

\bibitem[Coleman et~al.(2020)Coleman, Yeh, Mussmann, Mirzasoleiman, Bailis,
  Liang, Leskovec, and Zaharia]{svp}
Coleman, C., Yeh, C., Mussmann, S., Mirzasoleiman, B., Bailis, P., Liang, P.,
  Leskovec, J., and Zaharia, M.
\newblock Selection via proxy: Efficient data selection for deep learning.
\newblock In \emph{ICLR}, 2020.

\bibitem[Datar et~al.(2004)Datar, Immorlica, Indyk, and Mirrokni]{pstable}
Datar, M., Immorlica, N., Indyk, P., and Mirrokni, V.~S.
\newblock Locality-sensitive hashing scheme based on p-stable distributions.
\newblock In \emph{Proceedings of the twentieth annual symposium on
  Computational geometry}, pp.\  253--262, 2004.

\bibitem[Dettmers et~al.(2022)Dettmers, Lewis, Belkada, and
  Zettlemoyer]{llm_int8}
Dettmers, T., Lewis, M., Belkada, Y., and Zettlemoyer, L.
\newblock Llm. int8 (): 8-bit matrix multiplication for transformers at scale.
\newblock \emph{arXiv preprint arXiv:2208.07339}, 2022.

\bibitem[Devroye(1983)]{devroye1983equivalence}
Devroye, L.
\newblock The equivalence of weak, strong and complete convergence in l1 for
  kernel density estimates.
\newblock \emph{The Annals of Statistics}, pp.\  896--904, 1983.

\bibitem[Engstrom et~al.(2024)Engstrom, Feldmann, and Madry]{engstrom2024dsdm}
Engstrom, L., Feldmann, A., and Madry, A.
\newblock Dsdm: Model-aware dataset selection with datamodels, 2024.

\bibitem[Feldman et~al.(2020)Feldman, Schmidt, and Sohler]{feldman2020turning}
Feldman, D., Schmidt, M., and Sohler, C.
\newblock Turning big data into tiny data: Constant-size coresets for k-means,
  pca, and projective clustering.
\newblock \emph{SIAM Journal on Computing}, 49\penalty0 (3):\penalty0 601--657,
  2020.

\bibitem[Feldman \& Zhang(2020)Feldman and Zhang]{feldman2020neural}
Feldman, V. and Zhang, C.
\newblock What neural networks memorize and why: Discovering the long tail via
  influence estimation.
\newblock \emph{Advances in Neural Information Processing Systems},
  33:\penalty0 2881--2891, 2020.

\bibitem[Gemini et~al.(2023)Gemini, Anil, Borgeaud, Wu, Alayrac, Yu, Soricut,
  Schalkwyk, Dai, Hauth, et~al.]{gemini}
Gemini, T., Anil, R., Borgeaud, S., Wu, Y., Alayrac, J.-B., Yu, J., Soricut,
  R., Schalkwyk, J., Dai, A.~M., Hauth, A., et~al.
\newblock Gemini: a family of highly capable multimodal models.
\newblock \emph{arXiv preprint arXiv:2312.11805}, 2023.

\bibitem[Geva et~al.(2021)Geva, Khashabi, Segal, Khot, Roth, and
  Berant]{strategyqa}
Geva, M., Khashabi, D., Segal, E., Khot, T., Roth, D., and Berant, J.
\newblock Did aristotle use a laptop? a question answering benchmark with
  implicit reasoning strategies.
\newblock \emph{Transactions of the Association for Computational Linguistics},
  9:\penalty0 346--361, 2021.

\bibitem[Guo et~al.(2022{\natexlab{a}})Guo, Zhao, and Bai]{guo2022deepcore}
Guo, C., Zhao, B., and Bai, Y.
\newblock Deepcore: A comprehensive library for coreset selection in deep
  learning.
\newblock In \emph{International Conference on Database and Expert Systems
  Applications}, pp.\  181--195. Springer, 2022{\natexlab{a}}.

\bibitem[Guo et~al.(2022{\natexlab{b}})Guo, Zhao, and
  Bai]{uniform_sampling_good_2}
Guo, C., Zhao, B., and Bai, Y.
\newblock Deepcore: A comprehensive library for coreset selection in deep
  learning.
\newblock In \emph{International Conference on Database and Expert Systems
  Applications}, pp.\  181--195. Springer, 2022{\natexlab{b}}.

\bibitem[Guo et~al.(2020)Guo, Dai, Vrande{\v{c}}i{\'c}, and Al-Rfou]{wiki40b}
Guo, M., Dai, Z., Vrande{\v{c}}i{\'c}, D., and Al-Rfou, R.
\newblock Wiki-40b: Multilingual language model dataset.
\newblock In \emph{Proceedings of the Twelfth Language Resources and Evaluation
  Conference}, pp.\  2440--2452, 2020.

\bibitem[Guu et~al.(2023)Guu, Webson, Pavlick, Dixon, Tenney, and
  Bolukbasi]{simfluence}
Guu, K., Webson, A., Pavlick, E., Dixon, L., Tenney, I., and Bolukbasi, T.
\newblock Simfluence: Modeling the influence of individual training examples by
  simulating training runs.
\newblock \emph{arXiv preprint arXiv:2303.08114}, 2023.

\bibitem[Hastings(1970)]{hastings1970monte}
Hastings, W.~K.
\newblock Monte carlo sampling methods using markov chains and their
  applications.
\newblock 1970.

\bibitem[Hendrycks et~al.(2020)Hendrycks, Burns, Basart, Zou, Mazeika, Song,
  and Steinhardt]{mmlu}
Hendrycks, D., Burns, C., Basart, S., Zou, A., Mazeika, M., Song, D., and
  Steinhardt, J.
\newblock Measuring massive multitask language understanding.
\newblock \emph{arXiv preprint arXiv:2009.03300}, 2020.

\bibitem[Hermann et~al.(2015)Hermann, Kocisky, Grefenstette, Espeholt, Kay,
  Suleyman, and Blunsom]{cnndm}
Hermann, K.~M., Kocisky, T., Grefenstette, E., Espeholt, L., Kay, W., Suleyman,
  M., and Blunsom, P.
\newblock Teaching machines to read and comprehend.
\newblock \emph{Advances in neural information processing systems}, 28, 2015.

\bibitem[Hoffmann et~al.(2022)Hoffmann, Borgeaud, Mensch, Buchatskaya, Cai,
  Rutherford, de~Las~Casas, Hendricks, Welbl, Clark, et~al.]{chinchilla}
Hoffmann, J., Borgeaud, S., Mensch, A., Buchatskaya, E., Cai, T., Rutherford,
  E., de~Las~Casas, D., Hendricks, L.~A., Welbl, J., Clark, A., et~al.
\newblock An empirical analysis of compute-optimal large language model
  training.
\newblock \emph{Advances in Neural Information Processing Systems},
  35:\penalty0 30016--30030, 2022.

\bibitem[Indyk et~al.(2014)Indyk, Mahabadi, Mahdian, and
  Mirrokni]{indyk2014composable}
Indyk, P., Mahabadi, S., Mahdian, M., and Mirrokni, V.~S.
\newblock Composable core-sets for diversity and coverage maximization.
\newblock In \emph{Proceedings of the 33rd ACM SIGMOD-SIGACT-SIGART symposium
  on Principles of database systems}, pp.\  100--108, 2014.

\bibitem[Javaheripi et~al.(2023)Javaheripi, Bubeck, Abdin, Aneja, Bubeck,
  Mendes, Chen, Del~Giorno, Eldan, Gopi, et~al.]{javaheripi2023phi}
Javaheripi, M., Bubeck, S., Abdin, M., Aneja, J., Bubeck, S., Mendes, C. C.~T.,
  Chen, W., Del~Giorno, A., Eldan, R., Gopi, S., et~al.
\newblock Phi-2: The surprising power of small language models, 2023.

\bibitem[Jiang et~al.(2019)Jiang, Wong, Zhou, Andersen, Dean, Ganger, Joshi,
  Kaminksy, Kozuch, Lipton, et~al.]{selective_backprop}
Jiang, A.~H., Wong, D. L.-K., Zhou, G., Andersen, D.~G., Dean, J., Ganger,
  G.~R., Joshi, G., Kaminksy, M., Kozuch, M., Lipton, Z.~C., et~al.
\newblock Accelerating deep learning by focusing on the biggest losers.
\newblock \emph{arXiv preprint arXiv:1910.00762}, 2019.

\bibitem[Kaplan et~al.(2020)Kaplan, McCandlish, Henighan, Brown, Chess, Child,
  Gray, Radford, Wu, and Amodei]{kaplan_scaling}
Kaplan, J., McCandlish, S., Henighan, T., Brown, T.~B., Chess, B., Child, R.,
  Gray, S., Radford, A., Wu, J., and Amodei, D.
\newblock Scaling laws for neural language models.
\newblock \emph{arXiv preprint arXiv:2001.08361}, 2020.

\bibitem[Karnin \& Liberty(2019)Karnin and Liberty]{discrepancy}
Karnin, Z. and Liberty, E.
\newblock Discrepancy, coresets, and sketches in machine learning.
\newblock In \emph{Conference on Learning Theory}, pp.\  1975--1993. PMLR,
  2019.

\bibitem[Katharopoulos \& Fleuret(2018)Katharopoulos and
  Fleuret]{katharopoulos2018not}
Katharopoulos, A. and Fleuret, F.
\newblock Not all samples are created equal: Deep learning with importance
  sampling.
\newblock In \emph{International conference on machine learning}, pp.\
  2525--2534. PMLR, 2018.

\bibitem[Khashabi et~al.(2020)Khashabi, Min, Khot, Sabharwal, Tafjord, Clark,
  and Hajishirzi]{unifiedqa}
Khashabi, D., Min, S., Khot, T., Sabharwal, A., Tafjord, O., Clark, P., and
  Hajishirzi, H.
\newblock Unifiedqa: Crossing format boundaries with a single qa system.
\newblock \emph{arXiv preprint arXiv:2005.00700}, 2020.

\bibitem[Kirsch \& Mitzenmacher(2006)Kirsch and
  Mitzenmacher]{kirsch2006distance}
Kirsch, A. and Mitzenmacher, M.
\newblock Distance-sensitive bloom filters.
\newblock In \emph{2006 Proceedings of the Eighth Workshop on Algorithm
  Engineering and Experiments (ALENEX)}, pp.\  41--50. SIAM, 2006.

\bibitem[Lee et~al.(2023)Lee, Miranda, and Koyejo]{lee2023beyond}
Lee, A., Miranda, B., and Koyejo, S.
\newblock Beyond scale: the diversity coefficient as a data quality metric
  demonstrates llms are pre-trained on formally diverse data.
\newblock \emph{arXiv preprint arXiv:2306.13840}, 2023.

\bibitem[Lee et~al.(2022)Lee, Ippolito, Nystrom, Zhang, Eck, Callison-Burch,
  and Carlini]{lee2022deduplicating}
Lee, K., Ippolito, D., Nystrom, A., Zhang, C., Eck, D., Callison-Burch, C., and
  Carlini, N.
\newblock Deduplicating training data makes language models better.
\newblock In \emph{Proceedings of the 60th Annual Meeting of the Association
  for Computational Linguistics (Volume 1: Long Papers)}, pp.\  8424--8445,
  2022.

\bibitem[Liu et~al.(2023)Liu, Xu, Coleman, and Shrivastava]{liu2023one}
Liu, Z., Xu, Z., Coleman, B., and Shrivastava, A.
\newblock One-pass distribution sketch for measuring data heterogeneity in
  federated learning.
\newblock In \emph{Thirty-seventh Conference on Neural Information Processing
  Systems}, 2023.

\bibitem[Longpre et~al.(2023{\natexlab{a}})Longpre, Hou, Vu, Webson, Chung,
  Tay, Zhou, Le, Zoph, Wei, et~al.]{flanv2}
Longpre, S., Hou, L., Vu, T., Webson, A., Chung, H.~W., Tay, Y., Zhou, D., Le,
  Q.~V., Zoph, B., Wei, J., et~al.
\newblock The flan collection: Designing data and methods for effective
  instruction tuning.
\newblock \emph{arXiv preprint arXiv:2301.13688}, 2023{\natexlab{a}}.

\bibitem[Longpre et~al.(2023{\natexlab{b}})Longpre, Yauney, Reif, Lee, Roberts,
  Zoph, Zhou, Wei, Robinson, Mimno, et~al.]{longpre2023pretrainer}
Longpre, S., Yauney, G., Reif, E., Lee, K., Roberts, A., Zoph, B., Zhou, D.,
  Wei, J., Robinson, K., Mimno, D., et~al.
\newblock A pretrainer's guide to training data: Measuring the effects of data
  age, domain coverage, quality, \& toxicity.
\newblock \emph{arXiv preprint arXiv:2305.13169}, 2023{\natexlab{b}}.

\bibitem[Luccioni et~al.(2023)Luccioni, Viguier, and
  Ligozat]{luccioni2023estimating}
Luccioni, A.~S., Viguier, S., and Ligozat, A.-L.
\newblock Estimating the carbon footprint of bloom, a 176b parameter language
  model.
\newblock \emph{Journal of Machine Learning Research}, 24\penalty0
  (253):\penalty0 1--15, 2023.

\bibitem[Mai et~al.(2021)Mai, Musco, and Rao]{mai2021coresets}
Mai, T., Musco, C., and Rao, A.
\newblock Coresets for classification--simplified and strengthened.
\newblock \emph{Advances in Neural Information Processing Systems},
  34:\penalty0 11643--11654, 2021.

\bibitem[Maini et~al.(2024)Maini, Seto, Bai, Grangier, Zhang, and
  Jaitly]{maini2024rephrasing}
Maini, P., Seto, S., Bai, H., Grangier, D., Zhang, Y., and Jaitly, N.
\newblock Rephrasing the web: A recipe for compute and data-efficient language
  modeling, 2024.

\bibitem[Marion et~al.(2023)Marion, {\"U}st{\"u}n, Pozzobon, Wang, Fadaee, and
  Hooker]{marion2023less}
Marion, M., {\"U}st{\"u}n, A., Pozzobon, L., Wang, A., Fadaee, M., and Hooker,
  S.
\newblock When less is more: Investigating data pruning for pretraining llms at
  scale.
\newblock \emph{arXiv preprint arXiv:2309.04564}, 2023.

\bibitem[Meding et~al.(2021)Meding, Buschoff, Geirhos, and
  Wichmann]{meding2021trivial}
Meding, K., Buschoff, L. M.~S., Geirhos, R., and Wichmann, F.~A.
\newblock Trivial or impossible--dichotomous data difficulty masks model
  differences (on imagenet and beyond).
\newblock \emph{arXiv preprint arXiv:2110.05922}, 2021.

\bibitem[Miao et~al.(2021)Miao, Liang, and Su]{asdiv}
Miao, S.-Y., Liang, C.-C., and Su, K.-Y.
\newblock A diverse corpus for evaluating and developing english math word
  problem solvers.
\newblock \emph{arXiv preprint arXiv:2106.15772}, 2021.

\bibitem[Mindermann et~al.(2022)Mindermann, Brauner, Razzak, Sharma, Kirsch,
  Xu, H{\"o}ltgen, Gomez, Morisot, Farquhar, et~al.]{mindermann2022prioritized}
Mindermann, S., Brauner, J.~M., Razzak, M.~T., Sharma, M., Kirsch, A., Xu, W.,
  H{\"o}ltgen, B., Gomez, A.~N., Morisot, A., Farquhar, S., et~al.
\newblock Prioritized training on points that are learnable, worth learning,
  and not yet learnt.
\newblock In \emph{International Conference on Machine Learning}, pp.\
  15630--15649. PMLR, 2022.

\bibitem[Muennighoff et~al.(2023)Muennighoff, Rush, Barak, Scao, Piktus, Tazi,
  Pyysalo, Wolf, and Raffel]{muennighoff2023scaling}
Muennighoff, N., Rush, A.~M., Barak, B., Scao, T.~L., Piktus, A., Tazi, N.,
  Pyysalo, S., Wolf, T., and Raffel, C.
\newblock Scaling data-constrained language models.
\newblock \emph{arXiv preprint arXiv:2305.16264}, 2023.

\bibitem[Ni et~al.(2021)Ni, {\'A}brego, Constant, Ma, Hall, Cer, and
  Yang]{sentence_t5}
Ni, J., {\'A}brego, G.~H., Constant, N., Ma, J., Hall, K.~B., Cer, D., and
  Yang, Y.
\newblock Sentence-t5: Scalable sentence encoders from pre-trained text-to-text
  models.
\newblock \emph{arXiv preprint arXiv:2108.08877}, 2021.

\bibitem[OpenAI(2023)]{gpt4}
OpenAI.
\newblock Gpt-4 technical report, 2023.

\bibitem[Patel et~al.(2021)Patel, Bhattamishra, and Goyal]{svamp}
Patel, A., Bhattamishra, S., and Goyal, N.
\newblock Are nlp models really able to solve simple math word problems?
\newblock \emph{arXiv preprint arXiv:2103.07191}, 2021.

\bibitem[Paul et~al.(2021)Paul, Ganguli, and Dziugaite]{el2n}
Paul, M., Ganguli, S., and Dziugaite, G.~K.
\newblock Deep learning on a data diet: Finding important examples early in
  training.
\newblock \emph{Advances in Neural Information Processing Systems},
  34:\penalty0 20596--20607, 2021.

\bibitem[Phillips(2017)]{phillips2017coresets}
Phillips, J.~M.
\newblock Coresets and sketches.
\newblock In \emph{Handbook of discrete and computational geometry}, pp.\
  1269--1288. Chapman and Hall/CRC, 2017.

\bibitem[Phuong \& Hutter(2022)Phuong and Hutter]{phuong2022formal}
Phuong, M. and Hutter, M.
\newblock Formal algorithms for transformers.
\newblock \emph{arXiv preprint arXiv:2207.09238}, 2022.

\bibitem[Radford et~al.(2019)Radford, Wu, Child, Luan, Amodei, Sutskever,
  et~al.]{gpt2}
Radford, A., Wu, J., Child, R., Luan, D., Amodei, D., Sutskever, I., et~al.
\newblock Language models are unsupervised multitask learners.
\newblock \emph{OpenAI blog}, 1\penalty0 (8):\penalty0 9, 2019.

\bibitem[Raffel et~al.(2020)Raffel, Shazeer, Roberts, Lee, Narang, Matena,
  Zhou, Li, and Liu]{t5}
Raffel, C., Shazeer, N., Roberts, A., Lee, K., Narang, S., Matena, M., Zhou,
  Y., Li, W., and Liu, P.~J.
\newblock Exploring the limits of transfer learning with a unified text-to-text
  transformer.
\newblock \emph{The Journal of Machine Learning Research}, 21\penalty0
  (1):\penalty0 5485--5551, 2020.

\bibitem[Rajpurkar et~al.(2016)Rajpurkar, Zhang, Lopyrev, and Liang]{squad}
Rajpurkar, P., Zhang, J., Lopyrev, K., and Liang, P.
\newblock Squad: 100,000+ questions for machine comprehension of text.
\newblock \emph{arXiv preprint arXiv:1606.05250}, 2016.

\bibitem[Rosenbaum \& Rubin(1983)Rosenbaum and Rubin]{rosenbaum1983central}
Rosenbaum, P.~R. and Rubin, D.~B.
\newblock The central role of the propensity score in observational studies for
  causal effects.
\newblock \emph{Biometrika}, 70\penalty0 (1):\penalty0 41--55, 1983.

\bibitem[Rosenblatt(1956)]{rosenblatt1956remarks}
Rosenblatt, M.
\newblock Remarks on some nonparametric estimates of a density function.
\newblock \emph{The annals of mathematical statistics}, pp.\  832--837, 1956.

\bibitem[Sachdeva \& McAuley(2023)Sachdeva and McAuley]{dd_survey}
Sachdeva, N. and McAuley, J.
\newblock Data distillation: A survey.
\newblock \emph{Transactions on Machine Learning Research}, 2023.
\newblock ISSN 2835-8856.
\newblock Survey Certification.

\bibitem[Sachdeva et~al.(2021)Sachdeva, Wu, and McAuley]{svp_cf}
Sachdeva, N., Wu, C.-J., and McAuley, J.
\newblock Svp-cf: Selection via proxy for collaborative filtering data.
\newblock \emph{arXiv preprint arXiv:2107.04984}, 2021.

\bibitem[Sachdeva et~al.(2023)Sachdeva, He, Kang, Ni, Cheng, and
  McAuley]{farzi}
Sachdeva, N., He, Z., Kang, W.-C., Ni, J., Cheng, D.~Z., and McAuley, J.
\newblock Farzi data: Autoregressive data distillation.
\newblock \emph{arXiv preprint arXiv:2310.09983}, 2023.

\bibitem[Schubert et~al.(2014)Schubert, Zimek, and
  Kriegel]{schubert2014generalized}
Schubert, E., Zimek, A., and Kriegel, H.-P.
\newblock Generalized outlier detection with flexible kernel density estimates.
\newblock In \emph{Proceedings of the 2014 SIAM International Conference on
  Data Mining}, pp.\  542--550. SIAM, 2014.

\bibitem[Shen et~al.(2023)Shen, Hou, Zhou, Du, Longpre, Wei, Chung, Zoph,
  Fedus, Chen, et~al.]{shen2023mixture}
Shen, S., Hou, L., Zhou, Y., Du, N., Longpre, S., Wei, J., Chung, H.~W., Zoph,
  B., Fedus, W., Chen, X., et~al.
\newblock Mixture-of-experts meets instruction tuning: A winning combination
  for large language models.
\newblock \emph{arXiv preprint arXiv:2305.14705}, 2023.

\bibitem[Shumailov et~al.(2023)Shumailov, Shumaylov, Zhao, Gal, Papernot, and
  Anderson]{shumailov2023curse}
Shumailov, I., Shumaylov, Z., Zhao, Y., Gal, Y., Papernot, N., and Anderson, R.
\newblock The curse of recursion: Training on generated data makes models
  forget.(5 2023).
\newblock \emph{URl: https://arxiv. org/abs/2305.17493}, 2023.

\bibitem[Siminelakis et~al.(2019)Siminelakis, Rong, Bailis, Charikar, and
  Levis]{siminelakis2019rehashing}
Siminelakis, P., Rong, K., Bailis, P., Charikar, M., and Levis, P.
\newblock Rehashing kernel evaluation in high dimensions.
\newblock In \emph{International Conference on Machine Learning}, pp.\
  5789--5798. PMLR, 2019.

\bibitem[Sorscher et~al.(2022)Sorscher, Geirhos, Shekhar, Ganguli, and
  Morcos]{prototypes}
Sorscher, B., Geirhos, R., Shekhar, S., Ganguli, S., and Morcos, A.
\newblock Beyond neural scaling laws: beating power law scaling via data
  pruning.
\newblock \emph{Advances in Neural Information Processing Systems},
  35:\penalty0 19523--19536, 2022.

\bibitem[Srivastava et~al.(2022)Srivastava, Rastogi, Rao, Shoeb, Abid, Fisch,
  Brown, Santoro, Gupta, Garriga-Alonso, et~al.]{bigbench}
Srivastava, A., Rastogi, A., Rao, A., Shoeb, A. A.~M., Abid, A., Fisch, A.,
  Brown, A.~R., Santoro, A., Gupta, A., Garriga-Alonso, A., et~al.
\newblock Beyond the imitation game: Quantifying and extrapolating the
  capabilities of language models.
\newblock \emph{arXiv preprint arXiv:2206.04615}, 2022.

\bibitem[Strubell et~al.(2019)Strubell, Ganesh, and
  McCallum]{strubell2019energy}
Strubell, E., Ganesh, A., and McCallum, A.
\newblock Energy and policy considerations for deep learning in nlp.
\newblock \emph{arXiv preprint arXiv:1906.02243}, 2019.

\bibitem[Tirumala et~al.(2023)Tirumala, Simig, Aghajanyan, and
  Morcos]{tirumala2023d4}
Tirumala, K., Simig, D., Aghajanyan, A., and Morcos, A.~S.
\newblock D4: Improving llm pretraining via document de-duplication and
  diversification.
\newblock \emph{arXiv preprint arXiv:2308.12284}, 2023.

\bibitem[Toneva et~al.(2019)Toneva, Sordoni, Combes, Trischler, Bengio, and
  Gordon]{forgetting_events}
Toneva, M., Sordoni, A., Combes, R., Trischler, A., Bengio, Y., and Gordon, G.
\newblock An empirical study of example forgetting during deep neural network
  learning.
\newblock In \emph{ICLR}, 2019.

\bibitem[Touvron et~al.(2023{\natexlab{a}})Touvron, Lavril, Izacard, Martinet,
  Lachaux, Lacroix, Rozi{\`e}re, Goyal, Hambro, Azhar, et~al.]{llama}
Touvron, H., Lavril, T., Izacard, G., Martinet, X., Lachaux, M.-A., Lacroix,
  T., Rozi{\`e}re, B., Goyal, N., Hambro, E., Azhar, F., et~al.
\newblock Llama: Open and efficient foundation language models.
\newblock \emph{arXiv preprint arXiv:2302.13971}, 2023{\natexlab{a}}.

\bibitem[Touvron et~al.(2023{\natexlab{b}})Touvron, Martin, Stone, Albert,
  Almahairi, Babaei, Bashlykov, Batra, Bhargava, Bhosale, et~al.]{llama2}
Touvron, H., Martin, L., Stone, K., Albert, P., Almahairi, A., Babaei, Y.,
  Bashlykov, N., Batra, S., Bhargava, P., Bhosale, S., et~al.
\newblock Llama 2: Open foundation and fine-tuned chat models.
\newblock \emph{arXiv preprint arXiv:2307.09288}, 2023{\natexlab{b}}.

\bibitem[Tukan et~al.(2021)Tukan, Baykal, Feldman, and Rus]{tukan2021coresets}
Tukan, M., Baykal, C., Feldman, D., and Rus, D.
\newblock On coresets for support vector machines.
\newblock \emph{Theoretical Computer Science}, 890:\penalty0 171--191, 2021.

\bibitem[Wang et~al.(2018)Wang, Singh, Michael, Hill, Levy, and Bowman]{glue}
Wang, A., Singh, A., Michael, J., Hill, F., Levy, O., and Bowman, S.~R.
\newblock Glue: A multi-task benchmark and analysis platform for natural
  language understanding.
\newblock \emph{arXiv preprint arXiv:1804.07461}, 2018.

\bibitem[Wang et~al.(2019)Wang, Pruksachatkun, Nangia, Singh, Michael, Hill,
  Levy, and Bowman]{superglue}
Wang, A., Pruksachatkun, Y., Nangia, N., Singh, A., Michael, J., Hill, F.,
  Levy, O., and Bowman, S.
\newblock Superglue: A stickier benchmark for general-purpose language
  understanding systems.
\newblock \emph{Advances in neural information processing systems}, 32, 2019.

\bibitem[Wei et~al.(2022)Wei, Wang, Schuurmans, Bosma, Xia, Chi, Le, Zhou,
  et~al.]{wei2022chain}
Wei, J., Wang, X., Schuurmans, D., Bosma, M., Xia, F., Chi, E., Le, Q.~V.,
  Zhou, D., et~al.
\newblock Chain-of-thought prompting elicits reasoning in large language
  models.
\newblock \emph{Advances in Neural Information Processing Systems},
  35:\penalty0 24824--24837, 2022.

\bibitem[Weng(2023)]{efficient_inference_blog}
Weng, L.
\newblock Large transformer model inference optimization.
\newblock \emph{Lil'Log}, Jan 2023.
\newblock URL
  \url{https://lilianweng.github.io/posts/2023-01-10-inference-optimization/}.

\bibitem[Wenzek et~al.(2019)Wenzek, Lachaux, Conneau, Chaudhary, Guzm{\'a}n,
  Joulin, and Grave]{wenzek2019ccnet}
Wenzek, G., Lachaux, M.-A., Conneau, A., Chaudhary, V., Guzm{\'a}n, F., Joulin,
  A., and Grave, E.
\newblock Ccnet: Extracting high quality monolingual datasets from web crawl
  data.
\newblock \emph{arXiv preprint arXiv:1911.00359}, 2019.

\bibitem[Wied \& Wei{\ss}bach(2012)Wied and Wei{\ss}bach]{wied2012consistency}
Wied, D. and Wei{\ss}bach, R.
\newblock Consistency of the kernel density estimator: a survey.
\newblock \emph{Statistical Papers}, 53\penalty0 (1):\penalty0 1--21, 2012.

\bibitem[Xie et~al.(2023{\natexlab{a}})Xie, Pham, Dong, Du, Liu, Lu, Liang, Le,
  Ma, and Yu]{doremi}
Xie, S.~M., Pham, H., Dong, X., Du, N., Liu, H., Lu, Y., Liang, P., Le, Q.~V.,
  Ma, T., and Yu, A.~W.
\newblock Doremi: Optimizing data mixtures speeds up language model
  pretraining.
\newblock \emph{arXiv preprint arXiv:2305.10429}, 2023{\natexlab{a}}.

\bibitem[Xie et~al.(2023{\natexlab{b}})Xie, Santurkar, Ma, and
  Liang]{xie2023data}
Xie, S.~M., Santurkar, S., Ma, T., and Liang, P.
\newblock Data selection for language models via importance resampling.
\newblock \emph{arXiv preprint arXiv:2302.03169}, 2023{\natexlab{b}}.

\bibitem[Zhang et~al.(2022)Zhang, Roller, Goyal, Artetxe, Chen, Chen, Dewan,
  Diab, Li, Lin, et~al.]{zhang2022opt}
Zhang, S., Roller, S., Goyal, N., Artetxe, M., Chen, M., Chen, S., Dewan, C.,
  Diab, M., Li, X., Lin, X.~V., et~al.
\newblock Opt: Open pre-trained transformer language models.
\newblock \emph{arXiv preprint arXiv:2205.01068}, 2022.

\bibitem[Zhou et~al.(2023)Zhou, Liu, Xu, Iyer, Sun, Mao, Ma, Efrat, Yu, Yu,
  et~al.]{lima}
Zhou, C., Liu, P., Xu, P., Iyer, S., Sun, J., Mao, Y., Ma, X., Efrat, A., Yu,
  P., Yu, L., et~al.
\newblock Lima: Less is more for alignment.
\newblock \emph{arXiv preprint arXiv:2305.11206}, 2023.

\end{thebibliography}
